\documentclass[11pt]{article}
\usepackage[table]{xcolor}
\usepackage[final]{acl}
\usepackage{booktabs}
\usepackage{xcolor}
\usepackage{multirow}

\usepackage{colortbl}
\usepackage{graphicx} 
\definecolor{HeaderBlue}{HTML}{E6EEF8}
\definecolor{LightGreen}{HTML}{E8F5E9}
\definecolor{LightPink}{HTML}{FCE4EC}

\definecolor{RowGray}{gray}{0.96}
\definecolor{textblue}{RGB}{204,229,255}
\definecolor{imagepurple}{RGB}{221,212,232}
\definecolor{videogreen}{RGB}{208,231,217}
\definecolor{headercolor}{RGB}{249,235,222}
\definecolor{ourmethodcolor}{RGB}{220,245,248}
\definecolor{NoteCol}{RGB}{249,235,222}

\usepackage{amssymb} 
\usepackage{amsthm} 
\usepackage{enumitem}
\theoremstyle{plain} 
\newtheorem{theorem}{Theorem}
\newtheorem{assumption}{Assumption}
\theoremstyle{definition}
\newtheorem{definition}{Definition}
\theoremstyle{plain}
\newtheorem{proposition}{Proposition}
\newcommand{\trainable}[1]{\textcolor{orange!85!black}{\mathbf{#1}}}
\newcommand{\ema}[1]{\textcolor{teal}{\mathbf{#1}}}
\usepackage{tcolorbox}
\newtcolorbox{notebox}{
  colback=NoteCol,
  colframe=black,
    boxrule=0.3pt,
    arc=0pt,
    left=2pt, right=2pt, top=1.5pt, bottom=1.5pt,
    boxsep=1pt,
    before skip=2pt,
    after skip=2pt,
    fontupper=\small
}
\newtcolorbox{theorembox}{
    colback=ourmethodcolor,   
    colframe=black,           
    boxrule=0.3pt,
    arc=2pt,
    left=2pt, right=2pt, top=2.5pt, bottom=2.5pt,
    boxsep=1pt,
    before skip=2pt,
    after skip=2pt,
    fontupper=\small
}

\usepackage{times}
\usepackage{latexsym}
\usepackage{xcolor}
\definecolor{lightblue}{RGB}{0,180,255}
\newcommand{\frozen}[1]{\textcolor{lightblue}{\mathbf{#1}}}  %

\usepackage{colortbl}
\definecolor{SharedCol}{RGB}{220,245,248}
\definecolor{TaskCol}{RGB}{249,235,222}
\usepackage[T1]{fontenc}

\usepackage[utf8]{inputenc}

\usepackage{microtype}
\usepackage{amsmath}      
\usepackage{amssymb}      
\usepackage{amsfonts}     
\usepackage{mathtools}    
\usepackage{bm}           
\usepackage{physics}      
\usepackage{bbm}          
\usepackage{microtype}    
\usepackage{caption}
\usepackage{graphicx}     
\usepackage{xcolor}       
\usepackage{booktabs}     
\usepackage{caption}      
\usepackage{subcaption}   
\usepackage{nicefrac}     
\usepackage{hyperref}     
\usepackage{fontawesome5}
\usepackage{inconsolata}

\usepackage{graphicx}

\usepackage{graphicx}
\usepackage{capt-of}
\usepackage{cuted}

\usepackage{amsmath,amssymb}
\usepackage{xcolor}
\usepackage{algorithm}
\usepackage{algpseudocode} 

\definecolor{titleblue}{RGB}{46,108,235}
\definecolor{boxblue}{RGB}{245,248,255}
\definecolor{chipPinkFrame}{RGB}{204, 35, 112}
\definecolor{chipPinkBack}{RGB}{255, 230, 242}
\definecolor{chipOrangeFrame}{RGB}{210,120,  0}
\definecolor{chipOrangeBack}{RGB}{255,240,220}

\tcbset{
  monkeybox/.style={
    enhanced,
    breakable,
    colback=boxblue,
    colframe=titleblue!60!black,
    boxrule=0.7pt,
    arc=2.2mm,
    left=2mm,right=2mm,top=1mm,bottom=1mm,
    sharp corners,
    drop fuzzy shadow=black!25!white,
    fonttitle=\bfseries,
    coltitle=white,
    title filled,
    attach boxed title to top left={yshift*=-\tcboxedtitleheight/2},
    boxed title style={
      enhanced,
      size=small,
      colback=titleblue,
      colframe=titleblue,
      arc=2mm,
      boxrule=0pt,
      left=3mm,right=3mm,top=1mm,bottom=1mm
    }
  }
}

\newtcbox{\chip}[1][]{
  on line,
  enhanced,
  boxrule=0.6pt,
  arc=2.2mm,
  top=0.2ex,bottom=0.2ex,left=0.6ex,right=0.6ex,
  boxsep=0pt,
  fontupper=\bfseries\footnotesize,
  #1
}

\title{Monkey Jump: MoE-Style PEFT for Efficient Multi-Task Learning}

\author{
\textbf{Nusrat Jahan Prottasha}\textsuperscript{1},
\textbf{Md Kowsher}\textsuperscript{1},
\textbf{Chun-Nam Yu}\textsuperscript{2}, \\
\textbf{Chen Chen}\textsuperscript{1},
\textbf{Ozlem Garibay}\textsuperscript{1} \\[1ex]
\textsuperscript{1}University of Central Florida \quad
\textsuperscript{2}Nokia Bell Labs \\[1ex]
\href{https://github.com/Nusrat-Prottasha/MonkeyJump}{\faGithub} \quad
\href{https://nusrat-prottasha.github.io/MonkeyJump/}{\faGlobe}
}

\begin{document}

\maketitle
\begin{abstract}
Mixture-of-experts variants of parameter-efficient fine-tuning enable per-token specialization, but they introduce additional trainable routers and expert parameters, increasing memory and training costs. This undermines the core goal of parameter efficient finetuning. We propose \emph{Monkey Jump}\footnote{Named for the selective activation pattern: adapters ``jump'' on for some projections and off for others.}, which brings MoE-style specialization to PEFT without adding extra trainable parameters for experts and routers. Instead of introducing new PEFT adapters as experts, Monkey Jump treats the PEFT adapters already present in each Transformer block (e.g., query, key, value, up, and down projections) as implicit experts and routes tokens among them. Routing is performed via $k$-means clustering with EMA-updated centers—no gradients, no learned parameters. We give an idealized analysis motivating why token-wise routing can preserve adapter capacity that uniform application cancels. In multi-task experiments spanning 14 text, 14 image, and 19 video benchmarks, Monkey Jump achieves competitive performance with MoE-PEFT methods while using $5.3$--$9.2\times$ fewer trainable parameters than the MoE variant of the same base method, up to 48\% lower memory, and 1.5--2$\times$ faster training. Monkey Jump is architecture-agnostic and can be applied to any adapter-based PEFT method. 
\end{abstract}

\begin{figure*}[htbp]
  \centering
  \includegraphics[width=1.0\textwidth]{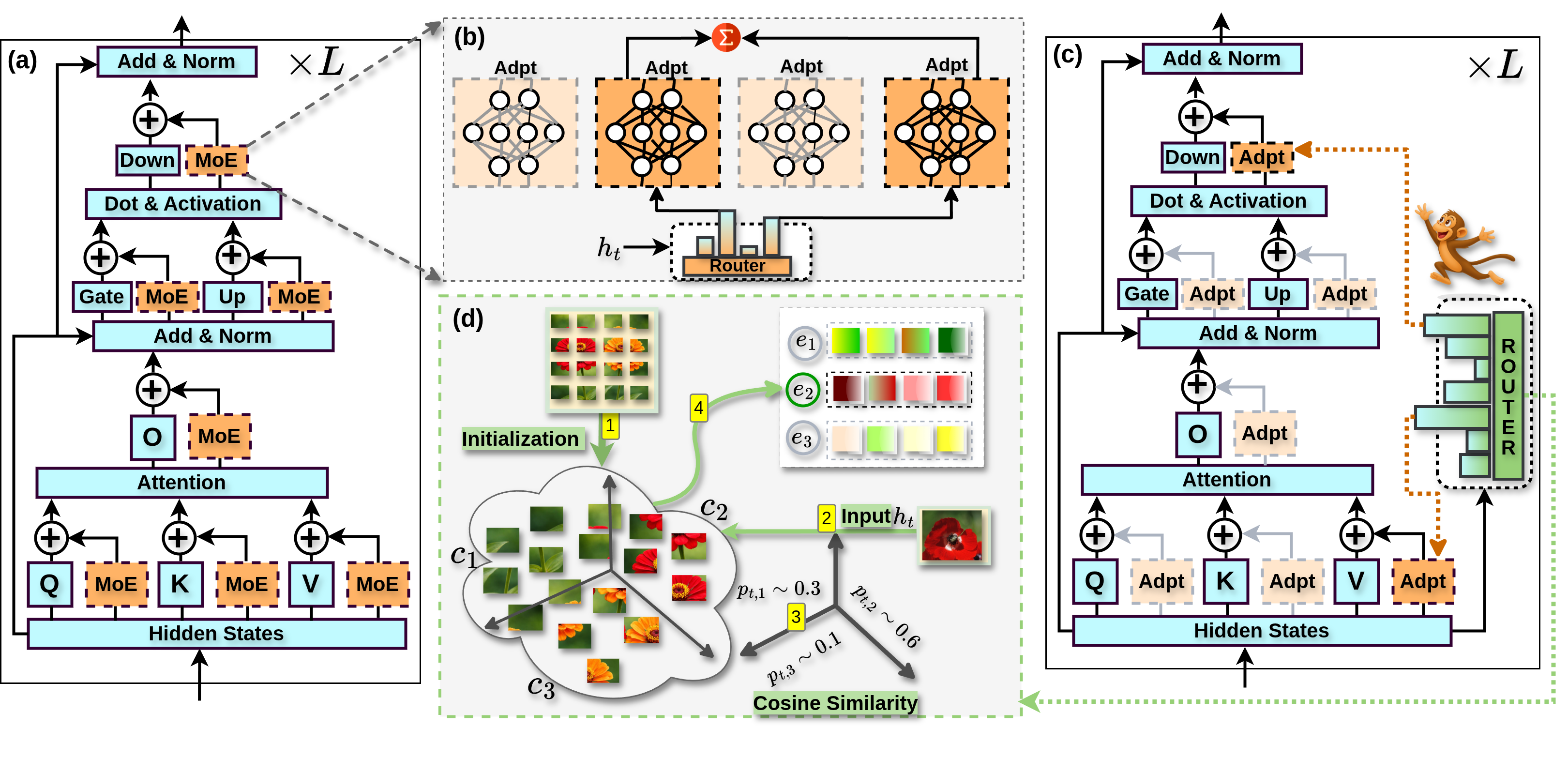}
  \vspace{-0.9cm}
\caption{Overview of MJ compared to MoE-PEFT. 
\textbf{(a)} MoE-PEFT architecture: each projection (Q, K, V, O, Gate, Up, Down) has multiple expert adapters with a learned router.
\textbf{(b)} MoE-PEFT routing: a trainable router selects among $N$ expert adapters, and outputs are summed ($\Sigma$). 
\textbf{(c)} MJ architecture: each projection has a single adapter (same as standard PEFT). Here, V and Down adapters are activated ($k{=}2$); the rest are skipped. For each token $h_t$, inactive projections ($m_{t,e}=0$) apply only the frozen weight: $y_{t,e} = \frozen{W_e} h_t$. Active projections ($m_{t,e}>0$) apply $y_{t,e} = \frozen{W_e} h_t + m_{t,e} \cdot \trainable{\Delta W_e} h_t$. \textbf{(d)} MJ routing mechanism: \textcircled{1} Initialize cluster centers $\ema{C}$ via $k$-means before training; \textcircled{2} For input token $h_t$, \textcircled{3} compute cosine similarity to each center; \textcircled{4} Select top-$k$ experts based on similarity (here $e_2$ is activated, $e_1$ and $e_3$ are skipped).}
\label{fig:moe_lora_monkeyjump}
\vspace{-0.4cm}
\end{figure*}

\section{Introduction}
\label{sec:intro}

Large language models achieve remarkable performance across many tasks, but fine-tuning all their parameters is expensive~\cite{prottasha2025peft}. Parameter-efficient fine-tuning (PEFT) methods like LoRA~\cite{hu2022lora} address this by freezing the pretrained weights and training only small adapter modules. This reduces memory usage and training time while maintaining strong performance. PEFT has become the standard approach for adapting large models to downstream tasks.

However, standard PEFT applies the same adapters uniformly to all inputs. Every token receives the same transformation, regardless of its content. This uniformity limits the model's ability to specialize for different input types, which becomes problematic in multi-task learning where diverse tasks require different adaptations ~\cite{ma2025ts, luo2024moelora}.

Recent work combines mixture-of-experts (MoE) with PEFT to address this limitation~\cite{li2024mixlora, luo2024moelora}. These methods create multiple adapter experts per layer and use a learned router to select which experts to apply for each input. This enables specialization—different inputs activate different experts—and improves performance on many benchmarks. However, MoE-PEFT methods introduce significant overhead compared to standard PEFT: (i) \textbf{more parameters}—multiple experts of a layer multiply the adapter count by $N\times$; (ii) \textbf{learned routers}—routing networks add $O(Nd)$ trainable parameters per layer; (iii) \textbf{higher memory}—evaluating multiple experts increases activation memory; and (iv) \textbf{slower training}—multiple activated experts participate in gradient updates. These costs conflict with the core goal of PEFT: efficient adaptation under resource constraints.

\begin{notebox}
\textit{\textbf{Motivation:} When PEFT is applied to all  projections in a Transformer block, it creates separate PEFT adapters. Rather than adding new experts in each  projection, Monkey Jump routes tokens among these existing adapters—achieving MoE-style specialization with zero additional expert parameters.}
\end{notebox}

We propose \textbf{Monkey Jump (MJ)}, a method that brings MoE-style specialization to PEFT while preserving its parameter efficiency (Figure~\ref{fig:moe_lora_monkeyjump}(c)). Standard PEFT attaches an adapter (e.g., for LoRA, $\trainable{\Delta W} = \trainable{BA}$) to each projection (Q, K, V, O, up, gate, down), applying all of them uniformly to every token. MJ routes each token to a subset of these projection adapters based on representation similarity, enabling natural specialization.

MJ works in three stages. \textbf{(i) Before training:} We run $k$-means on token representations from a data subset (§\ref{ab:cluster_data}) to initialize cluster centers—one per projection. \textbf{(ii) During training:} Each token computes cosine similarity to all centers (§\ref{app:ablation-similarity}). The centers are updated gradually via EMA to track how token patterns change—no gradients needed (§\ref{app:ablation-ema}). \textbf{(iii) Expert selection:} Based on similarity scores, each token activates the top-$k$ most similar adapters and skips the rest. Only activated adapters contribute to the output. This lets tokens ``jump'' between adapters based on content, hence the name.

The trainable parameter count remains \emph{exactly the same} as standard PEFT. This is because MJ introduces no new trainable components: (i) existing adapters serve as implicit experts—no additional expert parameters are added; (ii) routing centers are non-trainable buffers updated via EMA, not learned routers optimized by gradients. The only additions are the cluster centers ($O(Ed)$ per block), which are negligible compared to model size and do not participate in backpropagation. This results in comparable model size, memory, and trainable parameters to standard PEFT (§\ref{app:ablation-complexity}).

To motivate why routing helps, we conduct a preliminary experiment comparing two settings using SmolLM-360M with LoRA ($r{=}1$) applied to Q, K, V projections. In the \textbf{shared adapter} setting, we train a single set of adapters on all three GLUE tasks jointly—every task uses the same adapter weights, similar to standard process of LoRA. In the \textbf{task-specific adapter} setting, we train separate adapters for each task—each task gets its own dedicated adapter weights, but the total parameter count remains identical (we partition the adapters across tasks). As shown in Table~\ref{tab:task_specialization}, task-specific adapters consistently outperform shared adapters across all three tasks (+1.19\% on SST-2, +0.61\% on CoLA, +1.21\% on MRPC). This suggests that different tasks benefit from different adapter configurations, and routing tokens to specialized adapters—rather than applying the same adapter uniformly—can improve performance without adding parameters.

\begin{table}[t]
\centering
\small
\setlength{\tabcolsep}{3pt}
\resizebox{0.90\linewidth}{!}{%
\begin{tabular}{@{}lccc|ccc@{}}
\toprule
\textbf{Setting} 
& \multicolumn{3}{c|}{\cellcolor{SharedCol}\textbf{Shared Adapters}} 
& \multicolumn{3}{c}{\cellcolor{TaskCol}\textbf{Task-Specific Adapters}} \\
\midrule
\textbf{Task} 
& \cellcolor{SharedCol}SST-2 
& \cellcolor{SharedCol}CoLA 
& \cellcolor{SharedCol}MRPC 
& \cellcolor{TaskCol}SST-2 
& \cellcolor{TaskCol}CoLA 
& \cellcolor{TaskCol}MRPC \\
\midrule
\textbf{Accuracy} 
& \cellcolor{SharedCol}88.28 
& \cellcolor{SharedCol}60.21 
& \cellcolor{SharedCol}84.62 
& \cellcolor{TaskCol}\textbf{89.47} 
& \cellcolor{TaskCol}\textbf{60.82} 
& \cellcolor{TaskCol}\textbf{85.83} \\
\bottomrule
\end{tabular}
}
\caption{Shared vs.\ task-specific adapters on GLUE tasks (SST2, CoLA, MRPC). \textbf{Shared}: one adapter trained on all tasks jointly. \textbf{Task-specific}: separate adapters per task (same total parameters). Specialization improves accuracy without adding parameters.}
\label{tab:task_specialization}
\vspace{-0.5cm}
\end{table}

\paragraph{Contributions.}
\textbf{(i)} We propose \textbf{Monkey Jump}, a gradient-free routing mechanism that achieves MoE-style specialization by treating existing PEFT adapters as implicit experts. The routing uses $k$-means initialization and EMA-updated centers—no learned routers, no additional parameters.
\textbf{(ii)} We provide \textbf{theoretical analysis} showing that routing increases expressivity (Theorem~\ref{thm:expressivity}) and identifying the last token as the sequence representation for the sequence-wise routing variant (Theorem~\ref{thm:last-token}).
\textbf{(iii)} We demonstrate \textbf{competitive accuracy with superior efficiency} across 47 benchmarks: MJ matches MoE-PEFT performance using $5.3$--$9.2\times$ fewer parameters than the MoE variant of the same base method, 48\% lower memory, and 2$\times$ faster training. \textbf{(iv)} We show that MJ is a \textbf{general recipe}: any adapter-based PEFT method can be converted to a mixture-of-experts model by adding routing that introduces no trainable parameters.

\section{Preliminaries}
\label{sec:prelim}

We consider a Transformer with $L$ blocks. Each block receives a sequence 
of token representations $H = [h_1, \ldots, h_T] \in \mathbb{R}^{d \times T}$ 
and transforms them through a set of linear  projections 
$\mathcal{S} = \{\texttt{q}, \texttt{k}, \texttt{v}, \texttt{o}, 
\texttt{up}, \texttt{gate}, \texttt{down}\}$, corresponding to the 
attention projections (query, key, value, output) and feed-forward layers 
(up, gate, down). Each  projection $s \in \mathcal{S}$ applies a linear 
transformation $\frozen{W_s} \in \mathbb{R}^{d_{\mathrm{out}} \times d_{\mathrm{in}}}$ 
to its input.

\textbf{PEFT.}
PEFT adapts pretrained models by freezing the original weights $\frozen{W_s}$ 
and introducing a small trainable adapter $\trainable{\Delta W_s}$ for each 
 projection. For a token $h_t$, the transformation becomes
{\setlength{\abovedisplayskip}{0pt}
 \setlength{\belowdisplayskip}{0pt}
\begin{equation}
\label{eq:peft}
y_t = \frozen{W_s}\, h_t + \trainable{\Delta W_s}\, h_t.
\end{equation}}
In LoRA~\citep{hu2022lora}, the adapter is parameterized as a low-rank product 
$\trainable{\Delta W_s} = \trainable{B_s}\trainable{A_s}$ with 
$\trainable{A_s} \in \mathbb{R}^{r \times d_{\mathrm{in}}}$ and 
$\trainable{B_s} \in \mathbb{R}^{d_{\mathrm{out}} \times r}$, where 
$r \ll \min(d_{\mathrm{in}}, d_{\mathrm{out}})$ keeps the parameter count small. 
Standard PEFT applies every adapter uniformly—each token $h_t$ receives the 
same correction $\trainable{\Delta W_s}\, h_t$ regardless of its content.

\textbf{MoE.}
Mixture-of-experts architectures achieve input-dependent computation by 
maintaining multiple expert networks $\{f_1, \ldots, f_N\}$ and a router 
$g(\cdot)$ that selects which experts to apply. The output is a weighted 
combination $y = \sum_{n=1}^{N} g_n(x) \cdot f_n(x)$, where the router 
weights satisfy $\sum_n g_n(x) = 1$ and are typically sparse via top-$k$ 
selection. While MoE enables specialization, it multiplies parameters by 
$N$ and requires a learned router with $O(Nd)$ additional parameters.

\textbf{MoE-PEFT.}
Recent work combines these ideas by instantiating $N$ adapter experts per 
 projection (Figure~\ref{fig:moe_lora_monkeyjump}(a)). For example, MoE-LoRA maintains 
$N$ adapter pairs $\{(\trainable{A_s^{(n)}}, \trainable{B_s^{(n)}})\}_{n=1}^N$ 
and computes
{\setlength{\abovedisplayskip}{0pt}
 \setlength{\belowdisplayskip}{0pt}
\begin{equation}
\label{eq:moe-peft}
y_{t,s} = \frozen{W_s}\, h_t + \sum_{n=1}^{N} g_n(h_t) \cdot 
\trainable{B_s^{(n)}}\trainable{A_s^{(n)}}\, h_t.
\end{equation}}
This achieves token-wise specialization but increases trainable parameters 
by a factor of $N$ and adds a learned router—undermining the efficiency 
that motivated PEFT. 

\textbf{Notation.} We use color-coded notation to distinguish parameter types: $\frozen{W}$ (frozen), $\trainable{\Delta W}$ (trainable via gradient descent), and $\ema{C}$ (updated via EMA). We denote $\ell_2$ norm by $\|\cdot\|$.


\section{Monkey Jump}
\label{sec:method}

Monkey Jump (MJ) brings MoE-style specialization to PEFT while preserving the 
parameter budget of standard PEFT. The core idea is simple: each Transformer 
block already contains multiple adapters—one per  projection (query, key, value, 
etc.)—and these adapters can serve as \emph{implicit experts}. Rather than 
applying all adapters uniformly to every token, MJ lets each token select a 
subset of adapters based on its content.

Concretely, consider a block with $E = |\mathcal{S}|$ adapter  projections, 
indexed by $e \in \{1, \ldots, E\}$. In standard PEFT, every token $h_t$ 
receives the full contribution from every adapter. In MJ, we introduce a 
routing coefficient $m_{t,e} \geq 0$ for each token- projection pair that 
modulates the adapter's contribution:
{\setlength{\abovedisplayskip}{0pt}
 \setlength{\belowdisplayskip}{0pt}
\begin{equation}
\label{eq:routed-output}
y_{t,e} = \frozen{W_e}\, h_t + m_{t,e} \cdot \trainable{\Delta W_e}\, h_t.
\end{equation}}
When $m_{t,e} = 0$, adapter $e$ has no effect on token $t$; when $m_{t,e} > 0$, 
it contributes proportionally. The frozen base transformation $\frozen{W_e}\, h_t$ 
is always applied—only the adapter contribution is gated. To maintain sparsity, 
each token activates at most $k$ out of $E$ adapters.
\subsection{Routing}
\label{sec:routing}

\begin{notebox}
\textit{\textbf{Question:} How can we route tokens to adapters without adding trainable parameters?}
\end{notebox}

In standard MoE, a learned router network maps inputs to expert weights, 
introducing $O(Ed)$ trainable parameters per block. This conflicts with 
our goal of parameter efficiency. Instead, we observe that tokens requiring 
similar adaptations tend to have similar representations~\citep{bal2025grasp}, 
motivating a \emph{clustering-based} approach: we group tokens by 
representation similarity and route each group to the same adapters.

Each block maintains $E$ routing centers $\ema{C} = [\ema{c}_1, \ldots, \ema{c}_E] 
\in \mathbb{R}^{E \times d}$, one per  projection. These centers are non-trainable 
buffers updated via EMA—no gradients flow through them~\citep{cai2021exponential}. 
For each token $h_t$, we compute cosine similarity to each center, convert to 
probabilities via softmax, and select the top-$k$ projections:
{\setlength{\abovedisplayskip}{0pt}
 \setlength{\belowdisplayskip}{0pt}
\begin{equation}
\label{eq:routing}
\begin{aligned}
z_{t,e} &= \frac{1}{\tau} \left\langle \frac{h_t}{\|h_t\|_2}, 
\frac{\ema{c}_e}{\|\ema{c}_e\|_2} \right\rangle + \epsilon_{t,e}, \\[2pt]
p_{t,e} &= \frac{\exp(z_{t,e})}{\sum_{e'=1}^{E} \exp(z_{t,e'})}, \\[2pt]
\mathcal{T}_t &= \mathrm{TopK}(p_t, k), \\[2pt]
m_{t,e} &= \frac{p_{t,e}}{\sum_{e' \in \mathcal{T}_t} p_{t,e'}} \cdot 
\mathbb{I}[e \in \mathcal{T}_t],
\end{aligned}
\end{equation}}
where $\tau > 0$ is a temperature controlling routing sharpness and 
$\mathcal{T}_t$ is the set of selected projections. The term 
$\epsilon_{t,e} \sim \mathcal{U}(-\eta, \eta)$ with $\eta = 0.1$ is jitter noise 
applied during training only ($\epsilon_{t,e} = 0$ at inference), which 
encourages exploration across centers and prevents premature commitment to a 
single adapter. The selected coefficients are renormalized to sum to one, so the 
total adapter contribution to a token does not depend on how sharply the 
similarities are distributed. This mechanism routes similar tokens to the same 
adapters without any learned parameters.
\subsection{Center Initialization and Online Updates}
\label{sec:centers}

The effectiveness of routing depends critically on the quality of centers $\ema{C}$. 
Poor initialization leads to arbitrary routing decisions (Figure~\ref{fig:kmneas_init}(c)) 
that provide no benefit over uniform application. We address 
this in two stages.

\paragraph{Initialization.}
Before training, we initialize centers using $k$-means clustering~\citep{arthur2006k}. 
We randomly sample a subset of training examples, perform a forward pass through the 
frozen backbone to collect token representations, and run $k$-means with cosine 
similarity on the $L_2$-normalized representations:
{\setlength{\abovedisplayskip}{0pt}
 \setlength{\belowdisplayskip}{0pt}
\begin{equation}
\label{eq:kmeans-init}
\ema{C} \leftarrow \mathrm{KMeans}\!\left(
\left\{ \frac{h_t}{\|h_t\|_2} \right\}_{t \in \mathcal{D}_{\mathrm{init}}},\, E
\right),
\end{equation}}
where $\mathcal{D}_{\mathrm{init}}$ is the initialization subset. This ensures 
that each center corresponds to a distinct cluster in the representation space 
from the first training iteration (§\ref{ab:intitialization}, §\ref{app:ablation-kmeans}).

\paragraph{Online updates.}
During training, centers are updated via EMA, 
entirely outside the gradient path:
{\setlength{\abovedisplayskip}{0pt}
 \setlength{\belowdisplayskip}{0pt}
\begin{equation}
\label{eq:ema}
\ema{c}_e \leftarrow \beta\,\ema{c}_e + (1-\beta)\,\bar{h}_e, \quad
\bar{h}_e = \frac{1}{|\mathcal{B}_e|} \sum_{t \in \mathcal{B}_e} h_t,
\end{equation}}
where $\mathcal{B}_e$ is the set of tokens routed to adapter $e$ in the current 
batch and $\beta \in (0,1)$ is the momentum coefficient. If $|\mathcal{B}_e| = 0$, 
the center remains unchanged. This allows centers to track the evolving token 
distribution as adapters are trained.

\begin{notebox}
\textit{\textbf{Important:} K-means initialization matters. Replacing it with random 
centers costs 0.5--1.0\% on most tasks, as centers land in sparse or overlapping 
regions of the representation space (\S\ref{app:ablation-kmeans}).}
\end{notebox}
\subsection{Routing Variants}
\label{sec:variants}

The routing mechanism can be instantiated at different levels of granularity, 
trading off specialization against computational cost.

\paragraph{Token-wise routing.}
Each token $h_t$ is routed independently based on its own representation, 
as described in Equation~\ref{eq:routing}. This provides fine-grained, 
per-token specialization and is the default mode.

\paragraph{Sequence-wise routing.}
Sequence-wise routing follows the same mechanism as token-wise routing, with one difference: instead of routing each token independently, all tokens in a sequence share the same routing decision. The routing is computed using the last token representation $h_T$, and the resulting coefficients $m_e$ are applied uniformly to all tokens in the sequence.

\begin{notebox}
\textit{\textbf{Design note:} We use the last token $h_T$ for sequence-wise routing: it is the most informative single position in a causal Transformer (Theorem~\ref{thm:last-token}) and requires no pooling or extra parameters.}
\end{notebox}

\paragraph{Shared adapter.}
Optionally, multiple projections can be designated as always-active ($m_{t,e^*} = 1$ for all $t$), providing stable global adaptation alongside routed adapters. These shared adapters contribute to every token regardless of routing decisions. We find that FFN projections work best as shared adapters—in our experiments, we use O and gate as shared adapters while Q, K, V participate in routing (§\ref{app:ablation-shared}).
\subsection{Computational Cost}
\label{sec:cost}

MJ exactly preserves the trainable parameter count of the underlying PEFT 
method—no new learned weights are added. The only additions are: \textbf{(i) Routing centers:} non-trainable buffers of size $O(Ed)$ per block. \textbf{(ii) Routing computation:} $O(TEd)$ operations for similarity and top-$k$ selection. Both are negligible compared to the adapter forward pass ($O(TEdr)$) or attention ($O(T^2d)$) (§~\ref{sec:efficiency}, §~\ref{app:ablation-complexity}).

\begin{notebox}
\textit{\textbf{Monkey Jump = Standard PEFT adapters ($\trainable{\Delta W}$) + 
routing with no trainable parameters ($\ema{C}$).} Same trainable parameters, better specialization.}
\end{notebox}
\section{Theoretical Analysis}
\label{sec:theory}
We analyze two design choices in MJ. First, we ask why selecting a subset of
adapters per token can be better than applying all of them. We answer this in a simplified setting where adapters share one input and one output space, which isolates the effect of interest (Theorem~\ref{thm:expressivity}). Second, for the sequence-wise variant of \S\ref{sec:variants}, we ask which representation should drive the routing decision, and show that the last token dominates pooling in causal Transformers under the assumptions we state
(Theorem~\ref{thm:last-token}).
\paragraph{Expressivity of Token-wise Routing.}
\label{sec:expressivity}
Why does routing help? When adapters are applied uniformly, their effects are
combined for every token. If two adapters make opposing corrections along some
dimension, these corrections cancel and expressivity is lost. Routing avoids
this: by sending different tokens through different adapters, each adapter's
full effect is preserved for the tokens that need it.

To make this precise we work with a single adaptation site carrying $E$
adapters that share one input and one output space, so that their
contributions combine additively. This is the minimal setting in which
``apply every adapter'' and ``apply a token-dependent subset'' are directly
comparable, and it lets us measure the difference by \emph{output rank}: how
many independent directions the adapter contributions can span. For each
adapter $\Delta W_e$, let $\mathcal{C}_e := \mathrm{Col}(\Delta W_e)$ be its
column space and $\mathcal{C}_{\mathrm{all}} := \sum_{e=1}^{E}\mathcal{C}_e$
the sum of all column spaces.

Uniform application gives every token the summed correction, with aggregate
$U^{\mathrm{uni}} = \big(\sum_e \Delta W_e\big) H$. This lies in
$\mathrm{Col}(\sum_e \Delta W_e)$---the column space of the \emph{summed}
matrix---which can be strictly smaller than $\mathcal{C}_{\mathrm{all}}$
because directions from different adapters can cancel inside the sum. Hard
routing (top-1) instead assigns each token to exactly one adapter, giving
$U^{\mathrm{route}} = [\Delta W_1 H_1 \; \cdots \; \Delta W_E H_E]$, where
$H_e$ contains the tokens routed to adapter $e$. A horizontal concatenation
has column space $\sum_e \mathrm{Col}(\Delta W_e H_e)$, which is potentially
much larger.

\begin{theorembox}
\begin{theorem}[Expressivity of routing]
\label{thm:expressivity}
For adapters sharing one input and output space, under hard routing, if all
adapters are activated and receive sufficiently diverse inputs (i.e.,
$\mathrm{rank}(\Delta W_e H_e) = \mathrm{rank}(\Delta W_e)$ for all $e$), then
{\setlength{\abovedisplayskip}{1pt}
 \setlength{\belowdisplayskip}{0pt}
\[
\mathrm{rank}(U^{\mathrm{route}}) \;\geq\; \mathrm{rank}(U^{\mathrm{uni}}),
\]}
with strict inequality whenever
$\mathrm{Col}(\sum_e \Delta W_e) \subsetneq \sum_e \mathcal{C}_e$.
\end{theorem}
\end{theorembox}

The proof is in Appendix~\ref{app:proof}; the extension to soft top-$k$
routing is in Appendix~\ref{app:soft-routing}.

\paragraph{Example.}
Consider two rank-1 adapters:
\[
\Delta W_1 = \begin{bmatrix} 1 & 0 \\ 0 & 0 \end{bmatrix}, \qquad
\Delta W_2 = \begin{bmatrix} 0 & 0 \\ -1 & 0 \end{bmatrix}.
\]
Adapter 1 produces outputs along $e_1 = (1,0)^\top$; adapter 2 along
$e_2 = (0,1)^\top$. Together, $\mathcal{C}_1 + \mathcal{C}_2 = \mathbb{R}^2$.
Their sum $\Delta W_1 + \Delta W_2 = \left[\begin{smallmatrix} 1 & 0 \\ -1 & 0 \end{smallmatrix}\right]$,
however, produces outputs only along $(1,-1)^\top$---a 1D subspace. With input
$H = [e_1 \; e_1]$, uniform application gives
$U^{\mathrm{uni}} = \left[\begin{smallmatrix} 1 & 1 \\ -1 & -1 \end{smallmatrix}\right]$
(rank 1), while routing gives
$U^{\mathrm{route}} = \left[\begin{smallmatrix} 1 & 0 \\ 0 & -1 \end{smallmatrix}\right]$
(rank 2). Keeping the two contributions separate doubles the effective rank.

\begin{notebox}
\textit{\textbf{Key insight:} When adapters share an output, summing them can
cancel directions that each contributes alone; routing keeps them separate.
This is a property of how adapter effects combine, not of any particular
architecture---in a Transformer block the adapters sit at different
projections, so Theorem~\ref{thm:expressivity} motivates routing rather than
predicting its gain. }
\end{notebox}

\paragraph{Optimality of Last-Token Routing.}
\label{sec:last-token-theory} For sequence-wise routing, all tokens share the same routing decision based 
on a single sequence representation. Common choices include mean pooling 
($\bar{h} = \frac{1}{T}\sum_t h_t$), max pooling, or the last token ($h_T$). 
We show that in causal Transformers, the last token is theoretically optimal.

In causal attention, each token $h_t$ can only attend to positions $1, \ldots, t$. 
This creates an information asymmetry: early tokens have limited context, while 
later tokens accumulate information from the entire prefix. We formalize this 
using mutual information.

\begin{theorembox}
\begin{theorem}[Information Maximality]
\label{thm:last-token}
Let $X = (x_1, \ldots, x_T)$ be an input sequence and $h_1, \ldots, h_T$ be 
the hidden representations at any layer of a causal Transformer. Under the Information Preservation and Attention Aggregation properties: \textbf{(i) Monotonicity:} $I(h_t; X) \leq I(h_{t+1}; X)$ for all $t < T$. \textbf{(ii) Maximality:} $I(h_T; X) \geq I(h_t; X)$ for all $t \leq T$. \textbf{(iii) Dominance over pooling.} If additionally
$I(h_1,\dots,h_{T-1};X \mid h_T) = 0$, then
$I(h_T;X) \ge I(\bar{h};X)$, where $\bar{h} = \frac{1}{T}\sum_{t=1}^{T} h_t$.
\end{theorem}
\end{theorembox}
\begin{proof}[Proof sketch]
(i) Since $h_t = f_t(x_1, \ldots, x_t)$ is a deterministic function of the causal prefix, $I(h_t; X) = I(h_t; X_{1:t}) \leq H(X_{1:t})$ by the data processing inequality. The Information Preservation property bounds the representation loss $\epsilon_t := H(X_{1:t}) - I(h_t; X_{1:t})$ so that added context in $x_{t+1}$ is not offset by extra loss, giving $I(h_t; X) \leq I(h_{t+1}; X)$. (ii) Follows from (i) by induction. (iii) $\bar{h}$ is a deterministic function of $(h_1, \ldots, h_T)$, so $I(\bar{h}; X) \leq I(h_1, \ldots, h_T; X)$ by DPI; the Attention Aggregation property gives $I(h_1, \ldots, h_T; X) = I(h_T; X)$ (Full proof in §~\ref{app:last-token-proof} and empirical evidence~\ref{app:ablation-linearprobe}).
\end{proof}

\begin{notebox}
\textit{\textbf{Finding:} Later positions carry more sequence information than earlier ones (Theorem~\ref{thm:last-token}, i--ii). Part (iii), dominance over pooling, is conditional on Assumption~\ref{assume:attention-agg}.}
\end{notebox}

\section{Experiments}
\label{sec:experiments}

We evaluate MJ on multi-task benchmarks spanning text, image, and video. Our experiments address: (i) Is MJ competitive with MoE-PEFT methods while using fewer parameters? (ii) Does MJ improve efficiency? (iii) How do design choices affect performance?

We evaluate on 47 benchmarks across three modalities: \textbf{Text} (14 tasks, 98K samples), \textbf{Image} (14 tasks, 42K samples), and \textbf{Video} (19 tasks, 13K samples). Details are in §\ref{app:datasets} and Table~\ref{tab:datasets_detailed}. For image/video tasks, we use LLaVA-OneVision-Qwen2-7B~\cite{li2024llava}; for text tasks, we use Llama-3-8B-Instruct~\cite{grattafiori2024llama}. We apply PEFT or MoE-PEFT to projections Q, K, V, O, and gate. Ablations use LLaVA-OneVision-Qwen2-0.5B with adapters applied only to attention projections (Q, K, V, O). We compare against standard PEFT methods (LoRA~\cite{hu2022lora}, LoRA-FA~\cite{zhang2023lora}, AdaLoRA~\cite{zhang2023adalora}, Propulsion~\cite{kowsher2025propulsion}) and MoE-PEFT methods (MoELoRA~\cite{luo2024moelora}, HydraLoRA~\cite{tian2024hydralora}, MoLA~\cite{gao2024higher}, MoRE~\cite{zhang2025more}, MoA~\cite{cao2025moa}). We implement four MJ variants: MJLoRA, MJLoRA-FA, MJAdaLoRA, and MJPropulsion. Full hyperparameters are in Appendix~\ref{app:hyperparameters}.
\begin{table*}[t]
\centering
\scriptsize
\setlength{\tabcolsep}{3pt}
\renewcommand{\arraystretch}{1.05}
\begin{tabular}{lccccccc}
\toprule
\rowcolor{headercolor}
\textbf{Method} &
\textbf{GLUE} &
\textbf{CS \& QA} &
\textbf{ImgCls} &
\textbf{VLQA} &
\textbf{ActObj} &
\textbf{Motion} &
\textbf{HighLvl} \\
\midrule
LoRA~\cite{hu2022lora} 
& $89.58_{\pm 0.56}$ & $71.35_{\pm 1.21}$ & $94.09_{\pm 0.48}$ & $71.52_{\pm 1.97}$ & $45.82_{\pm 1.17}$ & $52.68_{\pm 1.16}$ & $56.63_{\pm 1.22}$ \\
AdaLoRA~\cite{zhang2023adalora}
& $89.15_{\pm 0.68}$ & $71.01_{\pm 1.06}$ & $93.53_{\pm 0.61}$ & $71.04_{\pm 1.93}$ & $45.27_{\pm 1.21}$ & $51.89_{\pm 1.17}$ & $55.83_{\pm 1.24}$ \\
Propulsion~\cite{kowsher2025propulsion}
& $88.97_{\pm 0.53}$ & $70.97_{\pm 1.16}$ & $93.63_{\pm 0.47}$ & $71.24_{\pm 1.85}$ & $45.10_{\pm 1.19}$ & $51.59_{\pm 1.17}$ & $55.55_{\pm 1.14}$ \\
LoRA-FA~\cite{zhang2023lora}
& $89.06_{\pm 0.50}$ & $70.86_{\pm 1.12}$ & $93.62_{\pm 0.49}$ & $71.06_{\pm 1.91}$ & $45.13_{\pm 1.22}$ & $51.70_{\pm 1.14}$ & $55.73_{\pm 1.29}$ \\
MoELoRA~\cite{luo2024moelora}
& $89.64_{\pm 0.70}$ & $71.31_{\pm 1.13}$ & $94.30_{\pm 0.53}$ & $71.79_{\pm 1.78}$ & $45.91_{\pm 1.20}$ & $52.88_{\pm 1.07}$ & $56.83_{\pm 1.14}$ \\
MixLoRA~\cite{li2024mixlora}
& $89.38_{\pm 0.53}$ & $71.59_{\pm 1.07}$ & $94.64_{\pm 0.54}$ & $72.41_{\pm 1.69}$ & $47.31_{\pm 1.11}$ & $54.17_{\pm 1.15}$ & $58.51_{\pm 1.17}$ \\
HydraLoRA~\cite{tian2024hydralora}
& $89.75_{\pm 0.46}$ & $71.73_{\pm 1.09}$ & $\mathbf{94.86}_{\pm 0.48}$ & $71.85_{\pm 1.80}$ & $47.85_{\pm 1.15}$ & $54.96_{\pm 1.21}$ & $59.05_{\pm 1.14}$ \\
MoLA~\cite{gao2024higher}
& $89.74_{\pm 0.59}$ & $71.30_{\pm 1.07}$ & $94.03_{\pm 0.53}$ & $\mathbf{72.54}_{\pm 1.73}$ & $45.78_{\pm 1.20}$ & $52.49_{\pm 1.08}$ & $\mathbf{59.10}_{\pm 1.12}$ \\
MoRE~\cite{zhang2025more}
& $89.72_{\pm 0.53}$ & $71.62_{\pm 1.12}$ & $94.74_{\pm 0.51}$ & $72.47_{\pm 1.69}$ & $47.81_{\pm 1.13}$ & $54.64_{\pm 1.12}$ & $58.49_{\pm 1.23}$ \\
MoA~\cite{cao2025moa}
& $89.57_{\pm 0.55}$ & $71.46_{\pm 1.08}$ & $94.46_{\pm 0.39}$ & $71.68_{\pm 1.75}$ & $46.35_{\pm 1.13}$ & $53.23_{\pm 1.17}$ & $57.16_{\pm 1.12}$ \\
\midrule
\rowcolor{ourmethodcolor}
MJLoRA
& $\mathbf{89.91}_{\pm 0.50}$ & $71.73_{\pm 0.98}$ & $93.39_{\pm 0.44}$ & $72.01_{\pm 1.83}$ & $46.99_{\pm 1.11}$ & $\mathbf{55.02}_{\pm 1.22}$ & $57.94_{\pm 1.10}$ \\
\rowcolor{ourmethodcolor}
MJAdaLoRA
& $89.78_{\pm 0.56}$ & $71.46_{\pm 1.03}$ & $93.84_{\pm 0.47}$ & $71.44_{\pm 1.83}$ & $\mathbf{48.06}_{\pm 1.17}$ & $54.49_{\pm 1.19}$ & $58.67_{\pm 1.13}$ \\
\rowcolor{ourmethodcolor}
MJPropulsion
& $89.80_{\pm 0.56}$ & $71.40_{\pm 1.04}$ & $93.51_{\pm 0.67}$ & $72.48_{\pm 1.79}$ & $47.46_{\pm 1.21}$ & $54.83_{\pm 1.17}$ & $59.02_{\pm 1.13}$ \\
\rowcolor{ourmethodcolor}
MJLoRA-FA
& $89.90_{\pm 0.52}$ & $\mathbf{71.95}_{\pm 1.05}$ & $94.03_{\pm 0.65}$ & $71.38_{\pm 1.83}$ & $47.42_{\pm 1.19}$ & $54.07_{\pm 1.15}$ & $56.37_{\pm 1.18}$ \\
\bottomrule
\end{tabular}
\caption{Average performance across task families (mean $\pm$ std over 5 runs). Columns: GLUE, Commonsense \& QA, Image Classification, Vision--Language QA, Action \& Object-Centric Reasoning, Motion \& Scene Understanding, and High-Level Reasoning. Highlighted rows denote MJ variants. Per-task results in Tables~\ref{tab:glue_results}, \ref{tab:qa_results}, \ref{tab:image_cls}, \ref{tab:vqa}, \ref{tab:action_object_results}, \ref{tab:peft_motion_scene}, \ref{tab:video_reasoning}.}
\label{tab:avg_all_tasks}
\end{table*}
Table~\ref{tab:avg_all_tasks} summarizes average performance across task families. Per-task results are in Appendix~\ref{app:detailed-results}.
MJ variants achieve competitive performance with MoE-PEFT methods while using substantially fewer trainable parameters. On GLUE, MJLoRA achieves the best score (89.91\%), edging out HydraLoRA (89.75\%) and MoLA (89.74\%). On CS\&QA, MJLoRA-FA achieves the highest accuracy (71.95\%), outperforming HydraLoRA (71.73\%) and MoRE (71.62\%). For Action\&Object reasoning, MJAdaLoRA leads all methods (48.06\%), ahead of HydraLoRA (47.85\%) and MoRE (47.81\%). On Motion understanding, MJLoRA achieves the best score (55.02\%), while MJPropulsion shows strong results on High-Level Reasoning (59.02\%), trailing only MoLA (59.10\%) and HydraLoRA (59.05\%).
Compared to their base PEFT methods, MJ variants show consistent improvements: MJLoRA outperforms LoRA by +0.33\% on GLUE, +0.38\% on CS\&QA, +1.17\% on ActObj, +2.34\% on Motion, and +1.31\% on HighLvl---demonstrating that gradient-free routing provides meaningful specialization at no additional parameter cost.


\begin{figure*}[!htb]
    \centering
    \includegraphics[width=1.00\linewidth]{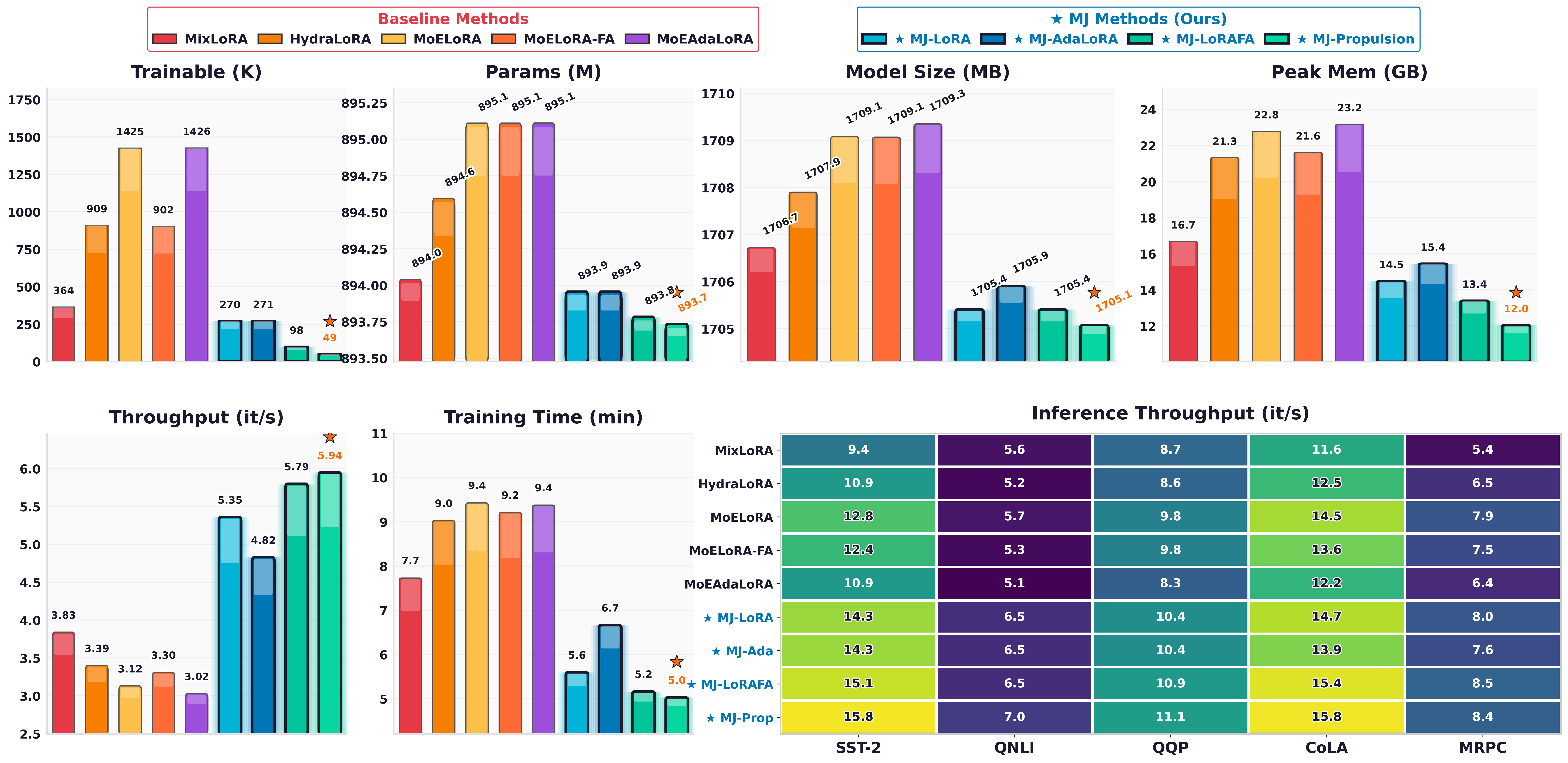}
    \caption{Efficiency comparison. \textbf{Top row:} Trainable parameters (K), total parameters (M), model size (MB), and peak GPU memory (GB). \textbf{Bottom row:} Training throughput (it/s = iterations per second), training time (min), and inference throughput across GLUE tasks. }
    \label{fig:efficiency}
\end{figure*}

\subsection{Efficiency Analysis}
\label{sec:efficiency}

A core goal of MJ is achieving MoE-style specialization without sacrificing PEFT efficiency. We evaluate using LLaVA-OneVision-Qwen2-0.5B with rank 2, applying MoE-based PEFT to attention projections (Q, K, V, O). For fair comparison, all methods use the same environment: H100 GPU, Transformers library, PyTorch, batch size 8, and gradient accumulation 2. Figure~\ref{fig:efficiency} compares MJ variants against MoE-PEFT baselines across six metrics.

\paragraph{Parameter efficiency.}
We report parameter savings against the MoE variant of the same base PEFT method, which isolates the reduction attributable to MJ's routing mechanism rather than to the choice of base adapter. On this like-for-like basis, MJLoRA (270K) uses $5.3\times$ fewer trainable parameters than MoELoRA (1{,}425K), MJAdaLoRA (271K) $5.3\times$ fewer than MoEAdaLoRA (1{,}426K), and MJLoRA-FA (98K) $9.2\times$ fewer than MoELoRA-FA (902K). The larger ratio for LoRA-FA follows from its asymmetric design: freezing $A$ and training only $B$ already halves the adapter cost, so introducing $N$ experts multiplies an already-reduced count. In cross-method comparison, MJPropulsion (49K) is $29\times$ smaller than
MoELoRA, reflecting Propulsion's smaller base as well as the routing
mechanism. Despite these reductions, total model size remains nearly identical (${\sim}1{,}705$\,MB for MJ vs.\ 1{,}706--1{,}709\,MB for MoE-PEFT), since MJ reuses existing adapters rather than adding new experts.

\paragraph{Memory efficiency.}
MJ reduces peak GPU memory by up to 48\%. MJ-Propulsion uses only 12.0GB compared to 23.2GB for MoEAdaLoRA and 22.8GB for MoELoRA. Even the highest-memory MJ variant (MJ-AdaLoRA, 15.4GB) uses 33\% less than MoEAdaLoRA. This reduction comes from top-$k$ sparse routing—MJ evaluates fewer adapter branches per forward pass.

\paragraph{Training speed.}
MJ achieves 1.5--2$\times$ faster training. MJ-Propulsion reaches 5.94 it/s throughput and completes training in 5.0 minutes, compared to 3.02--3.83 it/s and 7.7--9.4 minutes for MoE-PEFT methods. All MJ variants exceed 4.8 it/s, while no MoE-PEFT method exceeds 3.9 it/s.

\paragraph{Inference speed.}
MJ maintains efficiency at inference. Across GLUE tasks, MJ-Propulsion achieves 15.8 it/s on SST-2 versus 9.4--12.8 it/s for MoE-PEFT. On average, MJ variants achieve 10--25\% higher inference throughput.

\begin{notebox}
\textit{\textbf{Key result:} MJ achieves $5.3$--$9.2\times$ fewer trainable parameters than the MoE variant of the same base method, up to 48\% lower peak memory, and 1.5--2$\times$ faster training—while maintaining comparable accuracy to MoE-PEFT.  (Table~\ref{tab:avg_all_tasks}).}
\end{notebox}

\subsection{Ablation Study}
\label{sec:ablation}
MJ introduces several design choices: how to initialize routing centers, how often to update them, and how many layers to equip with routing.

\begin{figure*}[!htb]
    \centering
    \includegraphics[width=1.00\linewidth]{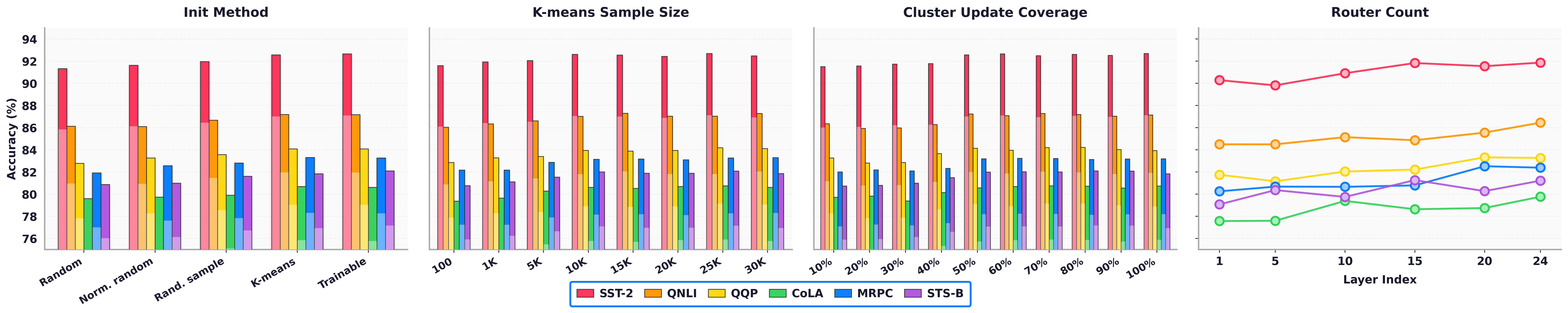} 
    \caption{(a) $K$-means initialization matches trainable routers. (b) More samples improve initialization, saturating at 5K--10K. (c) EMA update coverage of 50--70\% suffices. (d) More routing layers improve performance.}

    \label{fig:router_init}
\end{figure*}

\paragraph{Initialization method.}
Figure~\ref{fig:router_init}(a) compares five center initialization strategies: random vectors, normalized random, random token sampling, $k$-means clustering, and trainable router. $K$-means achieves the best performance among training-free methods and matches the trainable router, confirming that clustering captures meaningful structure without learned parameters.

\paragraph{$K$-means sample size.}\label{ab:intitialization}
Figure~\ref{fig:router_init}(b) varies the number of samples for $k$-means initialization from 100 to 30K randomly. Performance saturates at 5K--10K samples, indicating that a modest sample size suffices for effective initialization.

\paragraph{Cluster update coverage.}\label{ab:cluster_data}
Figure~\ref{fig:router_init}(c) varies the percentage of training steps where EMA updates are applied. Performance plateaus at 50--70\% coverage, suggesting that centers need regular but not continuous updates to track the evolving token distribution.

\paragraph{Router count.}
Figure~\ref{fig:router_init}(d) varies the number of Transformer blocks equipped with MJ routing (1 to 24), with remaining blocks using standard PEFT. Performance improves as more blocks use routing, with 20--24 blocks achieving the best results.

\begin{table*}[t]
\centering
\scriptsize
\setlength{\tabcolsep}{3pt}
\renewcommand{\arraystretch}{0.90}
\resizebox{\textwidth}{!}{%
\begin{tabular}{llcccccccc}
\toprule
\rowcolor{headercolor}
\textbf{Method} & \textbf{Granularity} &
\textbf{BoolQ} & \textbf{PIQA} & \textbf{SIQA} & \textbf{H.Sw.} &
\textbf{W.Gra} & \textbf{ARC-e} & \textbf{ARC-c} & \textbf{OBQA} \\
\midrule
\multirow{3}{*}{MJLoRA} & Token & $71.11_{\pm0.67}$ & $73.73_{\pm0.53}$ & $64.73_{\pm1.41}$ & $51.68_{\pm2.20}$ & $77.93_{\pm0.30}$ & $85.13_{\pm0.66}$ & $71.42_{\pm1.19}$ & $78.14_{\pm0.89}$\\
 & Sentence & $71.00_{\pm0.88}$ & $73.24_{\pm0.74}$ & $65.22_{\pm1.69}$ & $51.22_{\pm2.52}$ & $77.27_{\pm0.61}$ & $84.73_{\pm0.93}$ & $71.45_{\pm1.36}$ & $77.42_{\pm1.18}$\\
 & Task & $71.60_{\pm0.41}$ & $74.03_{\pm0.65}$ & $\mathbf{65.70}_{\pm1.29}$ & $51.50_{\pm2.31}$ & $78.90_{\pm0.48}$ & $86.19_{\pm0.72}$ & $71.48_{\pm1.27}$ & $77.89_{\pm1.06}$ \\
\midrule
\multirow{3}{*}{MJPropulsion} & Token & $71.76_{\pm0.65}$ & $74.32_{\pm0.53}$ & $64.02_{\pm1.41}$ & $50.52_{\pm2.33}$ & $77.18_{\pm0.18}$ & $85.03_{\pm0.80}$ & $70.91_{\pm1.53}$ & $77.46_{\pm0.92}$\\
 & Sentence & $71.20_{\pm0.91}$ & $74.53_{\pm0.88}$ & $63.73_{\pm1.72}$ & $50.49_{\pm2.66}$ & $77.34_{\pm0.47}$ & $84.92_{\pm1.04}$ & $71.10_{\pm1.79}$ & $77.39_{\pm1.21}$\\
 & Task & $72.26_{\pm0.37}$ & $\mathbf{74.62}_{\pm0.79}$ & $64.47_{\pm1.31}$ & $51.03_{\pm2.42}$ & $77.50_{\pm0.52}$ & $84.95_{\pm0.65}$ & $71.33_{\pm1.33}$ & $77.99_{\pm1.14}$ \\
\midrule
\multirow{3}{*}{MJAdaLoRA} & Token & $71.52_{\pm0.22}$ & $73.14_{\pm0.58}$ & $64.15_{\pm1.51}$ & $50.12_{\pm2.55}$ & $78.79_{\pm0.31}$ & $85.08_{\pm0.45}$ & $71.98_{\pm1.67}$ & $76.91_{\pm0.92}$\\
 & Sentence & $71.11_{\pm0.63}$ & $73.67_{\pm0.79}$ & $64.17_{\pm1.83}$ & $50.31_{\pm2.71}$ & $79.37_{\pm0.56}$ & $84.71_{\pm0.88}$ & $71.83_{\pm1.92}$ & $79.57_{\pm1.24}$\\
 & Task & $\mathbf{72.84}_{\pm0.29}$ & $73.74_{\pm0.74}$ & $65.26_{\pm1.37}$ & $49.94_{\pm2.51}$ & $\mathbf{79.73}_{\pm0.61}$ & $\mathbf{86.41}_{\pm0.67}$ & $72.64_{\pm1.39}$ & $\mathbf{79.73}_{\pm1.15}$ \\
\midrule
\multirow{3}{*}{MJLoRA-FA} & Token & $71.38_{\pm0.27}$ & $74.04_{\pm0.62}$ & $64.94_{\pm1.46}$ & $51.55_{\pm2.64}$ & $78.51_{\pm0.35}$ & $85.86_{\pm0.61}$ & $71.39_{\pm1.58}$ & $77.92_{\pm0.90}$\\
 & Sentence & $71.34_{\pm0.71}$ & $73.62_{\pm0.91}$ & $64.82_{\pm1.77}$ & $51.70_{\pm2.85}$ & $78.20_{\pm0.68}$ & $84.87_{\pm0.94}$ & $\mathbf{72.94}_{\pm1.96}$ & $77.77_{\pm1.19}$\\
 & Task & $71.69_{\pm0.34}$ & $74.46_{\pm0.76}$ & $64.64_{\pm1.40}$ & $\mathbf{52.50}_{\pm2.55}$ & $78.91_{\pm0.59}$ & $86.13_{\pm0.71}$ & $71.70_{\pm1.42}$ & $77.70_{\pm1.08}$ \\
\bottomrule
\end{tabular}
}
\caption{\textbf{Token/Sentence}: unsupervised routing via representation similarity. \textbf{Task}: supervised routing using dataset ID (oracle). For task-specific routing, ARC-e and ARC-c share one expert.}
\label{tab:qa_granularity}
\end{table*}

\paragraph{Routing granularity.}
Table~\ref{tab:qa_granularity} compares three routing strategies using Llama-3-8B-Instruct on QA benchmarks. \textbf{Token-wise} and \textbf{sentence-wise} routing are unsupervised (representation-based), while \textbf{task-specific} routing uses known dataset IDs as an oracle.

Task-specific routing performs best on average because each group receives a dedicated expert that specializes without interference. Token-wise routing closely matches the oracle with an average gap of only 0.53\% across the four MJ variants, despite having no task labels—and even outperforms task routing on individual benchmarks (e.g., MJLoRA-FA on SIQA, MJAdaLoRA on H.Sw.). Sentence-wise routing is generally weaker than token-wise routing because a single routing decision cannot capture intra-sequence variation—different tokens (e.g., question vs answer) may benefit from different adapters—though the gap is small and sentence routing wins on some configurations (e.g., MJAdaLoRA OBQA and W.Gra). Overall, this shows that task-relevant structure emerges naturally in token representations, and MJ discovers it through clustering alone.

\begin{notebox}
\textit{\textbf{Finding:} Unsupervised token-wise routing achieves near-oracle performance, demonstrating that MJ learns meaningful specialization from representations without task supervision.}
\end{notebox}

\begin{notebox}
\textit{\textbf{Extended ablations in Appendix}~\ref{app:ablation}:}
\textbf{(i)} Similarity function (§\ref{app:ablation-similarity}); 
\textbf{(ii)} Routing temperature $\tau$ (§\ref{app:ablation-temperature}); 
\textbf{(iii)} EMA smoothing $\beta$ (§\ref{app:ablation-ema}); 
\textbf{(iv)} Update schedule (§\ref{app:ablation-schedule}); 
\textbf{(v)} Projection specialization (§\ref{app:ablation-projection});
\textbf{(vi)} Linear probing validation (§\ref{app:ablation-linearprobe});
\textbf{(vii)} Expert permutation (§\ref{app:ablation-permutation});
\textbf{(viii)} Shared expert (§\ref{app:ablation-shared}); 
\textbf{(ix)} Rank sensitivity (§\ref{app:ablation-rank}); 
\textbf{(x)} Expert combinations (§\ref{app:ablation-combination}); 
\textbf{(xi)} Self-balancing (§\ref{app:ablation-balance}); 
\textbf{(xii)} $K$-means impact (§\ref{app:ablation-kmeans}); 
\textbf{(xiii)} Complexity and Parameter Analysis  (§\ref{app:ablation-complexity});
\textbf{(xiv)} Layer-wise visualization (§\ref{app:ablation-layerwise}).
\end{notebox}
\section{Related Work}
\label{sec:related}

We provide extended discussion of related work in Appendix~\ref{appendix:related}; here we summarize the key connections.

PEFT methods adapt frozen LLMs using lightweight modules such as low-rank adapters~\cite{hu2022lora,houlsby2019parameter,zaken2022bitfit,lester2021power}, but apply the same adapters uniformly to all inputs. MoE-PEFT methods~\cite{dou2024loramoe,luo2024moelora,li2024mixlora,gou2023mixture,liao2025hmora} introduce learned routing for input-dependent specialization, but add trainable router parameters and memory overhead from multi-expert evaluation. MJ bridges these approaches: it achieves MoE-style specialization while preserving the exact parameter budget of standard PEFT through gradient-free clustering-based routing, avoiding the overhead of learned routers and multi-expert execution.

\section{Conclusion}
\label{sec:conclusion}

We presented Monkey Jump, a method that achieves MoE-style specialization in PEFT without adding extra trainable parameters for experts and routing. By treating existing adapters as implicit experts and routing tokens via gradient-free clustering, MJ achieves comparable accuracy to MoE-PEFT while using 5.3--9.2$\times$ fewer trainable parameters than the MoE variant of the same base method, up to 48\% lower memory, and 1.5--2$\times$ faster training, and 
10--25\% higher inference throughput. Across 47 benchmarks spanning text, image, and video tasks, MJ improves over standard PEFT methods on six of seven task families, with image classification the exception. MJ is architecture-agnostic—we demonstrated gains with LoRA, LoRA-FA, AdaLoRA, and Propulsion, showing that any adapter-based PEFT method can benefit from gradient-free routing.

\section{Limitations}
\label{sec:limitations}

While Monkey Jump achieves strong results, several limitations remain:

\paragraph{Fixed expert capacity.} MJ treats each projection adapter as an implicit expert, so the number of experts is fixed at the number of projections (typically maximum 7). This limits specialization compared to standard MoE, which can scale to hundreds of experts. Increasing expert capacity would require adding multiple adapters per projection, reintroducing the parameter overhead MJ is designed to avoid. This trade-off is intentional: MJ prioritizes parameter efficiency over expert scalability.

\paragraph{Clustering assumptions.}
MJ assumes that token representations naturally cluster in meaningful ways that correspond to different adaptation needs. While this holds for the benchmarks we tested, highly complex or heterogeneous data distributions may benefit from learned routing that can capture more nuanced patterns beyond what $k$-means clustering provides.

\paragraph{Initialization overhead.}
MJ requires a $k$-means initialization step before training, involving a forward pass through a data subset. While this adds only a small overhead, it introduces an extra pipeline stage compared to standard PEFT methods.

\paragraph{Hyperparameter sensitivity.}
Although our ablations show MJ is robust to most hyperparameters (§\ref{app:ablation}), performance depends on choices like top-$k$ value, EMA momentum $\beta$, and update schedule. Suboptimal settings can degrade performance, particularly on small datasets.

\bibliography{custom}
\clearpage

\appendix

\onecolumn
\tableofcontents
\twocolumn

\clearpage
\section{Proof of Theorem~\ref{thm:expressivity}}
\label{app:proof}
We work in the setting of \S\ref{sec:expressivity}: a single adaptation site
carrying $E$ adapters that share one input and one output space, so that their
contributions combine additively. Let
$H = [h_1, \ldots, h_T] \in \mathbb{R}^{d \times T}$ denote the input token
representations, and let adapter $e \in \{1, \ldots, E\}$ apply the trainable
correction $\trainable{\Delta W_e}$. The frozen weight is applied in both the
uniform and routed cases and so contributes identically to each; we therefore
omit it and analyze the adapter contributions directly. We assume hard routing
(top-1) throughout: each token is assigned to exactly one adapter.

Uniform application sums all adapter contributions:
\[
U^{\mathrm{uni}} = \sum_{e=1}^{E} \trainable{\Delta W_e} H = \left(\sum_{e=1}^{E} \trainable{\Delta W_e}\right) H.
\]
Routing instead applies a token-dependent subset:
\begin{align*}
U^{\mathrm{route}} &= \sum_{e=1}^{E} \trainable{\Delta W_e} H D_e, \nonumber \\
D_e &= \mathrm{diag}(m_{1,e}, \ldots, m_{T,e}).
\end{align*}
Let $\mathcal{E}_t = \{e : m_{t,e} > 0\}$ denote the set of active experts for token $t$, and define the activated expert set as $\mathcal{A} = \bigcup_{t=1}^{T} \mathcal{E}_t$.

\paragraph{Proof.}
We begin by defining the relevant column spaces. For each expert $e$, define
\[
\mathcal{C}_e := \mathrm{Col}(\trainable{\Delta W_e}) = \{\trainable{\Delta W_e} x : x \in \mathbb{R}^{d_{\mathrm{in}}}\},
\]
which is the set of all output vectors that adapter $e$ can produce. By definition, $\dim(\mathcal{C}_e) = \mathrm{rank}(\trainable{\Delta W_e}) = r_e$. The sum of column spaces is
\[
\mathcal{C}_{\mathrm{all}} := \sum_{e=1}^{E} \mathcal{C}_e,
\]
which contains every vector of the form $v_1 + \cdots + v_E$ with $v_e \in \mathcal{C}_e$.

We first establish the upper bound for uniform application. The output is
\[
U^{\mathrm{uni}} = \left(\sum_{e=1}^{E} \trainable{\Delta W_e}\right) H.
\]
Since $\mathrm{Col}(AB) \subseteq \mathrm{Col}(A)$ for any matrix product,
\[
\mathrm{Col}(U^{\mathrm{uni}}) \subseteq \mathrm{Col}\left(\sum_{e=1}^{E} \trainable{\Delta W_e}\right).
\]
Any column of $\sum_e \trainable{\Delta W_e}$ is the sum of corresponding columns from each $\trainable{\Delta W_e}$, each lying in $\mathcal{C}_e$, so
\[
\mathrm{Col}\left(\sum_{e=1}^{E} \trainable{\Delta W_e}\right) \subseteq \sum_{e=1}^{E} \mathcal{C}_e = \mathcal{C}_{\mathrm{all}}.
\]
This inclusion can be strict: directions in individual column spaces may cancel in the matrix sum. Combining these,
\[
\mathrm{rank}(U^{\mathrm{uni}}) \leq \mathrm{rank}\left(\sum_{e=1}^{E} \trainable{\Delta W_e}\right) \leq \dim(\mathcal{C}_{\mathrm{all}})
\]

Next, we derive the upper bound for routing. The output is
\[
U^{\mathrm{route}} = \sum_{e=1}^{E} \trainable{\Delta W_e} H D_e.
\]
The $t$-th column is $u_t^{\mathrm{route}} = \sum_{e=1}^{E} m_{t,e} \trainable{\Delta W_e} h_t$. Each nonzero term $m_{t,e} \trainable{\Delta W_e} h_t$ lies in $\mathcal{C}_e$, so
\[
u_t^{\mathrm{route}} \in \sum_{e \in \mathcal{E}_t} \mathcal{C}_e \subseteq \sum_{e \in \mathcal{A}} \mathcal{C}_e \subseteq \mathcal{C}_{\mathrm{all}},
\]
where $\mathcal{E}_t = \{e : m_{t,e} > 0\}$ and $\mathcal{A} = \bigcup_t \mathcal{E}_t$ is the set of all activated experts. Therefore,
\[
\mathrm{Col}(U^{\mathrm{route}}) \subseteq \sum_{e \in \mathcal{A}} \mathcal{C}_e,
\]
giving
\[
\mathrm{rank}(U^{\mathrm{route}}) \leq \dim\left(\sum_{e \in \mathcal{A}} \mathcal{C}_e\right) \leq \dim(\mathcal{C}_{\mathrm{all}}).
\]

We now show that routing can achieve this upper bound under favorable conditions. Suppose the routing activates all experts, i.e., $\mathcal{A} = \{1, \ldots, E\}$. For each expert $e$, let $\mathcal{T}_e = \{t : e \in \mathcal{E}_t\}$ be the set of tokens routed to expert $e$, and let $H_e$ be the corresponding submatrix of $H$.

Under hard routing, each token activates exactly one expert, so $\{\mathcal{T}_e\}$ partitions $\{1,\ldots,T\}$. The output becomes
\[
U^{\mathrm{route}} = \big[\trainable{\Delta W_1} H_1, \;\ldots,\; \trainable{\Delta W_E} H_E\big].
\]
The column space of a horizontal concatenation equals the sum of column spaces:
\[
\mathrm{Col}(U^{\mathrm{route}}) = \sum_{e=1}^{E} \mathrm{Col}(\trainable{\Delta W_e} H_e).
\]
If the inputs routed to each expert span the row space of that adapter, then $\mathrm{Col}(\trainable{\Delta W_e} H_e) = \mathcal{C}_e$, and
\[
\mathrm{Col}(U^{\mathrm{route}}) = \sum_{e=1}^{E} \mathcal{C}_e = \mathcal{C}_{\mathrm{all}}.
\]
Thus $\mathrm{rank}(U^{\mathrm{route}}) = \dim(\mathcal{C}_{\mathrm{all}})$, achieving the upper bound.

Combining these results establishes the theorem. Under full activation with diverse inputs, $\mathrm{rank}(U^{\mathrm{route}}) = \dim(\mathcal{C}_{\mathrm{all}})$ while $\mathrm{rank}(U^{\mathrm{uni}}) \leq \dim(\mathcal{C}_{\mathrm{all}})$. Therefore $\mathrm{rank}(U^{\mathrm{route}}) \geq \mathrm{rank}(U^{\mathrm{uni}})$.

The inequality is strict when
\[
\mathrm{Col}\left(\sum_{e=1}^{E} \trainable{\Delta W_e}\right) \subsetneq \mathcal{C}_{\mathrm{all}}.
\]
This occurs when directions in individual column spaces cancel in the matrix sum.

As a concrete example, consider $E=2$ with
\[
\trainable{\Delta W_1} = \begin{bmatrix} 1 & 0 \\ 0 & 0 \end{bmatrix}, \qquad \trainable{\Delta W_2} = \begin{bmatrix} 0 & 0 \\ -1 & 0 \end{bmatrix}.
\]
Then $\mathcal{C}_1 = \mathrm{span}\{e_1\}$, $\mathcal{C}_2 = \mathrm{span}\{e_2\}$, and $\mathcal{C}_{\mathrm{all}} = \mathbb{R}^2$. The matrix sum is
\[
\trainable{\Delta W_1} + \trainable{\Delta W_2} = \begin{bmatrix} 1 & 0 \\ -1 & 0 \end{bmatrix},
\]
with $\mathrm{Col}(\trainable{\Delta W_1} + \trainable{\Delta W_2}) = \mathrm{span}\{(1,-1)^\top\}$, a one-dimensional subspace strictly contained in $\mathcal{C}_{\mathrm{all}} = \mathbb{R}^2$.

With $H = [e_1 \; e_1]$ and hard routing (token 1 to expert 1, token 2 to expert 2):
\[
U^{\mathrm{route}} = \begin{bmatrix} 1 & 0 \\ 0 & -1 \end{bmatrix}, \qquad U^{\mathrm{uni}} = \begin{bmatrix} 1 & 1 \\ -1 & -1 \end{bmatrix}.
\]
Thus $\mathrm{rank}(U^{\mathrm{route}}) = 2 > 1 = \mathrm{rank}(U^{\mathrm{uni}})$. \qed

\clearpage

\section{\texorpdfstring{Extension to Soft Top-$k$ Routing}{Extension to Soft Top-k Routing}}
\label{app:soft-routing}
Theorem~\ref{thm:expressivity} assumes hard routing (top-1) for simplicity.
Here we relax this to soft top-$k$ routing, where each token activates up to
$k$ adapters with weighted coefficients, as in the implemented method. We
remain in the shared-output setting of \S\ref{sec:expressivity} throughout.
Under soft top-$k$ routing, each token $t$ selects $k$ experts $\mathcal{E}_t = \mathrm{TopK}(\{p_{t,e}\}, k)$ and assigns weights $m_{t,e} = \big(p_{t,e} / \sum_{e' \in \mathcal{E}_t} p_{t,e'}\big) \cdot \mathbb{I}[e \in \mathcal{E}_t]$. The effective adapter for token $t$ is
\[
\trainable{\Delta W_t^{\mathrm{eff}}} = \sum_{e \in \mathcal{E}_t} m_{t,e} \trainable{\Delta W_e},
\]
and the routed output column for token $t$ is
\[
u_t^{\mathrm{route}} = \trainable{\Delta W_t^{\mathrm{eff}}} h_t = \sum_{e \in \mathcal{E}_t} m_{t,e} \trainable{\Delta W_e} h_t.
\]
Each output $u_t^{\mathrm{route}}$ lies in $\sum_{e \in \mathcal{E}_t} \mathcal{C}_e$, since it is a linear combination of vectors from the selected adapters' column spaces. The full output satisfies
\[
\mathrm{Col}(U^{\mathrm{route}}) \subseteq \sum_{e \in \mathcal{A}} \mathcal{C}_e,
\]
where $\mathcal{A} = \bigcup_t \mathcal{E}_t$ is the set of all activated experts. This upper bound is identical to the hard routing case.
The key difference from hard routing is that the output $U^{\mathrm{route}}$ is no longer a simple concatenation of per-expert blocks. Instead, each column $u_t^{\mathrm{route}}$ is a weighted combination of up to $k$ adapter outputs.
Let $\mathcal{P} = \{\mathcal{E}_t : t = 1, \ldots, T\}$ denote the set of distinct routing patterns. For each pattern $P \in \mathcal{P}$, define the effective adapter
\[
\trainable{\Delta W_P} = \sum_{e \in P} \bar{m}_e \trainable{\Delta W_e},
\]
where $\bar{m}_e$ represents the average routing weight for expert $e$ within pattern $P$. The achievable rank depends on how many distinct effective adapters arise from different routing patterns.
\begin{proposition}[Soft Routing Expressivity]
\label{prop:soft}
Under soft top-$k$ routing:
\begin{enumerate}[label=(\roman*), leftmargin=2em]
\item The upper bound remains
$\mathrm{rank}(U^{\mathrm{route}}) \leq \dim\!\left(\sum_{e \in \mathcal{A}} \mathcal{C}_e\right)$.
\item If the routing patterns are diverse (different tokens select different expert subsets),
soft routing can still achieve higher rank than uniform application.
\item The maximum achievable rank is
\begin{align*}
\mathrm{rank}(U^{\mathrm{route}})
&\leq \min\Big(\dim(\mathcal{C}_{\mathrm{all}}), \\
&\qquad |\mathcal{P}| \cdot \max_P \mathrm{rank}(\trainable{\Delta W_P})\Big).
\end{align*}
\end{enumerate}
\end{proposition}
Hard routing ($k=1$) maximizes the diversity of effective adapters: each token uses a pure adapter $\trainable{\Delta W_e}$ rather than a blend. This makes achieving the full rank $\dim(\mathcal{C}_{\mathrm{all}})$ straightforward.
Soft routing ($k > 1$) introduces blending, which can reduce diversity. However, if the routing patterns are sufficiently varied and the blending coefficients differ across tokens, soft routing can still achieve high expressivity. In practice, the temperature parameter $\tau$ controls the sharpness of routing: lower $\tau$ yields sharper (more hard-like) routing, while higher $\tau$ yields softer blending.
The theoretical analysis suggests three practical guidelines:
\begin{itemize}
\item \textbf{Lower $k$ increases expressivity:} Fewer experts per token means more distinct routing patterns, closer to the hard routing ideal.
\item \textbf{Lower $\tau$ sharpens routing:} Concentrating probability mass on fewer experts mimics hard routing benefits.
\item \textbf{Diverse routing is key:} The expressivity advantage of Monkey Jump depends on tokens being routed to different experts, which is encouraged by the $k$-means clustering and EMA updates.
\end{itemize}

\clearpage

\section{ Information Maximality}
\label{app:last-token-proof}

We establish that in causal Transformers, the last token's hidden representation is theoretically optimal for sequence-wise routing decisions.

\subsection{Definitions}

\begin{definition}[Entropy]
For a discrete random variable $X$ with probability mass function $p(x)$:
\[
H(X) = -\sum_{x \in \mathcal{X}} p(x) \log p(x)
\]
\end{definition}

\begin{definition}[Conditional Entropy]
For random variables $X$ and $Y$:
\begin{align}
H(Y \mid X) = -\sum_{x \in \mathcal{X}} \sum_{y \in \mathcal{Y}} p(x, y) \log p(y \mid x)
\end{align}
\end{definition}

\begin{definition}[Mutual Information]
For random variables $X$ and $Y$:
\begin{align}
I(X; Y) &= H(X) - H(X \mid Y) \nonumber \\
&= H(Y) - H(Y \mid X)
\end{align}
\end{definition}

\begin{definition}[Conditional Mutual Information]
For random variables $X$, $Y$, and $Z$:
\[
I(X; Y \mid Z) = H(X \mid Z) - H(X \mid Y, Z)
\]
\end{definition}

\begin{definition}[Kullback-Leibler Divergence]
For two probability distributions $P$ and $Q$ over the same sample space $\mathcal{X}$:
\[
D_{\mathrm{KL}}(P \| Q) = \sum_{x \in \mathcal{X}} P(x) \log \frac{P(x)}{Q(x)}
\]
\end{definition}

\begin{definition}[Causal Hidden Representation]
In a causal Transformer, the hidden representation $h_t$ at position $t$ is a deterministic function of the prefix tokens:
\[
h_t = f_t(X_{1:t}) = f_t(x_1, x_2, \ldots, x_t)
\]
where $f_t$ is determined by the model architecture and parameters.
\end{definition}

\subsection{Assumptions}

\begin{assumption}[Information Preservation Property]
\label{assume:info-preserve}
A causal Transformer satisfies the \emph{Information Preservation Property} if the information loss $\epsilon_t := H(X_{1:t}) - I(h_t; X_{1:t})$ satisfies:
\[
\epsilon_{t+1} - \epsilon_t \leq H(x_{t+1} \mid X_{1:t}) \quad \text{for all } t < T.
\]
This holds when additional context tokens contribute more information than is lost through the representation bottleneck.
\end{assumption}

\begin{assumption}[Attention Aggregation Property]
\label{assume:attention-agg}
A causal Transformer satisfies the \emph{Attention Aggregation Property} if the final hidden state captures the information from all positions:
\[
I(h_1, \ldots, h_{T-1}; X \mid h_T) = 0.
\]
This is satisfied when attention weights $\alpha_{T,s}^{(\ell)} > 0$ for all $s \in \{1, \ldots, T\}$ across layers and the model has sufficient capacity.
\end{assumption}

\subsection{Proof of Theorem~\ref{thm:last-token}:}

\begin{proof}
We prove all three parts by analyzing the information-theoretic properties of causal hidden representations.

 We establish monotonicity for the cumulative token contexts $X_{1:t} = (x_1, \ldots, x_t)$.

Since $X_{1:t}$ is a sub-tuple of $X = (x_1, \ldots, x_T)$, knowing $X$ completely determines $X_{1:t}$. Formally, for any realization $x = (x_1, \ldots, x_T)$ of $X$, there is exactly one corresponding realization $x_{1:t} = (x_1, \ldots, x_t)$ of $X_{1:t}$.

This means the conditional probability satisfies:
\[
p(x_{1:t} \mid x) = \begin{cases} 1 & \text{if } x_{1:t} = (x_1, \ldots, x_t) \\ 0 & \text{otherwise} \end{cases}
\]

By Definition 2, the conditional entropy is:
\begin{align}
&H(X_{1:t} \mid X) \nonumber \\
&= -\sum_{x \in \mathcal{X}} \sum_{x_{1:t} \in \mathcal{X}_{1:t}} p(x, x_{1:t}) \log p(x_{1:t} \mid x)
\end{align}

Since $p(x_{1:t} \mid x) = 1$ when $x_{1:t}$ matches the first $t$ components of $x$, and $\log 1 = 0$:
\begin{align}
H(X_{1:t} \mid X) &= -\sum_{x \in \mathcal{X}} p(x) \cdot 1 \cdot \log 1 \nonumber \\
&= -\sum_{x \in \mathcal{X}} p(x) \cdot 0 = 0
\end{align}

By Definition 3, the mutual information between $X_{1:t}$ and $X$ is:
\[
I(X_{1:t}; X) = H(X_{1:t}) - H(X_{1:t} \mid X)
\]

Substituting $H(X_{1:t} \mid X) = 0$:
\begin{equation}
\label{eq:context_mi_t}
I(X_{1:t}; X) = H(X_{1:t}) - 0 = H(X_{1:t})
\end{equation}

Similarly, for $X_{1:t+1}$:
\begin{align}
\label{eq:context_mi_t1}
I(X_{1:t+1}; X) &= H(X_{1:t+1}) - H(X_{1:t+1} \mid X) \nonumber \\
&= H(X_{1:t+1}) - 0 = H(X_{1:t+1})
\end{align}

We now show that $H(X_{1:t}) \leq H(X_{1:t+1})$.

By the chain rule of entropy, the joint entropy of $(X_{1:t}, x_{t+1})$ can be decomposed as:
\[
H(X_{1:t}, x_{t+1}) = H(X_{1:t}) + H(x_{t+1} \mid X_{1:t})
\]

Since $X_{1:t+1} = (X_{1:t}, x_{t+1})$ by definition:
\begin{equation}
\label{eq:entropy_chain}
H(X_{1:t+1}) = H(X_{1:t}) + H(x_{t+1} \mid X_{1:t})
\end{equation}

We now show that $H(x_{t+1} \mid X_{1:t}) \geq 0$.

By Definition 2:
\begin{align}
&H(x_{t+1} \mid X_{1:t}) \nonumber \\
&= -\sum_{x_{1:t} \in \mathcal{X}_{1:t}} \sum_{x_{t+1} \in \mathcal{X}_{t+1}} p(x_{1:t}, x_{t+1}) \nonumber \\
&\quad \cdot \log p(x_{t+1} \mid x_{1:t})
\end{align}

Using the chain rule of probability $p(x_{1:t}, x_{t+1}) = p(x_{1:t}) p(x_{t+1} \mid x_{1:t})$:
\begin{align}
&H(x_{t+1} \mid X_{1:t}) \nonumber \\
&= -\sum_{x_{1:t} \in \mathcal{X}_{1:t}} p(x_{1:t}) \nonumber \\
&\quad \cdot \sum_{x_{t+1} \in \mathcal{X}_{t+1}} p(x_{t+1} \mid x_{1:t}) \log p(x_{t+1} \mid x_{1:t})
\end{align}

This can be written as an expectation over $X_{1:t}$:
\begin{align}
&H(x_{t+1} \mid X_{1:t}) \nonumber \\
&= \sum_{x_{1:t} \in \mathcal{X}_{1:t}} p(x_{1:t}) \cdot H(x_{t+1} \mid X_{1:t} = x_{1:t})
\end{align}

where the conditional entropy given a specific value $x_{1:t}$ is:
\begin{align}
&H(x_{t+1} \mid X_{1:t} = x_{1:t}) \nonumber \\
&= -\sum_{x_{t+1} \in \mathcal{X}_{t+1}} p(x_{t+1} \mid x_{1:t}) \log p(x_{t+1} \mid x_{1:t})
\end{align}

For any probability distribution $p(x_{t+1} \mid x_{1:t})$ over $\mathcal{X}_{t+1}$, the entropy is non-negative:
\[
H(x_{t+1} \mid X_{1:t} = x_{1:t}) \geq 0 \quad \text{for all } x_{1:t} \in \mathcal{X}_{1:t}
\]

Since $p(x_{1:t}) \geq 0$ for all $x_{1:t}$, the weighted sum is also non-negative:
\begin{align}
&H(x_{t+1} \mid X_{1:t}) \nonumber \\
&= \sum_{x_{1:t} \in \mathcal{X}_{1:t}} p(x_{1:t}) \cdot H(x_{t+1} \mid X_{1:t} = x_{1:t}) \geq 0
\end{align}

Returning to equation~\eqref{eq:entropy_chain}:
\[
H(X_{1:t+1}) = H(X_{1:t}) + H(x_{t+1} \mid X_{1:t})
\]

Since $H(x_{t+1} \mid X_{1:t}) \geq 0$:
\begin{align}
H(X_{1:t+1}) &= H(X_{1:t}) + H(x_{t+1} \mid X_{1:t}) \nonumber \\
&\geq H(X_{1:t}) + 0 = H(X_{1:t})
\end{align}

Therefore:
\[
H(X_{1:t+1}) \geq H(X_{1:t})
\]

Combining with equations~\eqref{eq:context_mi_t} and~\eqref{eq:context_mi_t1}:
\begin{align}
I(X_{1:t+1}; X) &= H(X_{1:t+1}) \nonumber \\
&\geq H(X_{1:t}) = I(X_{1:t}; X)
\end{align}

This inequality holds for all $t \in \{1, 2, \ldots, T-1\}$.

Applying this result iteratively:

For $t = T-1$:
\[
I(X_{1:T}; X) \geq I(X_{1:T-1}; X)
\]

For $t = T-2$:
\[
I(X_{1:T-1}; X) \geq I(X_{1:T-2}; X)
\]

Continuing this pattern for $t = T-3, T-4, \ldots, 2, 1$:
\begin{align}
I(X_{1:T-2}; X) &\geq I(X_{1:T-3}; X) \geq \cdots \nonumber \\
&\geq I(X_{1:2}; X) \geq I(X_{1:1}; X)
\end{align}

Combining all these inequalities by transitivity:
\begin{align}
I(X_{1:T}; X) &\geq I(X_{1:T-1}; X) \geq \cdots \nonumber \\
&\geq I(X_{1:2}; X) \geq I(X_{1:1}; X)
\end{align}

 We now extend the monotonicity result to the hidden representations $h_t$ under Assumption~\ref{assume:info-preserve}.

In a causal Transformer, the hidden representation $h_t$ at position $t$ is a deterministic function of the prefix tokens $X_{1:t}$:
\[
h_t = f_t(X_{1:t})
\]

We analyze the mutual information $I(h_t; X)$ by decomposing it using the chain rule.

The full sequence $X$ can be written as the concatenation of $X_{1:t}$ and $X_{t+1:T}$:
\[
X = (X_{1:t}, X_{t+1:T})
\]

By the chain rule of mutual information:
\[
I(h_t; X) = I(h_t; X_{1:t}, X_{t+1:T})
\]

Expanding using the chain rule:
\begin{align}
\label{eq:ht_chain}
&I(h_t; X_{1:t}, X_{t+1:T}) \nonumber \\
&= I(h_t; X_{1:t}) + I(h_t; X_{t+1:T} \mid X_{1:t})
\end{align}

We now evaluate the conditional mutual information $I(h_t; X_{t+1:T} \mid X_{1:t})$.

By Definition 4:
\begin{align}
\label{eq:cond_mi}
&I(h_t; X_{t+1:T} \mid X_{1:t}) \nonumber \\
&= H(h_t \mid X_{1:t}) - H(h_t \mid X_{t+1:T}, X_{1:t})
\end{align}

Since $h_t = f_t(X_{1:t})$ is a deterministic function of $X_{1:t}$, knowing $X_{1:t}$ completely determines $h_t$. Therefore:
\[
H(h_t \mid X_{1:t}) = 0
\]

Similarly, since $(X_{t+1:T}, X_{1:t}) = X$ contains $X_{1:t}$ as a component, and $h_t$ is determined by $X_{1:t}$:
\[
H(h_t \mid X_{t+1:T}, X_{1:t}) = H(h_t \mid X) = 0
\]

Substituting into equation~\eqref{eq:cond_mi}:
\[
I(h_t; X_{t+1:T} \mid X_{1:t}) = 0 - 0 = 0
\]

Returning to equation~\eqref{eq:ht_chain}:
\begin{align}
I(h_t; X) &= I(h_t; X_{1:t}) + I(h_t; X_{t+1:T} \mid X_{1:t}) \nonumber \\
&= I(h_t; X_{1:t}) + 0 = I(h_t; X_{1:t})
\end{align}

Therefore, the mutual information between $h_t$ and $X$ equals the mutual information between $h_t$ and its causal context:
\begin{equation}
\label{eq:ht_equals_context}
I(h_t; X) = I(h_t; X_{1:t})
\end{equation}

We now establish an upper bound on $I(h_t; X_{1:t})$.

Since $h_t = f_t(X_{1:t})$ is a deterministic function of $X_{1:t}$, the random variables form a Markov chain:
\[
X \to X_{1:t} \to h_t
\]

This means that $h_t$ is conditionally independent of $X$ given $X_{1:t}$:
\[
p(h_t \mid X_{1:t}, X) = p(h_t \mid X_{1:t})
\]

By the Data Processing Inequality, for any Markov chain $A \to B \to C$:
\[
I(A; C) \leq I(A; B)
\]

Applying this to our Markov chain with $A = X$, $B = X_{1:t}$, and $C = h_t$:
\[
I(X; h_t) \leq I(X; X_{1:t})
\]

Since mutual information is symmetric, $I(X; h_t) = I(h_t; X)$ and $I(X; X_{1:t}) = I(X_{1:t}; X)$:
\[
I(h_t; X) \leq I(X_{1:t}; X)
\]

From equation~\eqref{eq:context_mi_t}, we have $I(X_{1:t}; X) = H(X_{1:t})$, so:
\[
I(h_t; X) \leq H(X_{1:t})
\]

Similarly, for $h_{t+1} = f_{t+1}(X_{1:t+1})$:
\[
I(h_{t+1}; X) \leq H(X_{1:t+1})
\]

We now establish the monotonicity for hidden representations using Assumption~\ref{assume:info-preserve}.

By Assumption~\ref{assume:info-preserve}, the information loss $\epsilon_t := H(X_{1:t}) - I(h_t; X_{1:t})$ satisfies:
\[
\epsilon_{t+1} - \epsilon_t \leq H(x_{t+1} \mid X_{1:t})
\]

This can be rewritten as:
\begin{align}
&H(X_{1:t+1}) - I(h_{t+1}; X_{1:t+1}) \nonumber \\
&\quad - H(X_{1:t}) + I(h_t; X_{1:t}) \leq H(x_{t+1} \mid X_{1:t})
\end{align}

Using the chain rule $H(X_{1:t+1}) = H(X_{1:t}) + H(x_{t+1} \mid X_{1:t})$ from equation~\eqref{eq:entropy_chain}:
\begin{align}
&H(X_{1:t}) + H(x_{t+1} \mid X_{1:t}) - I(h_{t+1}; X_{1:t+1}) \nonumber \\
&\quad - H(X_{1:t}) + I(h_t; X_{1:t}) \leq H(x_{t+1} \mid X_{1:t})
\end{align}

Simplifying:
\begin{align}
&H(x_{t+1} \mid X_{1:t}) - I(h_{t+1}; X_{1:t+1}) + I(h_t; X_{1:t}) \nonumber \\
&\quad \leq H(x_{t+1} \mid X_{1:t})
\end{align}

\[
- I(h_{t+1}; X_{1:t+1}) + I(h_t; X_{1:t}) \leq 0
\]

\[
I(h_t; X_{1:t}) \leq I(h_{t+1}; X_{1:t+1})
\]

Combining with equation~\eqref{eq:ht_equals_context}:
\begin{align}
I(h_t; X) &= I(h_t; X_{1:t}) \nonumber \\
&\leq I(h_{t+1}; X_{1:t+1}) = I(h_{t+1}; X)
\end{align}

Therefore:
\[
I(h_t; X) \leq I(h_{t+1}; X) \quad \text{for all } t < T
\]

Applying iteratively:

For $t = T-1$:
\[
I(h_{T-1}; X) \leq I(h_T; X)
\]

For $t = T-2$:
\[
I(h_{T-2}; X) \leq I(h_{T-1}; X)
\]

Continuing this pattern for $t = T-3, T-4, \ldots, 2, 1$:
\begin{align}
&I(h_{T-3}; X) \leq I(h_{T-2}; X), \quad \ldots, \nonumber \\
&\quad I(h_1; X) \leq I(h_2; X)
\end{align}

Combining all these inequalities by transitivity:
\begin{align}
I(h_1; X) &\leq I(h_2; X) \leq \cdots \nonumber \\
&\leq I(h_{T-1}; X) \leq I(h_T; X)
\end{align}

Therefore, for all $t \leq T$:
\[
I(h_T; X) \geq I(h_t; X)
\]

 We now prove that the last token dominates mean pooling in terms of mutual information with the input, under Assumption~\ref{assume:attention-agg}.

The mean-pooled representation is defined as:
\[
\bar{h} = \frac{1}{T}\sum_{t=1}^T h_t
\]

This is a deterministic function of the tuple $(h_1, h_2, \ldots, h_T)$:
\[
\bar{h} = \phi(h_1, h_2, \ldots, h_T)
\]

where $\phi$ is the averaging function.

Since $\bar{h}$ is determined by $(h_1, \ldots, h_T)$, we have the Markov chain:
\[
X \to (h_1, h_2, \ldots, h_T) \to \bar{h}
\]

By the Data Processing Inequality:
\begin{equation}
\label{eq:pooling_dpi}
I(\bar{h}; X) \leq I(h_1, h_2, \ldots, h_T; X)
\end{equation}

We now analyze $I(h_1, h_2, \ldots, h_T; X)$.

Each hidden representation $h_t = f_t(X_{1:t})$ is a deterministic function of $X_{1:t}$. The tuple $(h_1, \ldots, h_T)$ is therefore a deterministic function of $X$:
\begin{align}
&(h_1, h_2, \ldots, h_T) \nonumber \\
&= (f_1(X_{1:1}), f_2(X_{1:2}), \ldots, f_T(X_{1:T})) = F(X)
\end{align}

for some function $F$ determined by the model.

This gives us the Markov chain:
\[
X \to (h_1, h_2, \ldots, h_T)
\]

By the Data Processing Inequality applied in reverse (since $(h_1, \ldots, h_T) = F(X)$ is a deterministic function):
\[
I(h_1, h_2, \ldots, h_T; X) \leq H(X)
\]

In causal Transformers, the final hidden state $h_T$ has access to all positions through the causal attention mechanism. At each layer $\ell$, position $T$ computes:
\begin{align}
h_T^{(\ell)} &= \text{Attention}^{(\ell)}(Q_T^{(\ell)}, K_{1:T}^{(\ell)}, V_{1:T}^{(\ell)}) \nonumber \\
&\quad + h_T^{(\ell-1)}
\end{align}

where the attention weights are:
\begin{align}
\alpha_{T,s}^{(\ell)} &= \frac{\exp(Q_T^{(\ell)} \cdot K_s^{(\ell)} / \sqrt{d})}{\sum_{s'=1}^{T} \exp(Q_T^{(\ell)} \cdot K_{s'}^{(\ell)} / \sqrt{d})} \nonumber \\
&\quad \text{for } s \in \{1, \ldots, T\}
\end{align}

By Assumption~\ref{assume:attention-agg}, we have:
\[
I(h_1, \ldots, h_{T-1}; X \mid h_T) = 0
\]

We can decompose the joint mutual information using the chain rule:
\begin{align}
&I(h_1, \ldots, h_T; X) \nonumber \\
&= I(h_T; X) + I(h_1, \ldots, h_{T-1}; X \mid h_T)
\end{align}

This follows from the chain rule of mutual information:
\[
I(A, B; C) = I(A; C) + I(B; C \mid A)
\]

with $A = h_T$, $B = (h_1, \ldots, h_{T-1})$, and $C = X$.

Substituting Assumption~\ref{assume:attention-agg}:
\begin{align}
&I(h_1, \ldots, h_T; X) \nonumber \\
&= I(h_T; X) + I(h_1, \ldots, h_{T-1}; X \mid h_T) \nonumber \\
&= I(h_T; X) + 0 = I(h_T; X)
\end{align}

Returning to equation~\eqref{eq:pooling_dpi}:
\begin{align}
I(\bar{h}; X) &\leq I(h_1, h_2, \ldots, h_T; X) \nonumber \\
&= I(h_T; X)
\end{align}

Therefore:
\[
I(h_T; X) \geq I(\bar{h}; X)
\]

This completes the proof.
\end{proof}

Theorem~\ref{thm:last-token} establishes that in causal Transformers satisfying the Information Preservation Property (Assumption~\ref{assume:info-preserve}), the last token's hidden representation $h_T$ contains the most information about the input sequence $X$ among all single-position representations. Furthermore, under the Attention Aggregation Property (Assumption~\ref{assume:attention-agg}), $h_T$ contains at least as much information as mean pooling. These properties provide theoretical justification for using the last token for sequence-wise routing decisions, as it maximizes the information available for determining expert assignment.
\clearpage

\section{Extended Ablation Studies}
\label{app:ablation}
\begin{figure*}[!htb]
    \centering
    \includegraphics[width=1.00\linewidth]{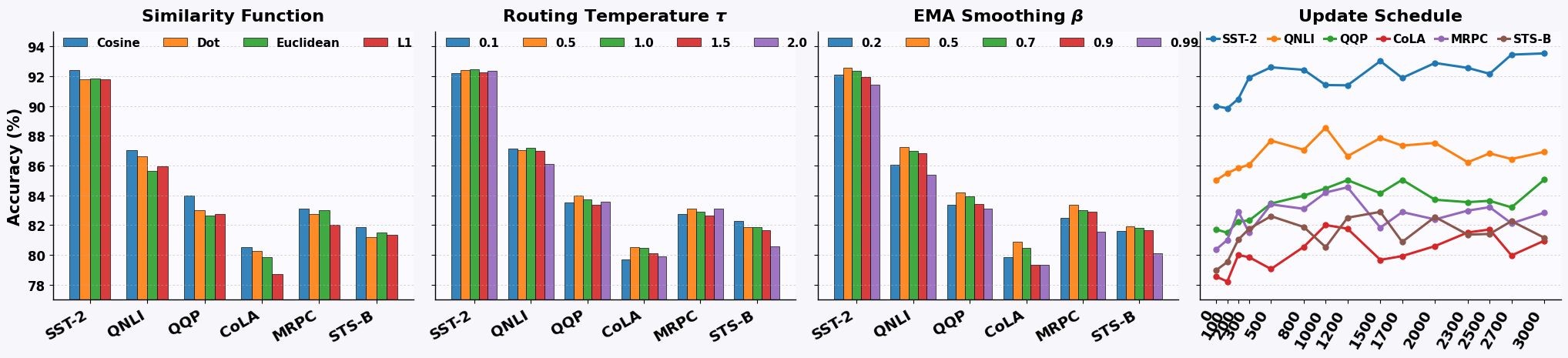}
    \caption{Routing hyper-parameters. (a) Similarity function: cosine similarity outperforms distance-based metrics. (b) Routing temperature $\tau$: best at 1.0, stable across [0.5, 1.5]. (c) EMA smoothing $\beta$: best at 0.7, stable across [0.5, 0.9]. (d) Update schedule: performance improves with updates until 1500--2000 steps, then saturates.}
    \label{fig:router}
\end{figure*}

We provide additional ablation studies to analyze MJ's sensitivity to hyperparameters and design choices. 

\subsection{Similarity Function}
\label{app:ablation-similarity}

Figure~\ref{fig:router}(a) compares four similarity functions for computing routing scores between token representations and cluster centers: cosine similarity, dot product, Euclidean distance, and L1 distance.

Cosine similarity achieves the best overall performance across all six GLUE tasks. This is expected because cosine similarity is scale-invariant—it measures the angle between vectors rather than their magnitudes. In high-dimensional representation spaces, token embeddings can have varying norms depending on their position and content, making scale-invariant measures more robust for routing decisions.

Dot product performs comparably to cosine on most tasks but shows slightly higher variance. Euclidean and L1 distances perform worse, particularly on tasks like MRPC and CoLA. Distance-based metrics are sensitive to the absolute scale of representations, which can vary significantly across layers and tokens.

\begin{notebox}
\textit{\textbf{Recommendation:} Use cosine similarity for routing. It is scale-invariant and consistently outperforms distance-based metrics across all tasks.}
\end{notebox}

\subsection{Routing Temperature}
\label{app:ablation-temperature}

Figure~\ref{fig:router}(b) ablates the softmax temperature $\tau$ used in the routing mechanism. Lower temperatures produce sharper (more concentrated) routing distributions, while higher temperatures produce softer (more uniform) distributions.

We test $\tau \in \{0.1, 0.5, 1.0, 1.5, 2.0\}$. Very low temperature ($\tau = 0.1$) leads to near-hard routing where most probability mass concentrates on a single expert, reducing the benefits of soft top-$k$ weighting. Very high temperature ($\tau = 2.0$) spreads probability too uniformly, diminishing the specialization effect of routing.

The best performance is achieved at $\tau = 1.0$, which provides a balanced trade-off between specialization and smoothness. Performance is relatively stable across $\tau \in [0.5, 1.5]$, indicating that MJ is not overly sensitive to this hyperparameter.

\subsection{EMA Smoothing Factor}
\label{app:ablation-ema}

Figure~\ref{fig:router}(c) ablates the EMA momentum coefficient $\beta$ used for online center updates. Higher $\beta$ values make centers update more slowly (more weight on historical values), while lower $\beta$ values make centers adapt more quickly to recent batches.

We test $\beta \in \{0.2, 0.5, 0.7, 0.9, 0.99\}$. Very low momentum ($\beta = 0.2$) causes centers to fluctuate rapidly, potentially destabilizing routing decisions. Very high momentum ($\beta = 0.99$) makes centers update too slowly, preventing them from tracking the evolving token distribution.

The best performance is achieved at $\beta = 0.7$, with $\beta \in [0.5, 0.9]$ performing comparably. This range provides a good balance: centers are stable enough to provide consistent routing decisions, yet adaptive enough to track distributional shifts as adapter weights change.

\begin{notebox}
\textit{\textbf{Finding:} MJ is robust to hyperparameter choices. Temperature $\tau \in [0.5, 1.5]$ and EMA momentum $\beta \in [0.5, 0.9]$ all achieve comparable performance, simplifying practical deployment.}
\end{notebox}

\subsection{Update Schedule}
\label{app:ablation-schedule}

Figure~\ref{fig:router}(d) ablates the update schedule—specifically, until which training step we continue updating cluster centers via EMA, after which centers are frozen.

We test stopping points from 0 (no updates, use only $k$-means initialization) to 3000 steps. Performance generally improves as we allow more updates, with most tasks showing improvement up to 1500--2000 steps. Beyond this point, additional updates provide diminishing returns.

Early stopping of updates (0--500 steps) performs worst because centers initialized via $k$-means on the frozen backbone become misaligned as adapter weights change during training. Allowing updates until 1500--2000 steps lets centers track the evolving token distribution during the critical early phase of adapter training. Stopping updates before training ends (rather than updating throughout) provides stable routing decisions during the final fine-tuning phase.

\begin{notebox}
\textit{\textbf{Key insight:} Centers should be updated during early training to track adapter changes, then frozen at 50--70\% of training for stable final convergence.}
\end{notebox}

\begin{figure*}[t]
    \centering
    \includegraphics[width=\linewidth]{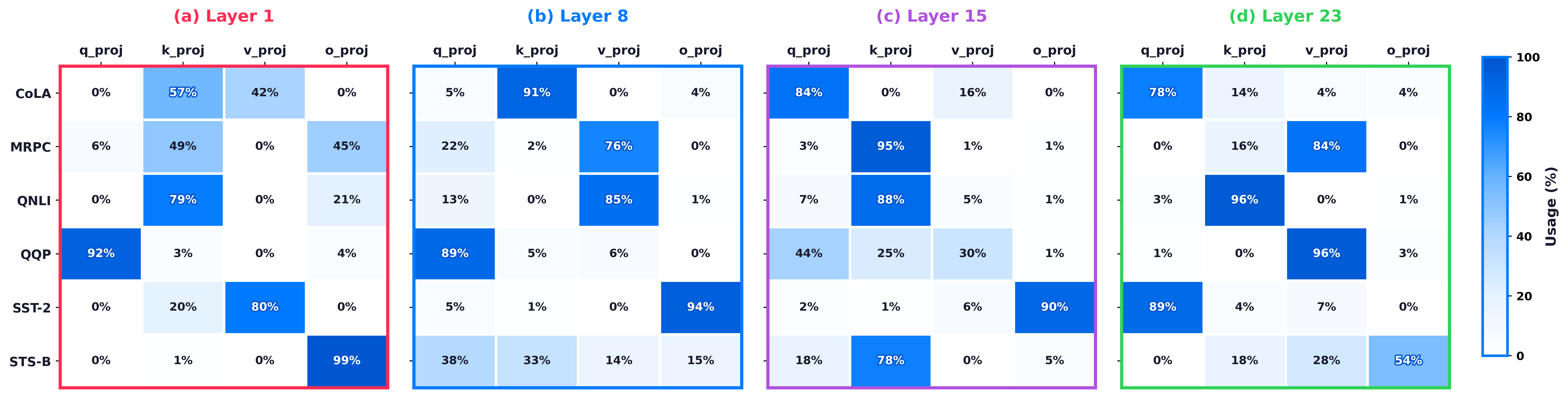}
    \caption{Expert usage across layers and GLUE tasks. Each heatmap shows the percentage of tokens routed to each attention projection (Q, K, V, O). Different tasks prefer different projections, and preferences evolve across layers—validating the submodules-as-experts design.}
    \label{fig:projection_routing}
\end{figure*}

\subsection{Projection Specialization}
\label{app:ablation-projection}
A key design choice in MJ is treating the existing projection adapters (Q, K, V, O) as implicit experts. A natural question arises: do tokens from different tasks actually route differently, or does routing collapse to the same adapter for everything? We analyze routing patterns across layers and tasks to justify this design. Throughout, we label each expert slot by the projection it happens to occupy; per \S\ref{app:ablation-permutation}, this assignment is arbitrary and performance is unchanged when it is shuffled.

Figure~\ref{fig:projection_routing} shows expert usage heatmaps for six GLUE tasks across four representative layers (1, 8, 15, 23) using attention projections. Each cell indicates what percentage of tokens from a given task are routed to each projection. The heatmaps reveal clear task-specific routing preferences. In Layer 1, QQP routes 92\% of tokens to Q projection, while STS-B routes 99\% to O projection. SST-2 prefers V projection (80\%), and QNLI prefers K projection (79\%). This diversity confirms that different tasks route to different adapters—if routing collapsed to a single slot, there would be no specialization to exploit.

Interestingly, tasks with similar structure share routing patterns. QNLI (question-answering inference) and MRPC (paraphrase detection) are both sentence-pair tasks requiring comparison between two text segments. Their routing patterns align across multiple layers: in Layer 1, both prefer K projection (QNLI: 79\%, MRPC: 49\%); in Layer 8, both shift to V projection (QNLI: 85\%, MRPC: 76\%); in Layer 15, both strongly prefer K projection (QNLI: 88\%, MRPC: 95\%). This suggests that MJ discovers task similarity through representation clustering—semantically related tasks naturally route to the same expert slots without explicit supervision.

The routing patterns also evolve across layers. CoLA shifts from K projection (57\%) in Layer 1 to K (91\%) in Layer 8, then to Q projection (84\%) in Layer 15 and remains Q-dominant (78\%) in Layer 23. SST-2 transitions from V (80\%) in early layers to O (94\%) in middle layers to Q (89\%) in later layers. Early layers show more distributed routing—multiple projections receive significant traffic—while later layers show sharper specialization with single projections dominating (e.g., QNLI: 96\% K, QQP: 96\% V, SST-2: 89\% Q in Layer 23). This aligns with the understanding that later Transformer layers capture more task-specific features.

\begin{notebox}
\textit{\textbf{Key finding:} Different tasks route to different expert slots, and semantically similar tasks (e.g., QNLI and MRPC) share routing patterns. This justifies treating the existing adapters as implicit experts—they naturally specialize for different tasks without explicit supervision. Which projection fills each slot is not what matters (\S\ref{app:ablation-permutation}); what matters is that tasks separate consistently across slots.}
\end{notebox}

\subsection{Linear Probing for Last-Token Routing}
\label{app:ablation-linearprobe}

Theorem~\ref{thm:last-token} makes two claims of different strength. Parts (i)--(ii) state that later positions in a causal Transformer carry more information about the sequence than earlier ones. Part (iii) additionally states that the final token dominates mean pooling, but only under Assumption~\ref{assume:attention-agg}. We test both empirically using linear probing on frozen representations.

Figure~\ref{fig:linear_probing}(a,b) shows linear probing accuracy at Layers 23 and 15 using different token positions for classification. We extract representations from four positions, ordered from earliest to latest: Token-1 (30th from the end), Token-2 (20th from the end), Token-3 (10th from the end), and Token-4 (the final token $h_T$), together with three pooling strategies: mean, max, and attention pooling. For this analysis we restrict to examples with at least 30 tokens, so that all four positions are well defined and distinct; this filter applies to the probing study only, not to the training or evaluation of MJ. A linear classifier is trained on each representation to predict task labels, so higher accuracy indicates more task-relevant information in that representation. Token-4 is the representation used by sequence-wise routing, so this comparison directly measures the quantity the theorem concerns.

The results support this theorem with a large margin. Accuracy increases monotonically from Token-1 to Token-4 across all tasks: at Layer 23, Token-4 reaches 82--90\% against 47--57\% for Token-1. This is the ordering the theorem predicts, and it follows from causal attention---later positions have aggregated context from the entire prefix, while early positions have almost none.

However, attention pooling shows performance competitive with the last token, which is unsurprising: it is a trainable module whose hypothesis class contains last-token selection as a special case.

\subsection{Expert Permutation Analysis}
\label{app:ablation-permutation}

\begin{figure*}[t]
    \centering
    \includegraphics[width=\linewidth]{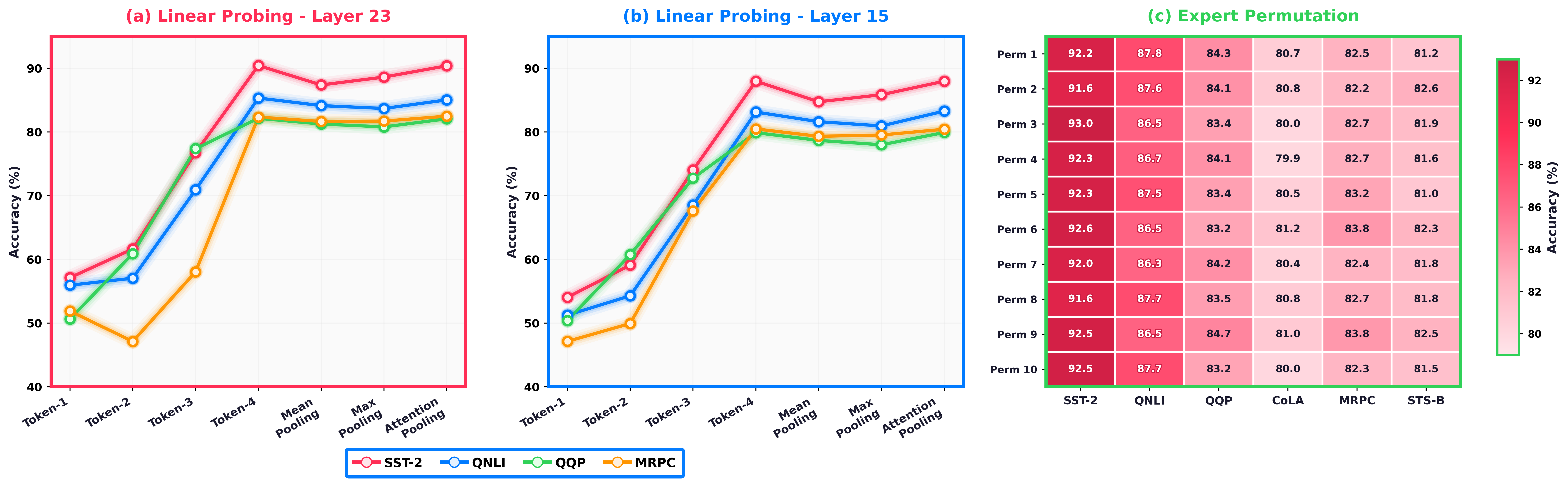}
    \caption{(a,b) Linear probing accuracy using different token positions and pooling strategies at Layers 23 and 15. Later tokens (Token-4) outperform earlier tokens, validating Theorem~\ref{thm:last-token}. (c) Expert permutation analysis: shuffling projection-to-expert assignments has minimal impact ($<$0.5\% variance), showing MJ is robust to specific assignments.}
    \label{fig:linear_probing}
\end{figure*}

A natural question is whether the specific assignment of projections to expert slots matters, or whether the routing mechanism itself drives performance. We investigate this by keeping the router exactly the same but shuffling which projection each expert slot controls.

Figure~\ref{fig:linear_probing}(c) shows results across 10 random permutations on GLUE tasks. Performance remains remarkably stable: SST-2 varies only between 91.6--93.0\%, QNLI between 86.3--87.8\%, and other tasks show similar consistency. The standard deviation across permutations is less than 0.5\% for most tasks.

This finding aligns with observations from the original LoRA paper's Table 5~\cite{hu2022lora}, which showed that applying LoRA to different projections yields very similar performance. The key insight is that the routing mechanism—deciding which adapters to activate for each token—matters more than the specific projection assigned to each expert slot. As long as tokens are routed consistently based on their representations, the model learns to utilize whichever projections are assigned effectively.

\begin{notebox}
\textit{\textbf{Key finding:} Expert permutation has minimal impact on performance ($<$0.5\% variance). The routing mechanism matters more than specific projection assignments—MJ is robust to expert-projection mapping.}
\end{notebox}


\subsection{Shared Expert Selection}
\label{app:ablation-shared}
\begin{figure*}[!htb]
    \centering
    \includegraphics[width=1.0\linewidth]{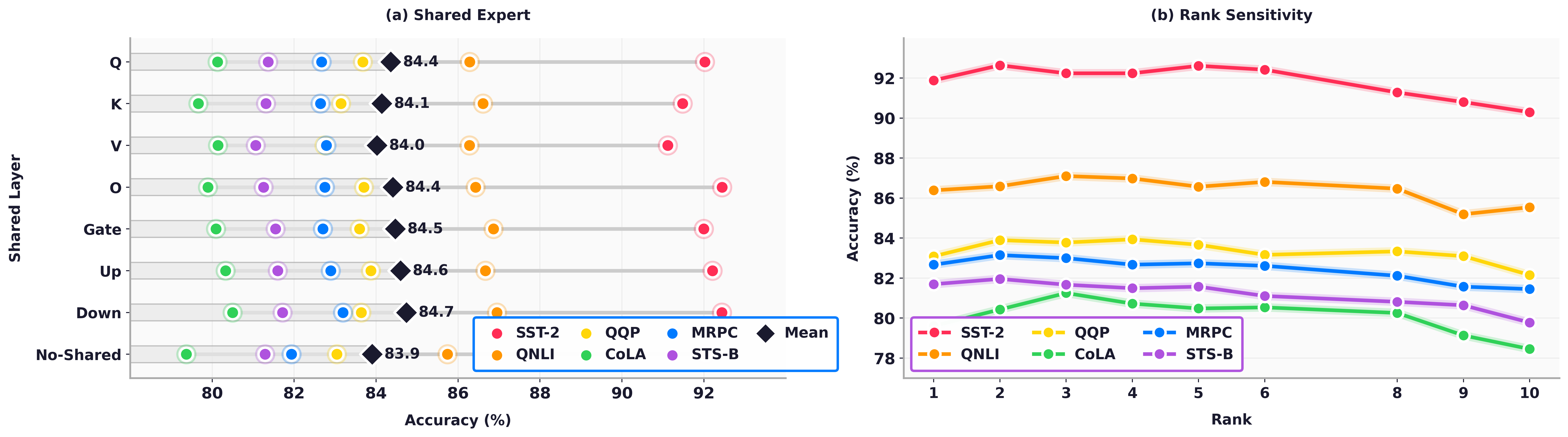}
    \caption{Ablation on shared expert and rank. (a) Shared expert selection: Up and Down projections achieve best performance (84.5\%); no-shared performs worst (83.9\%). (b) Rank sensitivity: performance improves up to rank 4--6, then plateaus or decreases due to overfitting.}
    \label{fig:shared_expert_rank}
\end{figure*}

Figure~\ref{fig:shared_expert_rank}(a) ablates which  projection to designate as the shared expert—the adapter that remains always active ($m_{t,e^*} = 1$) regardless of routing decisions. We compare using each of the seven  projections (Q, K, V, O, Gate, Up, Down) as the shared expert, as well as a no-shared baseline where all adapters participate in routing.

The results show that using Up or Down projection as the shared expert achieves the best mean accuracy (84.5\%), followed closely by Q and Gate (84.4\%). Using V as the shared expert performs slightly worse (84.0\%), and the no-shared configuration performs worst (83.9\%).

The FFN projections (Up, Down) are effective shared experts because they apply broad, task-general transformations that benefit all tokens. In contrast, attention projections (Q, K, V) are more input-specific—their optimal behavior varies depending on the token's role in the sequence. Having a shared FFN adapter provides a stable "backbone" adaptation, while routing specializes the attention adapters for different token types. 

In our main experiments, however, we use O and gate as the shared adapters rather than Up and Down. These four options fall within a narrow band (all within roughly 0.3 points), and O and gate are the more parameter-efficient choice, since the up and down projections are substantially wider than the attention output projection on the backbones we use. This also keeps the equipped set to five projections (Q, K, V, O, gate), matching the configuration used for every PEFT and MoE-PEFT baseline so that the comparison isolates the routing mechanism.

\subsection{Rank Sensitivity}
\label{app:ablation-rank}

Figure~\ref{fig:shared_expert_rank}(b) ablates the LoRA rank $r$ used in MJ adapters, varying from 1 to 10. Higher rank increases adapter capacity but also increases the risk of overfitting, especially on smaller datasets.

Performance improves as rank increases from 1 to 4, with most tasks showing clear gains. From rank 4 to 6, improvements continue but at a slower rate. Beyond rank 6, performance plateaus or slightly decreases on some tasks (e.g., CoLA, MRPC), indicating the onset of overfitting.

The sensitivity to rank varies by task. Tasks with larger training sets (SST-2, QNLI, QQP) continue to benefit from higher ranks up to 8--10. Tasks with smaller training sets (CoLA, MRPC) peak at rank 4--6 and degrade slightly at higher ranks. STS-B shows relatively stable performance across all ranks.


\begin{figure*}[!htb]
    \centering
    \includegraphics[width=1.0\linewidth]{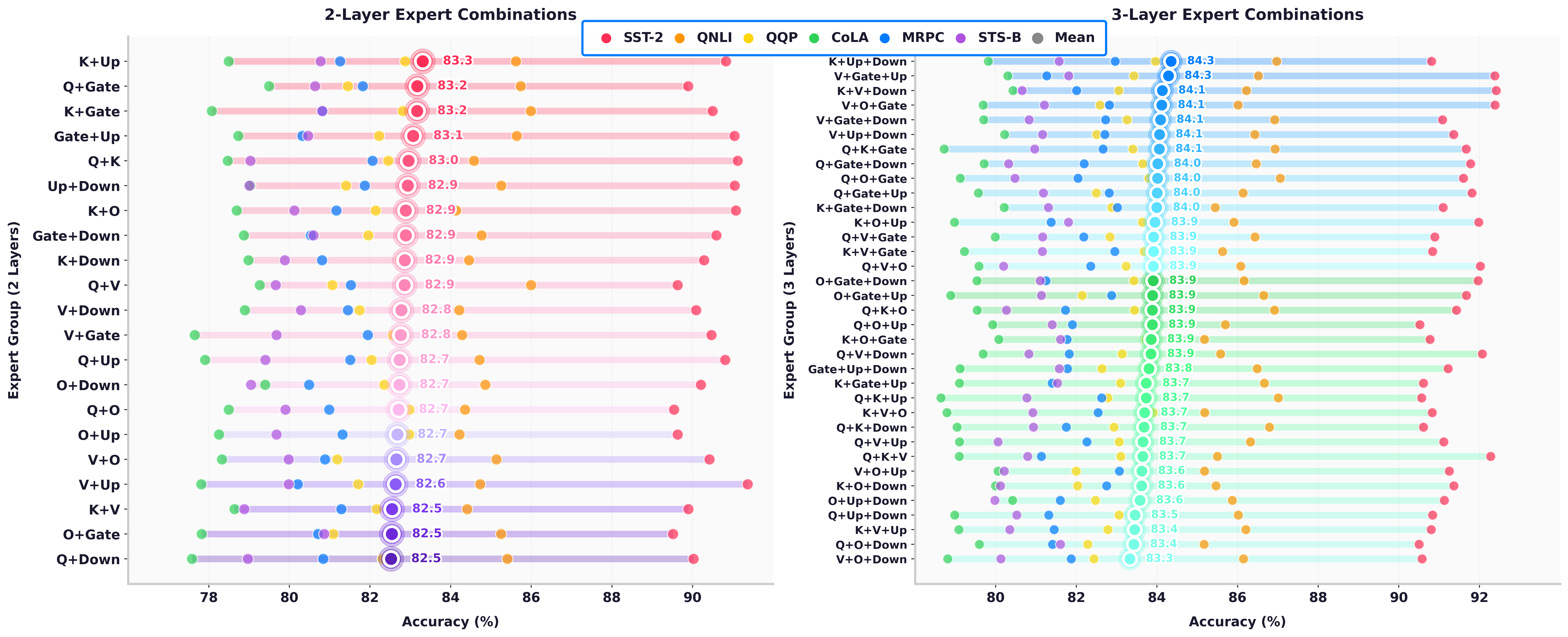}
    \caption{Expert combination analysis (2 and 3 layers). Left: 2-layer combinations; K+Up achieves best performance (83.3\%). Right: 3-layer combinations; K+Up+Down achieves 84.3\%. Combinations mixing attention and FFN components consistently outperform single-pathway combinations.}
    \label{fig:expert_combination1}
\end{figure*}

\begin{figure*}[!htb]
    \centering
    \includegraphics[width=1.0\linewidth]{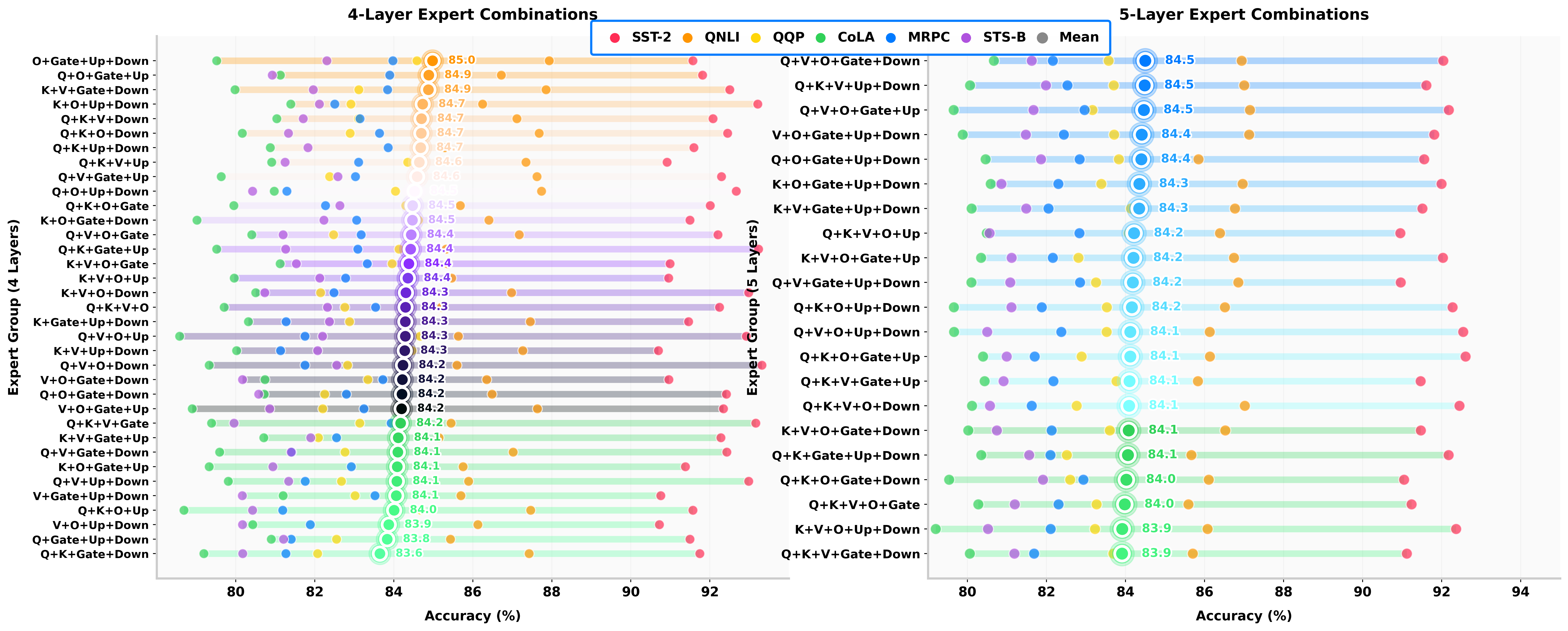}
   \caption{Expert combination analysis (4 and 5 layers). \textbf{Left}: 4-layer 
combinations; O+Gate+Up+Down achieves the best performance overall (85.0\%). 
\textbf{Right}: 5-layer combinations; Q+V+O+Gate+Down is best at 84.5\%, below 
the 4-layer optimum. Performance improves with routing coverage up to 4 
projections and then declines, indicating an optimal expert count for 
gradient-free routing.}
    \label{fig:expert_combination2}
\end{figure*}

\subsection{Expert Combination Analysis}
\label{app:ablation-combination}

We systematically analyze which combinations of  projections benefit most from MJ routing. Rather than applying routing to all  projections, we selectively enable routing on subsets of 2, 3, 4, and 5 layers, with remaining layers using standard PEFT. This reveals which  projections contribute most to routing-based specialization.

\paragraph{2-Layer combinations.}
Figure~\ref{fig:expert_combination1}(left) shows all pairwise combinations. The best performing pairs are K+Up (83.3\%), Q+Gate (83.2\%), and K+Gate (83.2\%). Notably, combinations involving FFN components (Up, Down, Gate) paired with attention components (Q, K) consistently outperform pairs of only attention or only FFN  projections. The worst combinations (Q+Down, O+Gate) score around 82.5\%, still above the no-routing baseline.

\paragraph{3-Layer combinations.}
Figure~\ref{fig:expert_combination1}(right) shows 3-layer combinations. Top performers include K+Up+Down (84.3\%), V+Gate+Up (84.3\%), and K+V+Down (84.3\%). Performance improves by approximately 1\% over 2-layer combinations. The best combinations typically include at least one attention component and at least one FFN component, reinforcing that routing benefits from covering both attention and FFN pathways.

\paragraph{4-Layer combinations.}
Figure~\ref{fig:expert_combination2}(left) shows 4-layer combinations. The best combination is O+Gate+Up+Down (85.0\%), followed by Q+O+Gate+Up (84.9\%) and K+V+Gate+Up (84.9\%). Including the output projection O alongside FFN components appears particularly effective. Performance continues to improve over 3-layer combinations.
\paragraph{5-Layer combinations.}
Figure~\ref{fig:expert_combination2}(right) shows 5-layer combinations. Surprisingly, 5-layer routing (84.5\% best) underperforms 4-layer routing (85.0\% best). Top 5-layer combinations include Q+V+O+Gate+Down (84.5\%), Q+K+V+Up+Down (84.5\%), and Q+V+O+Gate+Up (84.5\%). This performance drop occurs because MJ's gradient-free routing relies on natural clustering in the representation space. With more experts, token clusters become finer-grained, and similar tokens may be split across experts that would benefit from shared adaptation. Since the router is not trained to optimize task performance, it cannot compensate for suboptimal cluster boundaries—leading to potential expert underutilization when expert count exceeds the natural cluster structure of the data.


\paragraph{Practical guidance.}
For resource-constrained settings, a 3-layer combination like K+Up+Down provides 84.3\% accuracy with minimal overhead. For maximum performance, 4-layer combinations like O+Gate+Up+Down achieve 85.0\%—adding more layers decreases accuracy due to cluster fragmentation. The consistent appearance of Up and Down in top combinations aligns with our recommendation to use FFN projections as shared experts (Section~\ref{app:ablation-shared}).

\begin{figure*}[!htb]
    \centering
    \includegraphics[width=0.8\linewidth]{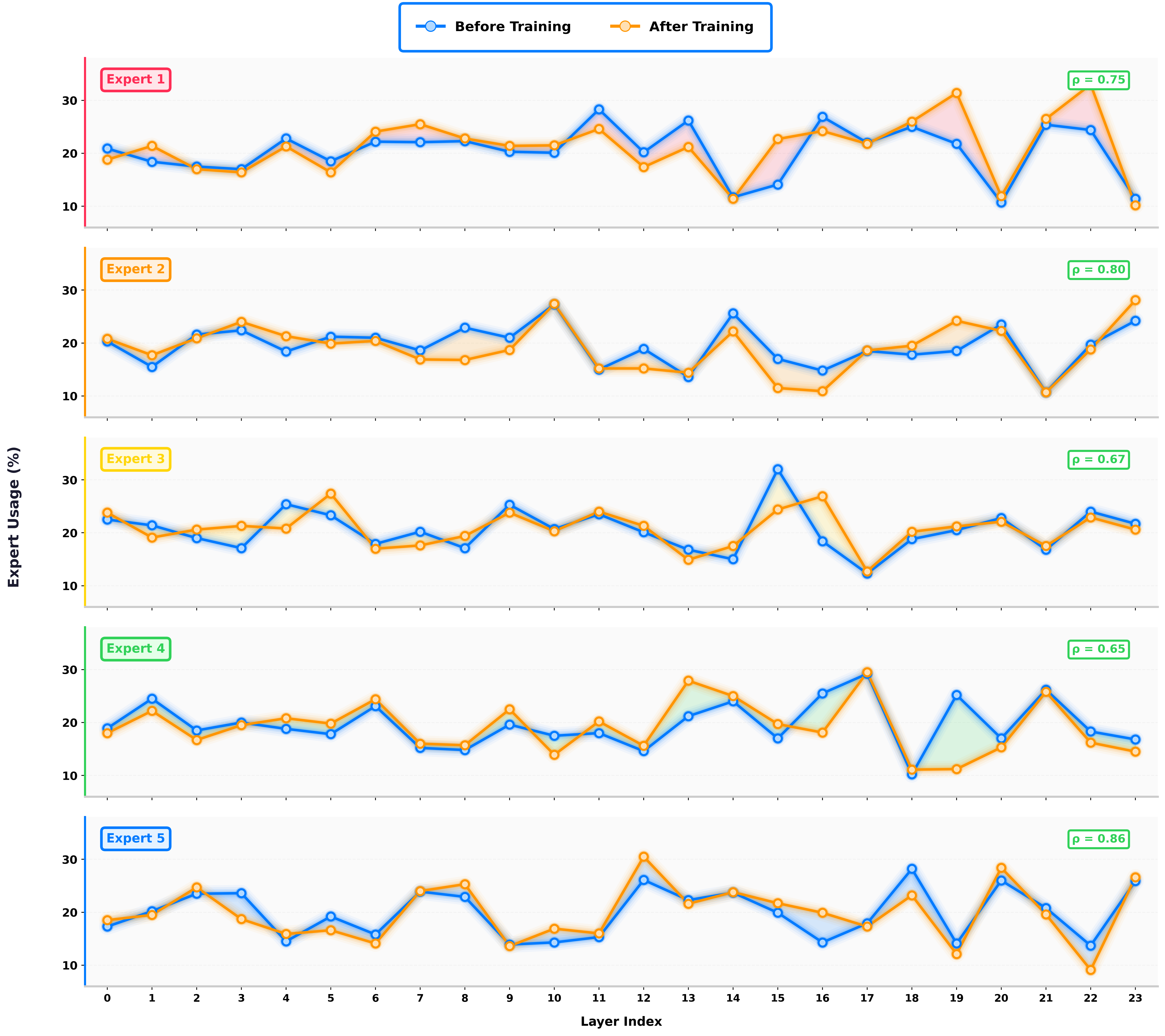}
    \caption{Expert usage across layers before and after training. \textcolor{blue}{\textbf{Blue}}: usage after $k$-means initialization; \textcolor{orange}{\textbf{Orange}}: usage after training with EMA updates. Correlation $\rho$ shows structure preservation. All experts maintain balanced usage (15--30\%) throughout training, demonstrating self-balancing without auxiliary losses.}
    \label{fig:expert_uses}
\end{figure*}

\begin{figure*}[!htb]
    \centering
    \includegraphics[width=1.0\linewidth]{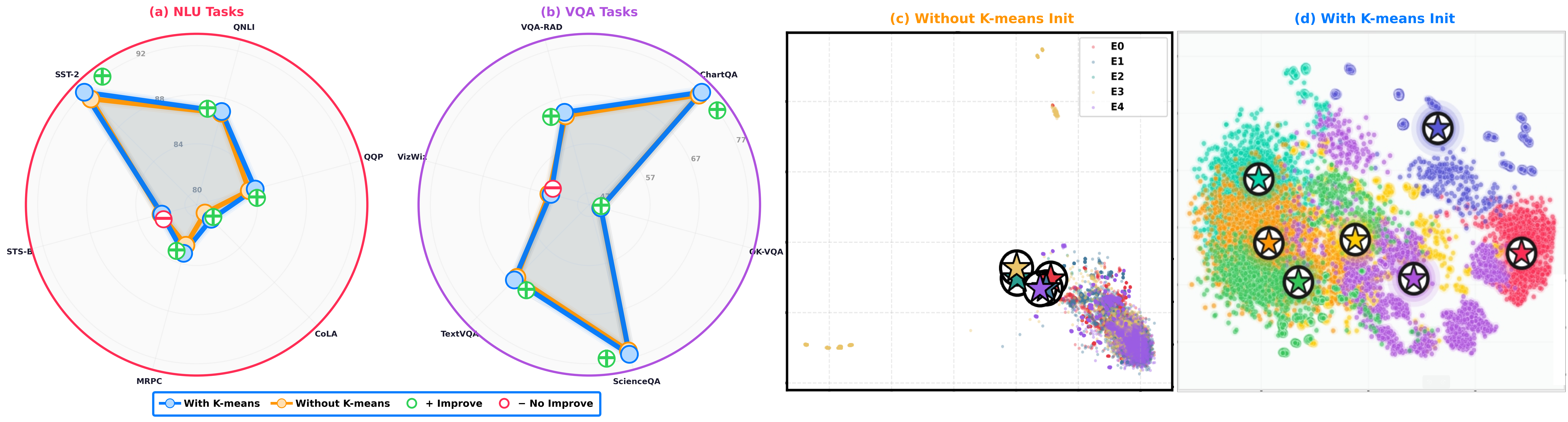}
    \caption{Impact of $k$-means initialization. (a,b) Comparing performance with (\textcolor{blue}{\textbf{blue}}) and without (\textcolor{orange}{\textbf{orange}}) $k$-means on NLU and VQA tasks; $\oplus$ indicates improvement, $\ominus$ indicates no improvement. (c,d) t-SNE visualization of token clusters: without $k$-means, centers are poorly positioned; with $k$-means, centers align with natural cluster structure, enabling meaningful expert specialization.}
    \label{fig:kmneas_init}
\end{figure*}

\subsection{Impact of K-means Initialization}
\label{app:ablation-kmeans}

We analyze the impact of $k$-means initialization by comparing MJ with properly initialized centers versus randomly initialized centers (without $k$-means). Figure~\ref{fig:kmneas_init}(a,b) shows performance comparison on NLU and VQA tasks, while Figure~\ref{fig:kmneas_init}(c,d) visualizes the resulting cluster structure via t-SNE.

On NLU tasks (Figure~\ref{fig:kmneas_init}(a)), $k$-means initialization improves performance on 5 out of 6 benchmarks: SST-2 (92.43\% vs 91.62\%), QNLI (87.06\% vs 86.92\%), QQP (83.98\% vs 83.44\%), CoLA (80.54\% vs 79.78\%), and MRPC (83.09\% vs 82.37\%). Only STS-B shows negligible difference (81.87\% vs 81.95\%). On VQA tasks (Figure~\ref{fig:kmneas_init}(b)), $k$-means similarly improves most benchmarks: ChartQA (76.69\% vs 75.88\%), ScienceQA (75.84\% vs 75.11\%), TextVQA (65.95\% vs 65.21\%), and VQA-RAD (63.80\% vs 63.05\%). OK-VQA shows minimal difference, and VizWiz slightly favors random initialization.

The t-SNE visualizations reveal why $k$-means initialization matters. Without $k$-means (Figure~\ref{fig:kmneas_init}(c)), cluster centers (marked with stars) are poorly positioned—some centers land in sparse regions with few nearby tokens, while others compete for the same dense region. This leads to uneven expert utilization and suboptimal routing decisions. With $k$-means initialization (Figure~\ref{fig:kmneas_init}(d)), centers are positioned at the centroids of natural token clusters in the representation space. Each expert captures a distinct, well-separated region, leading to meaningful specialization where similar tokens are consistently routed to the same expert.

\begin{notebox}
\textit{\textbf{Key finding:} $K$-means initialization places centers at natural cluster centroids, producing well-separated routing regions. This improves performance by 0.5--1.0\% on most tasks. Without proper initialization, centers may land in suboptimal positions, leading to overlapping or sparse routing regions.}
\end{notebox}


\subsection{Expert Usage and Self-Balancing}
\label{app:ablation-balance}

A common challenge in MoE systems is expert collapse, where routing converges to use only a subset of experts while others receive little or no traffic. Traditional MoE methods address this with auxiliary load-balancing losses that explicitly penalize uneven expert utilization, which force experts to be balanced even when tokens are not genuinely similar to them. This artificial balancing can lead to suboptimal routing decisions where tokens are assigned to mismatched experts simply to satisfy the load constraint. MJ takes a different approach: by initializing centers via $k$-means clustering and updating them via EMA, we achieve natural load balancing without any auxiliary losses—tokens are routed based purely on genuine similarity to cluster centers.

Figure~\ref{fig:expert_uses} visualizes expert usage across all 24 Transformer layers, comparing usage patterns immediately after $k$-means initialization (before training) versus after training with EMA updates. Each row corresponds to one of the five experts (adapters), and the y-axis shows the percentage of tokens routed to that expert at each layer. The correlation coefficient $\rho$ measures how well the initial routing structure is preserved after training.

The results reveal several important properties of MJ's clustering-based routing. First, $k$-means initialization produces well-balanced expert usage from the start—all five experts receive 15--30\% of tokens at each layer, with no expert dominating or being ignored. This balance emerges naturally because $k$-means partitions the token representation space into clusters based on actual data density, not artificial constraints. Second, after training, the usage patterns shift but remain balanced. The EMA updates allow centers to track the evolving token distribution as adapters are trained, but the updates are gradual enough to prevent sudden routing collapse. No expert degrades to near-zero usage, and no expert monopolizes routing. Third, the high correlation coefficients ($\rho = 0.65$--$0.86$) indicate that EMA updates preserve the overall structure established by $k$-means while allowing local adaptation. Expert 5 shows the highest correlation ($\rho = 0.86$), meaning its routing pattern changed minimally, while Expert 4 shows the lowest ($\rho = 0.65$), indicating more adaptation—yet both maintain balanced usage.

The shaded regions highlight layers where usage shifted most between initialization and training. These shifts reflect adapters learning specialized functions that attract different token types—not artificial rebalancing. When one expert gains usage at a layer, others compensate naturally through the EMA mechanism, maintaining smooth load distribution without explicit penalties.

\begin{notebox}
\textit{\textbf{Key finding:} MJ achieves self-balancing without auxiliary losses. Unlike traditional MoE that forces artificial balance, MJ routes tokens based on genuine similarity to cluster centers. $K$-means ensures a balanced starting point; EMA updates preserve balance while allowing natural adaptation ($\rho = 0.65$--$0.86$).}
\end{notebox}

\begin{table}[t]
\small
\centering
\setlength{\tabcolsep}{5pt}
\scalebox{0.80}{
\begin{tabular}{l|cccc}
\hline
\rowcolor{headercolor}
\textbf{Method} &
\textbf{Space} &
\textbf{Time} &
\textbf{TPs} &
\textbf{RPs}  \\
\hline
LoRA &
$O(Edr)$ &
$O(Edr)$ &
$2Edr$ &
$0$ \\
MoE-LoRA &
$O(ENdr)$ &
$O(ENdr)$ &
$2ENdr$ &
$Nd$ \\
\rowcolor{ourmethodcolor}
MJ-LoRA  &
$O(Edr)$ &
$O(Kdr)$ &
$2Edr$ &
$\mathbf{0}$ \\
\hline
LoRA-FA &
$O(Edr)$ &
$O(Edr)$ &
$Edr$ &
$0$ \\
MoE-LoRA-FA &
$O(ENdr)$ &
$O(ENdr)$ &
$ENdr$ &
$Nd$ \\
\rowcolor{ourmethodcolor}
MJ-LoRA-FA  &
$O(Edr)$ &
$O(Kdr)$ &
$Edr$ &
$\mathbf{0}$ \\
\hline
AdaLoRA &
$O(Edr)$ &
$O(Edr)$ &
$E(2dr + r^2)$ &
$0$ \\
MoE-AdaLoRA &
$O(ENdr)$ &
$O(ENdr)$ &
$EN(2dr + r^2)$ &
$Nd$ \\
\rowcolor{ourmethodcolor}
MJ-AdaLoRA  &
$O(Edr)$ &
$O(Kdr)$ &
$E(2dr + r^2)$ &
$\mathbf{0}$ \\
\hline
Propulsion &
$O(Ed)$ &
$O(Ed)$ &
$Ed$ &
$0$ \\
MoE-Propulsion &
$O(ENd)$ &
$O(ENd)$ &
$ENd$ &
$Nd$ \\
\rowcolor{ourmethodcolor}
MJ-Propulsion  &
$O(Ed)$ &
$O(Kd)$ &
$Ed$ &
$\mathbf{0}$ \\
\hline
\end{tabular}}
\caption{
Per-block complexity and parameter comparison.
$E$ = number of submodules per block;
$N$ = number of MoE experts;
$K$ = top-$k$ adapters activated per token in MJ;
$d$ = hidden dimension;
$r$ = adapter rank.
MoE methods scale parameters and compute by $N$ and require a trainable router.
\textbf{MJ} matches base PEFT in trainable parameters with zero router parameters;
its sparse top-$K$ routing reduces the adapter term from $O(Edr)$ to $O(Kdr)$ and adds
a rank-independent routing term $O(Ed)$.
}
\label{table:peft_block_comparison}
\end{table}

\subsection{Complexity and Parameter Analysis}
\label{app:ablation-complexity}
We provide a detailed complexity analysis comparing standard PEFT, MoE-PEFT, and MJ variants in Table~\ref{table:peft_block_comparison}. The comparison covers space complexity, time complexity, trainable parameters (TPs), and router parameters (RPs) per Transformer block.

Standard PEFT methods (LoRA, LoRA-FA, AdaLoRA, Propulsion) have space and time complexity of $O(Edr)$ where $E$ is the number of equipped projections, $d$ is the hidden dimension, and $r$ is the adapter rank. They introduce no router parameters since all adapters are applied uniformly to every token.

MoE-PEFT methods scale both space and time complexity by $N$ (number of experts), resulting in $O(ENdr)$. This is because each projection now has $N$ expert adapters instead of one. Additionally, MoE-PEFT requires a learned router with $O(Nd)$ trainable parameters per block to determine expert selection. For a typical configuration with $N=4$ experts, this means $4\times$ more adapter parameters plus router overhead.

MJ keeps the space complexity of standard PEFT, $O(Edr)$, because it reuses the existing adapters as implicit experts without adding new ones, and it introduces \textbf{zero router parameters}---routing uses $k$-means centers updated by EMA, which are non-trainable buffers. Its per-token time complexity is $O(Kdr + Ed)$, where $K$ is the number of adapters activated per token: sparse top-$k$ selection reduces the adapter term from $O(Edr)$ to $O(Kdr)$, and the added $O(Ed)$ term is the cost of computing cosine similarities to the $E$ routing centers. In our main configuration, $E{=}5$ projections are equipped (Q, K, V, O, gate), of which O and gate are always-active shared adapters and Q, K, V are routed with top-$k{=}2$, so each token evaluates four of the five adapters.
The two terms trade off against each other, and which dominates depends on the rank. The routing term is independent of $r$, so at the small ranks we use ($r{=}2$ for text, $r{=}1$ for image and video) it is not negligible relative to the adapter term---a point worth stating explicitly, since $O(Ed)$ is small compared to $O(Edr)$ only when $r$ is large. It is, however, offset by the compute saved from skipping $E-K$ adapters per token, and it is small in absolute terms: computing $E$ cosine similarities over a $d$-dimensional vector amounts to a single small matrix--vector product per token. This is consistent with the measured throughput in Figure~\ref{fig:efficiency}, where all MJ variants train faster than every MoE-PEFT baseline at exactly these ranks. Because the routing centers are non-trainable, the $O(Ed)$ term also adds nothing to the backward pass.

Initialization adds a one-time cost outside training: a single forward pass over 5{,}000 sampled tokens to collect representations, followed by $k$-means over the cached activations. Clustering takes 0.13--0.41\,s per block depending on backbone width, or roughly 3\,s in total for the 24-block 0.5B model, 9\,s for the 28-block 7B model, and 13\,s for the 32-block 8B model. This cost is paid once per run and is independent of training length, so it amortizes to a negligible fraction of total wall-clock time.


\begin{figure*}[!htb]
    \centering
    \includegraphics[width=1.0\linewidth]{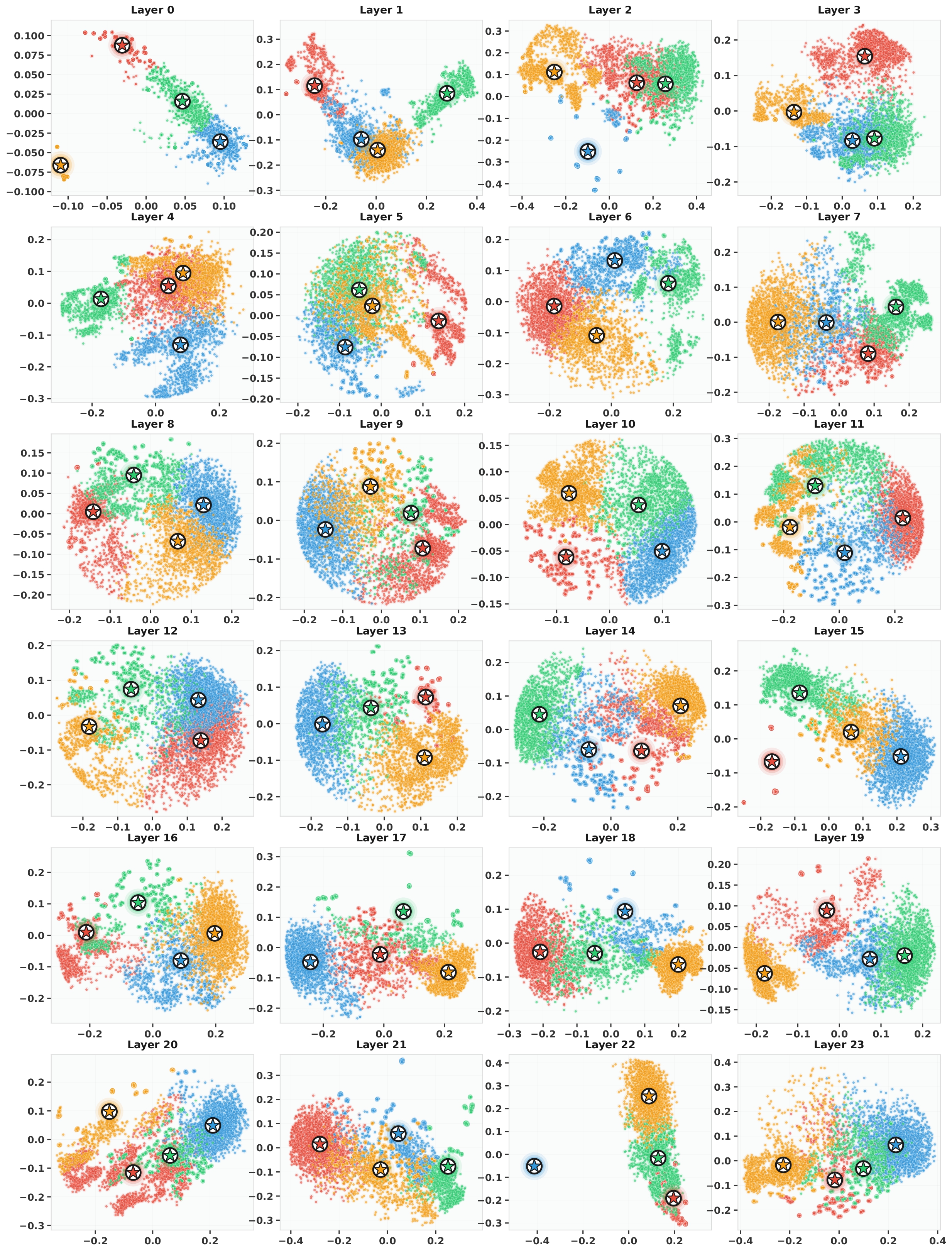}
    \caption{t-SNE visualization of token routing across all 24 Transformer layers. Each subplot shows token representations from 2K GLUE samples, colored by assigned expert (E0--E3). Stars mark cluster centers. Cluster structure evolves from diffuse (early layers) to well-separated (middle layers) to mixed patterns (later layers), with balanced expert coverage maintained throughout.}
    \label{fig:MJ_24layers}
\end{figure*}

\subsection{Layer-wise Cluster Visualization}
\label{app:ablation-layerwise}

To understand how MJ's routing behavior varies across the network, we visualize token clusters at each of the 24 Transformer layers. Figure~\ref{fig:MJ_24layers} shows t-SNE projections of token representations from 2K GLUE samples, colored by their assigned expert (E0--E3). Cluster centers learned by MJ are marked with stars.

The visualizations reveal that cluster structure evolves significantly across layers. In early layers (0--5), token representations form relatively diffuse clusters with moderate overlap between experts. This reflects the fact that early layers capture low-level features where tokens have not yet developed strong task-specific patterns. As we move to middle layers (6--15), clusters become more distinct and well-separated. The centers (stars) sit clearly at the centroids of their respective token groups, demonstrating that $k$-means initialization combined with EMA updates successfully tracks the natural structure of the representation space. In later layers (16--23), we observe two patterns: some layers maintain tight, well-separated clusters (e.g., layers 17, 19, 22), while others show more distributed patterns where experts cover broader regions (e.g., layers 20, 23). This suggests that later layers develop both specialized functions (tight clusters) and general functions (broader coverage) depending on the layer's role in the network.

Across all layers, the four experts maintain balanced coverage—no expert dominates the representation space, and no expert is marginalized to a tiny region. This confirms that MJ's clustering-based routing achieves natural load balancing across the entire network depth. The consistent positioning of centers at cluster centroids (rather than at cluster boundaries or in empty regions) demonstrates that both the $k$-means initialization and EMA updates work as intended.

\begin{notebox}
\textit{\textbf{Observation:} Cluster structure is layer-dependent. Early layers show diffuse clusters; middle layers show well-separated clusters; later layers show mixed patterns. MJ adapts to each layer's representation geometry while maintaining balanced expert coverage throughout.}
\end{notebox}


\clearpage


\begin{table*}[t]
\centering
\scriptsize
\setlength{\tabcolsep}{3pt}
\renewcommand{\arraystretch}{1.05}
\begin{tabular}{lccccccc}
\toprule
\rowcolor{headercolor}
\textbf{Method} & \textbf{SST-2} & \textbf{QNLI} & \textbf{QQP} & \textbf{CoLA} & \textbf{MRPC} & \textbf{STS-B} & \textbf{Avg} \\
\midrule
LoRA & $95.53_{\pm 0.71}$ & $92.46_{\pm 0.77}$ & $85.57_{\pm 0.82}$ & $84.31_{\pm 0.36}$ & $88.05_{\pm 0.38}$ & $\mathbf{91.58}_{\pm 0.29}$ & $89.58_{\pm 0.56}$ \\
AdaLoRA & $94.76_{\pm 0.57}$ & $91.55_{\pm 0.24}$ & $85.23_{\pm 0.86}$ & $84.29_{\pm 0.78}$ & $88.75_{\pm 0.83}$ & $90.33_{\pm 0.81}$ & $89.15_{\pm 0.68}$ \\
Propulsion & $94.14_{\pm 0.26}$ & $91.29_{\pm 0.51}$ & $85.26_{\pm 0.88}$ & $84.63_{\pm 0.91}$ & $87.23_{\pm 0.25}$ & $91.27_{\pm 0.37}$ & $88.97_{\pm 0.53}$ \\
LoRA-FA & $94.72_{\pm 0.52}$ & $92.01_{\pm 0.41}$ & $85.14_{\pm 0.31}$ & $84.02_{\pm 0.73}$ & $87.89_{\pm 0.56}$ & $90.55_{\pm 0.48}$ & $89.06_{\pm 0.50}$ \\
MoELoRA & $95.59_{\pm 0.74}$ & $92.56_{\pm 0.63}$ & $85.34_{\pm 0.71}$ & $85.00_{\pm 0.69}$ & $88.99_{\pm 0.54}$ & $90.37_{\pm 0.88}$ & $89.64_{\pm 0.70}$ \\
MixLoRA & $95.11_{\pm 0.31}$ & $92.33_{\pm 0.67}$ & $85.22_{\pm 0.29}$ & $84.62_{\pm 0.45}$ & $88.09_{\pm 0.81}$ & $90.88_{\pm 0.62}$ & $89.38_{\pm 0.53}$ \\
HydraLoRA & $95.28_{\pm 0.48}$ & $92.42_{\pm 0.36}$ & $85.56_{\pm 0.86}$ & $84.92_{\pm 0.55}$ & $89.08_{\pm 0.27}$ & $91.21_{\pm 0.24}$ & $89.75_{\pm 0.46}$ \\
MoLA & $\mathbf{95.85}_{\pm 0.55}$ & $92.83_{\pm 0.34}$ & $85.04_{\pm 0.79}$ & $85.04_{\pm 0.83}$ & $88.68_{\pm 0.41}$ & $91.00_{\pm 0.64}$ & $89.74_{\pm 0.59}$ \\
MoRE & $94.91_{\pm 0.52}$ & $\mathbf{93.30}_{\pm 0.57}$ & $85.46_{\pm 0.39}$ & $\mathbf{85.54}_{\pm 0.84}$ & $88.49_{\pm 0.48}$ & $90.64_{\pm 0.36}$ & $89.72_{\pm 0.53}$ \\
MoA & $95.10_{\pm 0.35}$ & $92.32_{\pm 0.29}$ & $85.67_{\pm 0.81}$ & $85.07_{\pm 0.58}$ & $88.39_{\pm 0.51}$ & $90.88_{\pm 0.73}$ & $89.57_{\pm 0.55}$ \\
\midrule
\rowcolor{ourmethodcolor}
MJLoRA & $95.74_{\pm 0.64}$ & $92.70_{\pm 0.23}$ & $85.93_{\pm 0.39}$ & $85.08_{\pm 0.36}$ & $88.76_{\pm 0.71}$ & $91.26_{\pm 0.67}$ & $\mathbf{89.91}_{\pm 0.50}$ \\
\rowcolor{ourmethodcolor}
MJAdaLoRA & $95.61_{\pm 0.25}$ & $92.11_{\pm 0.61}$ & $85.37_{\pm 0.54}$ & $85.38_{\pm 0.79}$ & $\mathbf{89.90}_{\pm 0.33}$ & $90.31_{\pm 0.85}$ & $89.78_{\pm 0.56}$ \\
\rowcolor{ourmethodcolor}
MJPropulsion & $95.56_{\pm 0.68}$ & $93.22_{\pm 0.34}$ & $\mathbf{85.97}_{\pm 0.62}$ & $84.19_{\pm 0.39}$ & $89.04_{\pm 0.52}$ & $90.79_{\pm 0.80}$ & $89.80_{\pm 0.56}$ \\
\rowcolor{ourmethodcolor}
MJLoRA-FA & $95.42_{\pm 0.47}$ & $92.88_{\pm 0.58}$ & $85.22_{\pm 0.44}$ & $85.43_{\pm 0.29}$ & $89.01_{\pm 0.74}$ & $91.44_{\pm 0.61}$ & $89.90_{\pm 0.52}$ \\
\bottomrule
\end{tabular}
\caption{Results on GLUE benchmarks. Highlighted rows denote our MJ variants. We have used Spearman correlation for STS-B, and Accuracy for rest of the datasets.}
\label{tab:glue_results}
\end{table*}

\begin{table*}[t]
\centering
\scriptsize
\setlength{\tabcolsep}{3pt}
\renewcommand{\arraystretch}{1.05}
\begin{tabular}{lccccccccc}
\toprule
\rowcolor{headercolor}
\textbf{Method} & \textbf{BoolQ} & \textbf{PIQA} & \textbf{SIQA} & \textbf{H.Sw.} & \textbf{W.Gra} & \textbf{ARC-e} & \textbf{ARC-c} & \textbf{OBQA} & \textbf{Avg} \\
\midrule
LoRA & $71.22_{\pm 0.61}$ & $73.48_{\pm 0.97}$ & $64.77_{\pm 1.40}$ & $51.04_{\pm 2.52}$ & $77.30_{\pm 0.27}$ & $84.47_{\pm 0.89}$ & $70.94_{\pm 1.66}$ & $77.59_{\pm 1.32}$ & $71.35_{\pm 1.21}$ \\
AdaLoRA & $70.10_{\pm 0.65}$ & $73.10_{\pm 0.57}$ & $64.06_{\pm 1.45}$ & $50.37_{\pm 2.30}$ & $77.00_{\pm 0.40}$ & $85.21_{\pm 0.41}$ & $70.78_{\pm 1.39}$ & $77.48_{\pm 1.32}$ & $71.01_{\pm 1.06}$ \\
Propulsion & $70.15_{\pm 0.33}$ & $73.45_{\pm 0.99}$ & $63.82_{\pm 1.66}$ & $50.55_{\pm 2.66}$ & $77.59_{\pm 0.38}$ & $84.30_{\pm 0.47}$ & $70.61_{\pm 1.64}$ & $77.26_{\pm 1.17}$ & $70.97_{\pm 1.16}$ \\
LoRA-FA & $70.41_{\pm 0.58}$ & $73.12_{\pm 0.59}$ & $64.02_{\pm 1.73}$ & $50.14_{\pm 2.22}$ & $77.41_{\pm 0.54}$ & $84.28_{\pm 0.59}$ & $70.64_{\pm 1.61}$ & $76.83_{\pm 1.08}$ & $70.86_{\pm 1.12}$ \\
MoELoRA  & $71.34_{\pm 0.39}$ & $73.81_{\pm 0.78}$ & $64.31_{\pm 1.64}$ & $50.57_{\pm 2.52}$ & $78.37_{\pm 0.65}$ & $84.70_{\pm 0.70}$ & $70.65_{\pm 1.57}$ & $76.73_{\pm 0.81}$ & $71.31_{\pm 1.13}$ \\
MixLoRA & $71.31_{\pm 0.49}$ & $73.49_{\pm 0.86}$ & $64.50_{\pm 1.64}$ & $50.75_{\pm 2.61}$ & $77.53_{\pm 0.31}$ & $85.83_{\pm 0.40}$ & $71.64_{\pm 1.35}$ & $77.65_{\pm 0.93}$ & $71.59_{\pm 1.07}$ \\
HydraLoRA & $\mathbf{71.79}_{\pm 0.24}$ & $73.77_{\pm 0.83}$ & $64.50_{\pm 1.62}$ & $51.00_{\pm 2.41}$ & $78.12_{\pm 0.21}$ & $84.90_{\pm 0.59}$ & $71.64_{\pm 1.73}$ & $\mathbf{78.15}_{\pm 1.07}$ & $71.73_{\pm 1.09}$ \\
MoLA & $70.53_{\pm 0.51}$ & $73.60_{\pm 0.56}$ & $63.96_{\pm 1.60}$ & $50.87_{\pm 2.48}$ & $77.57_{\pm 0.61}$ & $85.37_{\pm 0.67}$ & $71.57_{\pm 1.18}$ & $76.91_{\pm 0.96}$ & $71.30_{\pm 1.07}$ \\
MoRE & $71.64_{\pm 0.48}$ & $\mathbf{74.36}_{\pm 0.88}$ & $64.15_{\pm 1.29}$ & $51.23_{\pm 2.35}$ & $77.33_{\pm 0.73}$ & $85.49_{\pm 0.74}$ & $71.79_{\pm 1.42}$ & $76.97_{\pm 1.06}$ & $71.62_{\pm 1.12}$ \\
MoA & $70.58_{\pm 0.62}$ & $73.76_{\pm 0.58}$ & $64.36_{\pm 1.23}$ & $51.56_{\pm 2.42}$ & $78.19_{\pm 0.41}$ & $84.60_{\pm 0.64}$ & $71.06_{\pm 1.58}$ & $77.56_{\pm 1.15}$ & $71.46_{\pm 1.08}$ \\
\midrule
\rowcolor{ourmethodcolor}
MJLoRA & $71.11_{\pm 0.67}$ & $73.73_{\pm 0.53}$ & $64.73_{\pm 1.41}$ & $\mathbf{51.68}_{\pm 2.20}$ & $77.93_{\pm 0.30}$ & $85.13_{\pm 0.66}$ & $71.42_{\pm 1.19}$ & $78.14_{\pm 0.89}$ & $71.73_{\pm 0.98}$ \\
\rowcolor{ourmethodcolor}
MJAdaLoRA & $71.52_{\pm 0.22}$ & $73.14_{\pm 0.58}$ & $64.15_{\pm 1.51}$ & $50.12_{\pm 2.55}$ & $\mathbf{78.79}_{\pm 0.31}$ & $85.08_{\pm 0.45}$ & $\mathbf{71.98}_{\pm 1.67}$ & $76.91_{\pm 0.92}$ & $71.46_{\pm 1.03}$ \\
\rowcolor{ourmethodcolor}
MJPropulsion & $71.76_{\pm 0.65}$ & $74.32_{\pm 0.53}$ & $64.02_{\pm 1.41}$ & $50.52_{\pm 2.33}$ & $77.18_{\pm 0.18}$ & $85.03_{\pm 0.80}$ & $70.91_{\pm 1.53}$ & $77.46_{\pm 0.92}$ & $71.40_{\pm 1.04}$ \\
\rowcolor{ourmethodcolor}
MJLoRA-FA & $71.38_{\pm 0.27}$ & $74.04_{\pm 0.62}$ & $\mathbf{64.94}_{\pm 1.46}$ & $51.55_{\pm 2.64}$ & $78.51_{\pm 0.35}$ & $\mathbf{85.86}_{\pm 0.61}$ & $71.39_{\pm 1.58}$ & $77.92_{\pm 0.90}$ & $\mathbf{71.95}_{\pm 1.05}$ \\
\bottomrule
\end{tabular}
\caption{Results on commonsense reasoning and question answering benchmarks. Highlighted rows denote our MJ variants. Accuracy is used as the evaluation metric for all datasets.}
\label{tab:qa_results}
\end{table*}
\begin{table*}[t]
\centering
\scriptsize
\setlength{\tabcolsep}{3pt}
\renewcommand{\arraystretch}{1.05}
\begin{tabular}{lccccccc}
\toprule
\rowcolor{headercolor}
\textbf{Method} & \textbf{Camelyon} & \textbf{SVHN} & \textbf{Pets} & \textbf{Flowers102} & \textbf{EuroSAT} & \textbf{Caltech101} & \textbf{Avg} \\
\midrule
LoRA & $89.07_{\pm 0.72}$ & $93.81_{\pm 0.33}$ & $95.55_{\pm 0.91}$ & $94.60_{\pm 0.26}$ & $96.31_{\pm 0.45}$ & $95.22_{\pm 0.23}$ & $94.09_{\pm 0.48}$ \\
AdaLoRA & $88.62_{\pm 0.77}$ & $93.28_{\pm 0.48}$ & $95.03_{\pm 0.44}$ & $94.14_{\pm 0.39}$ & $95.41_{\pm 0.94}$ & $94.70_{\pm 0.66}$ & $93.53_{\pm 0.61}$ \\
Propulsion & $88.71_{\pm 0.61}$ & $93.41_{\pm 0.89}$ & $95.22_{\pm 0.35}$ & $94.39_{\pm 0.22}$ & $95.54_{\pm 0.43}$ & $94.48_{\pm 0.31}$ & $93.63_{\pm 0.47}$ \\
LoRA-FA & $88.89_{\pm 0.54}$ & $93.14_{\pm 0.82}$ & $95.11_{\pm 0.46}$ & $94.39_{\pm 0.37}$ & $95.55_{\pm 0.56}$ & $94.61_{\pm 0.21}$ & $93.62_{\pm 0.49}$ \\
MoELoRA & $89.11_{\pm 0.41}$ & $\mathbf{95.28}_{\pm 0.58}$ & $95.56_{\pm 0.66}$ & $94.33_{\pm 0.32}$ & $96.22_{\pm 0.71}$ & $95.31_{\pm 0.49}$ & $94.30_{\pm 0.53}$ \\
MixLoRA & $89.44_{\pm 0.63}$ & $94.71_{\pm 0.29}$ & $95.37_{\pm 0.47}$ & $95.14_{\pm 0.84}$ & $\mathbf{96.92}_{\pm 0.44}$ & $96.27_{\pm 0.56}$ & $94.64_{\pm 0.54}$ \\
HydraLoRA & $90.08_{\pm 0.36}$ & $95.11_{\pm 0.61}$ & $96.42_{\pm 0.92}$ & $95.33_{\pm 0.24}$ & $95.28_{\pm 0.55}$ & $\mathbf{96.95}_{\pm 0.18}$ & $\mathbf{94.86}_{\pm 0.48}$ \\
MoLA & $88.94_{\pm 0.88}$ & $93.91_{\pm 0.52}$ & $95.28_{\pm 0.59}$ & $94.76_{\pm 0.46}$ & $96.02_{\pm 0.33}$ & $95.24_{\pm 0.41}$ & $94.03_{\pm 0.53}$ \\
MoRE & $90.03_{\pm 0.49}$ & $94.60_{\pm 0.41}$ & $\mathbf{96.97}_{\pm 0.35}$ & $\mathbf{95.41}_{\pm 0.64}$ & $95.01_{\pm 0.79}$ & $96.44_{\pm 0.38}$ & $94.74_{\pm 0.51}$ \\
MoA & $89.51_{\pm 0.27}$ & $94.14_{\pm 0.33}$ & $96.18_{\pm 0.51}$ & $94.94_{\pm 0.19}$ & $96.48_{\pm 0.31}$ & $95.52_{\pm 0.73}$ & $94.46_{\pm 0.39}$ \\
\midrule
\rowcolor{ourmethodcolor}
MJLoRA  & $89.68_{\pm 0.52}$ & $93.46_{\pm 0.66}$ & $93.88_{\pm 0.25}$ & $93.83_{\pm 0.49}$ & $94.64_{\pm 0.41}$ & $94.83_{\pm 0.33}$ & $93.39_{\pm 0.44}$ \\
\rowcolor{ourmethodcolor}
MJAdaLoRA & $\mathbf{90.52}_{\pm 0.34}$ & $93.78_{\pm 0.43}$ & $94.69_{\pm 0.53}$ & $95.39_{\pm 0.64}$ & $94.02_{\pm 0.38}$ & $94.61_{\pm 0.50}$ & $93.84_{\pm 0.47}$ \\
\rowcolor{ourmethodcolor}
MJPropulsion & $89.04_{\pm 0.74}$ & $93.77_{\pm 0.53}$ & $95.61_{\pm 0.69}$ & $94.30_{\pm 0.58}$ & $94.00_{\pm 0.86}$ & $94.34_{\pm 0.62}$ & $93.51_{\pm 0.67}$ \\
\rowcolor{ourmethodcolor}
MJLoRA-FA & $90.01_{\pm 0.62}$ & $94.02_{\pm 0.47}$ & $96.25_{\pm 0.34}$ & $94.07_{\pm 0.95}$ & $94.00_{\pm 0.78}$ & $95.82_{\pm 0.71}$ & $94.03_{\pm 0.65}$ \\
\bottomrule
\end{tabular}
\caption{Results on \textbf{image classification benchmarks}. Each dataset consists of images annotated with a single class label. Highlighted rows denote our MJ variants.}
\label{tab:image_cls}
\end{table*}

\begin{table*}[t]
\centering
\scriptsize
\setlength{\tabcolsep}{3pt}
\renewcommand{\arraystretch}{1.05}
\begin{tabular}{lccccccccc}
\toprule
\rowcolor{headercolor}
\textbf{Method} & \textbf{ChartQA} & \textbf{OKVQA} & \textbf{ScienceQA} & \textbf{SeedBench} &
\textbf{Recognition} & \textbf{TextVQA} & \textbf{VizWizVQA} & \textbf{VQA-RAD} & \textbf{Avg} \\
\midrule
LoRA & $76.38_{\pm 1.67}$ & $57.12_{\pm 2.31}$ & $81.51_{\pm 1.96}$ & $71.90_{\pm 1.23}$ & $95.18_{\pm 1.84}$ & $66.07_{\pm 2.13}$ & $51.48_{\pm 2.28}$ & $72.54_{\pm 2.35}$ & $71.52_{\pm 1.97}$ \\
AdaLoRA & $75.77_{\pm 1.78}$ & $56.41_{\pm 2.02}$ & $80.34_{\pm 2.24}$ & $71.20_{\pm 1.51}$ & $96.49_{\pm 1.89}$ & $65.94_{\pm 2.31}$ & $50.54_{\pm 1.37}$ & $71.65_{\pm 2.28}$ & $71.04_{\pm 1.93}$ \\
Propulsion & $75.64_{\pm 1.33}$ & $56.74_{\pm 1.57}$ & $80.51_{\pm 2.48}$ & $71.50_{\pm 2.29}$ & $96.66_{\pm 1.61}$ & $66.03_{\pm 2.16}$ & $50.91_{\pm 1.35}$ & $71.92_{\pm 1.98}$ & $71.24_{\pm 1.85}$ \\
LoRA-FA & $75.82_{\pm 2.41}$ & $56.32_{\pm 1.84}$ & $80.60_{\pm 2.11}$ & $71.14_{\pm 1.32}$ & $96.66_{\pm 2.23}$ & $65.90_{\pm 1.76}$ & $50.44_{\pm 1.49}$ & $71.61_{\pm 2.08}$ & $71.06_{\pm 1.91}$ \\
MoELoRA & $76.54_{\pm 1.52}$ & $\mathbf{58.95}_{\pm 2.08}$ & $81.12_{\pm 1.44}$ & $72.06_{\pm 1.76}$ & $95.02_{\pm 1.95}$ & $66.44_{\pm 1.61}$ & $51.77_{\pm 1.88}$ & $72.44_{\pm 2.02}$ & $71.79_{\pm 1.78}$ \\
MixLoRA & $77.31_{\pm 1.85}$ & $57.98_{\pm 1.49}$ & $82.11_{\pm 2.16}$ & $72.87_{\pm 1.72}$ & $\mathbf{96.71}_{\pm 1.58}$ & $67.22_{\pm 1.34}$ & $51.63_{\pm 1.47}$ & $73.44_{\pm 1.91}$ & $72.41_{\pm 1.69}$ \\
HydraLoRA & $76.64_{\pm 1.43}$ & $58.34_{\pm 2.24}$ & $\mathbf{82.79}_{\pm 1.68}$ & $70.53_{\pm 1.36}$ & $95.14_{\pm 2.12}$ & $\mathbf{67.61}_{\pm 1.51}$ & $51.06_{\pm 2.29}$ & $72.68_{\pm 1.74}$ & $71.85_{\pm 1.80}$ \\
MoLA & $77.42_{\pm 1.74}$ & $57.89_{\pm 1.58}$ & $82.78_{\pm 1.89}$ & $73.56_{\pm 1.31}$ & $96.11_{\pm 1.27}$ & $66.98_{\pm 1.46}$ & $52.14_{\pm 2.42}$ & $73.44_{\pm 2.13}$ & $\mathbf{72.54}_{\pm 1.73}$ \\
MoRE & $\mathbf{77.52}_{\pm 1.79}$ & $57.84_{\pm 1.87}$ & $82.05_{\pm 1.28}$ & $\mathbf{73.58}_{\pm 1.64}$ & $96.02_{\pm 2.01}$ & $67.11_{\pm 1.73}$ & $51.74_{\pm 1.84}$ & $73.91_{\pm 1.32}$ & $72.47_{\pm 1.69}$ \\
MoA & $76.48_{\pm 2.04}$ & $57.11_{\pm 1.69}$ & $81.53_{\pm 1.81}$ & $72.00_{\pm 1.39}$ & $95.33_{\pm 1.48}$ & $66.58_{\pm 2.07}$ & $51.47_{\pm 2.22}$ & $72.91_{\pm 1.26}$ & $71.68_{\pm 1.75}$ \\
\midrule
\rowcolor{ourmethodcolor}
MJLoRA & $76.94_{\pm 1.92}$ & $57.43_{\pm 1.79}$ & $81.73_{\pm 2.04}$ & $72.43_{\pm 1.68}$ & $95.79_{\pm 1.55}$ & $67.42_{\pm 2.33}$ & $51.32_{\pm 1.28}$ & $73.00_{\pm 2.07}$ & $72.01_{\pm 1.83}$ \\
\rowcolor{ourmethodcolor}
MJAdaLoRA & $76.04_{\pm 1.66}$ & $56.18_{\pm 1.95}$ & $80.58_{\pm 2.28}$ & $71.38_{\pm 2.19}$ & $95.48_{\pm 2.01}$ & $65.28_{\pm 1.24}$ & $\mathbf{52.66}_{\pm 1.56}$ & $\mathbf{73.92}_{\pm 1.71}$ & $71.44_{\pm 1.83}$ \\
\rowcolor{ourmethodcolor}
MJPropulsion & $77.24_{\pm 1.81}$ & $58.31_{\pm 1.66}$ & $82.01_{\pm 1.39}$ & $72.78_{\pm 1.58}$ & $96.61_{\pm 2.32}$ & $67.19_{\pm 1.87}$ & $52.22_{\pm 2.25}$ & $73.44_{\pm 1.43}$ & $72.48_{\pm 1.79}$ \\
\rowcolor{ourmethodcolor}
MJLoRA-FA & $76.31_{\pm 1.97}$ & $57.19_{\pm 1.31}$ & $81.04_{\pm 2.03}$ & $71.81_{\pm 1.82}$ & $95.14_{\pm 1.46}$ & $66.26_{\pm 1.98}$ & $51.02_{\pm 2.45}$ & $72.29_{\pm 1.65}$ & $71.38_{\pm 1.83}$ \\
\bottomrule
\end{tabular}
\caption{Results on vision--language question answering benchmarks. Highlighted rows denote our MJ variants.}
\label{tab:vqa}
\end{table*}

\begin{table*}[t]
\centering
\scriptsize
\setlength{\tabcolsep}{3pt}
\renewcommand{\arraystretch}{1.05}
\begin{tabular}{lccccccccc}
\rowcolor{headercolor}
\toprule
\textbf{Method} &
\textbf{ActSeq} & \textbf{ActPred} & \textbf{ActAnt} & \textbf{FineAct} & \textbf{UnexpAct} &
\textbf{ObjExist} & \textbf{ObjInter} & \textbf{ObjShuffle} & \textbf{Avg} \\
\midrule
LoRA & $40.83_{\pm1.19}$ & $42.39_{\pm0.89}$ & $53.74_{\pm1.42}$ & $39.41_{\pm1.29}$ & $41.57_{\pm0.95}$
& $52.28_{\pm1.22}$ & $41.44_{\pm1.04}$ & $54.91_{\pm1.36}$ & $45.82_{\pm1.17}$ \\
AdaLoRA & $40.27_{\pm1.21}$ & $41.93_{\pm0.97}$ & $53.21_{\pm1.48}$ & $39.04_{\pm1.37}$ & $41.05_{\pm1.14}$
& $51.66_{\pm1.35}$ & $40.88_{\pm0.95}$ & $54.11_{\pm1.18}$ & $45.27_{\pm1.21}$ \\
Propulsion & $40.19_{\pm1.36}$ & $41.88_{\pm1.13}$ & $53.08_{\pm0.96}$ & $38.92_{\pm1.41}$ & $40.87_{\pm1.03}$
& $51.41_{\pm1.24}$ & $40.61_{\pm1.33}$ & $53.82_{\pm1.08}$ & $45.10_{\pm1.19}$ \\
LoRA-FA & $40.11_{\pm1.28}$ & $41.74_{\pm1.41}$ & $53.09_{\pm1.06}$ & $38.96_{\pm1.18}$ & $40.88_{\pm1.39}$
& $51.52_{\pm0.97}$ & $40.73_{\pm1.31}$ & $53.97_{\pm1.12}$ & $45.13_{\pm1.22}$ \\
MoELoRA & $40.91_{\pm1.33}$ & $42.48_{\pm1.06}$ & $53.82_{\pm1.41}$ & $39.44_{\pm0.92}$ & $41.75_{\pm1.24}$
& $52.46_{\pm1.18}$ & $41.53_{\pm1.09}$ & $54.92_{\pm1.38}$ & $45.91_{\pm1.20}$ \\
MixLoRA & $41.98_{\pm0.98}$ & $43.91_{\pm1.27}$ & $\mathbf{56.04}_{\pm1.02}$ & $40.76_{\pm1.13}$ & $42.89_{\pm1.07}$
& $53.84_{\pm1.36}$ & $42.78_{\pm0.85}$ & $56.28_{\pm1.21}$ & $47.31_{\pm1.11}$ \\
HydraLoRA & $42.37_{\pm0.91}$ & $\mathbf{44.72}_{\pm1.03}$ & $55.42_{\pm1.37}$ & $\mathbf{41.82}_{\pm1.16}$ & $43.33_{\pm1.42}$
& $\mathbf{55.03}_{\pm1.10}$ & $43.22_{\pm1.23}$ & $56.89_{\pm0.96}$ & $47.85_{\pm1.15}$ \\
MoLA & $40.88_{\pm1.44}$ & $42.31_{\pm1.19}$ & $53.61_{\pm1.11}$ & $39.52_{\pm1.32}$ & $41.64_{\pm0.88}$
& $52.19_{\pm1.05}$ & $41.32_{\pm1.46}$ & $54.73_{\pm1.12}$ & $45.78_{\pm1.20}$ \\
MoRE & $42.18_{\pm1.05}$ & $44.38_{\pm1.13}$ & $55.71_{\pm1.18}$ & $41.44_{\pm0.97}$ & $43.58_{\pm1.22}$ & $55.12_{\pm1.08}$ & $43.42_{\pm1.16}$ & $56.61_{\pm1.21}$ & $47.81_{\pm1.13}$ \\
MoA & $41.22_{\pm0.89}$ & $42.86_{\pm1.33}$ & $54.48_{\pm1.19}$ & $40.03_{\pm0.99}$ & $42.04_{\pm1.31}$
& $52.91_{\pm1.07}$ & $41.88_{\pm1.38}$ & $55.37_{\pm0.85}$ & $46.35_{\pm1.13}$ \\
\midrule
\rowcolor{ourmethodcolor}
MJLoRA & $41.64_{\pm1.11}$ & $43.20_{\pm1.02}$ & $56.08_{\pm1.24}$ & $40.12_{\pm0.91}$ & $43.94_{\pm1.19}$ & $53.20_{\pm0.98}$ & $42.11_{\pm1.34}$ & $55.64_{\pm1.09}$ & $46.99_{\pm1.11}$ \\
\rowcolor{ourmethodcolor}
MJAdaLoRA & $\mathbf{42.88}_{\pm1.00}$ & $44.51_{\pm1.26}$ & $55.83_{\pm1.12}$ & $41.59_{\pm0.90}$ & $43.72_{\pm1.37}$ & $54.81_{\pm1.02}$ & $\mathbf{43.74}_{\pm1.23}$ & $\mathbf{57.36}_{\pm1.43}$ & $\mathbf{48.06}_{\pm1.17}$ \\
\rowcolor{ourmethodcolor}
MJPropulsion & $42.36_{\pm1.38}$ & $43.72_{\pm1.08}$ & $55.47_{\pm1.29}$ & $40.82_{\pm1.03}$ & $43.29_{\pm1.35}$ & $53.91_{\pm1.14}$ & $43.19_{\pm0.99}$ & $56.93_{\pm1.40}$ & $47.46_{\pm1.21}$ \\
\rowcolor{ourmethodcolor}
MJLoRA-FA & $42.11_{\pm1.01}$ & $43.92_{\pm1.39}$ & $55.04_{\pm0.87}$ & $40.98_{\pm1.28}$ & $\mathbf{43.95}_{\pm1.21}$
& $54.02_{\pm1.34}$ & $42.89_{\pm0.96}$ & $56.41_{\pm1.42}$ & $47.42_{\pm1.19}$ \\
\bottomrule
\end{tabular}
\caption{ Results on \textbf{action and object-centric reasoning subtasks}. Highlighted rows denote our MJ variants.}
\label{tab:action_object_results}
\end{table*}

\begin{table*}[t]
\centering
\scriptsize
\setlength{\tabcolsep}{3pt}
\renewcommand{\arraystretch}{1.05}
\begin{tabular}{lccccccc}
\toprule
\rowcolor{headercolor}
\textbf{Method} &
\textbf{MoveDir} & \textbf{ActLoc} & \textbf{SceneTrans} &
\textbf{ActCount} & \textbf{MoveCount} & \textbf{MoveAttr} & \textbf{Avg} \\
\midrule
LoRA & $47.62_{\pm1.09}$ & $54.01_{\pm1.41}$ & $60.32_{\pm0.85}$ & $58.62_{\pm1.30}$ & $48.39_{\pm0.96}$ & $47.11_{\pm1.33}$ & $52.68_{\pm1.16}$ \\
AdaLoRA & $46.72_{\pm1.11}$ & $53.12_{\pm1.46}$ & $59.63_{\pm0.91}$ & $57.92_{\pm1.16}$ & $47.64_{\pm1.38}$ & $46.31_{\pm1.02}$ & $51.89_{\pm1.17}$ \\
Propulsion & $46.41_{\pm1.43}$ & $52.84_{\pm0.89}$ & $59.44_{\pm1.14}$ & $57.66_{\pm0.95}$ & $47.21_{\pm1.24}$ & $45.98_{\pm1.37}$ & $51.59_{\pm1.17}$ \\
LoRA-FA & $46.53_{\pm0.93}$ & $52.91_{\pm1.39}$ & $59.61_{\pm1.01}$ & $57.83_{\pm1.17}$ & $47.31_{\pm1.36}$ & $46.02_{\pm0.98}$ & $51.70_{\pm1.14}$ \\
MoELoRA & $47.85_{\pm0.97}$ & $54.21_{\pm1.16}$ & $60.55_{\pm0.88}$ & $58.83_{\pm1.41}$ & $48.61_{\pm0.92}$ & $47.20_{\pm1.08}$ & $52.88_{\pm1.07}$ \\
MixLoRA & $49.02_{\pm1.12}$ & $55.41_{\pm1.44}$ & $62.01_{\pm0.94}$ & $60.23_{\pm1.29}$ & $49.94_{\pm1.21}$ & $48.41_{\pm0.89}$ & $54.17_{\pm1.15}$ \\
HydraLoRA & $49.51_{\pm1.38}$ & $\mathbf{56.51}_{\pm0.96}$ & $63.18_{\pm1.33}$ & $60.79_{\pm1.05}$ & $50.89_{\pm1.42}$ & $48.87_{\pm1.11}$ & $54.96_{\pm1.21}$ \\
MoLA & $47.41_{\pm1.03}$ & $53.94_{\pm1.31}$ & $60.14_{\pm1.09}$ & $58.41_{\pm0.94}$ & $48.12_{\pm1.23}$ & $46.89_{\pm0.90}$ & $52.49_{\pm1.08}$ \\
MoRE & $49.33_{\pm1.08}$ & $55.52_{\pm1.19}$ & $62.14_{\pm0.86}$ & $60.46_{\pm1.35}$ & $\mathbf{51.01}_{\pm0.99}$ & $\mathbf{49.38}_{\pm1.26}$ & $54.64_{\pm1.12}$ \\
MoA & $48.11_{\pm1.34}$ & $54.47_{\pm0.95}$ & $61.01_{\pm1.12}$ & $59.28_{\pm1.31}$ & $48.89_{\pm0.87}$ & $47.61_{\pm1.40}$ & $53.23_{\pm1.17}$ \\
\midrule
\rowcolor{ourmethodcolor}
MJLoRA & $49.31_{\pm1.07}$ & $55.92_{\pm1.12}$ & $\mathbf{63.78}_{\pm1.19}$ & $60.97_{\pm1.21}$ & $50.82_{\pm1.48}$ & $49.31_{\pm1.24}$ & $\mathbf{55.02}_{\pm1.22}$ \\
\rowcolor{ourmethodcolor}
MJAdaLoRA & $\mathbf{49.88}_{\pm1.15}$ & $54.27_{\pm0.94}$ & $62.91_{\pm1.34}$ & $\mathbf{61.14}_{\pm1.08}$ & $49.61_{\pm1.42}$ & $49.12_{\pm1.19}$ & $54.49_{\pm1.19}$ \\
\rowcolor{ourmethodcolor}
MJPropulsion & $49.47_{\pm0.91}$ & $56.48_{\pm1.26}$ & $62.58_{\pm1.11}$ & $60.84_{\pm1.39}$ & $50.86_{\pm1.02}$ & $48.74_{\pm1.31}$ & $54.83_{\pm1.17}$ \\
\rowcolor{ourmethodcolor}
MJLoRA-FA & $48.38_{\pm1.40}$ & $54.83_{\pm1.01}$ & $63.14_{\pm1.23}$ & $59.49_{\pm1.13}$ & $49.18_{\pm0.82}$ & $\mathbf{49.38}_{\pm1.29}$ & $54.07_{\pm1.15}$ \\
\bottomrule
\end{tabular}
\caption{Results on motion and scene understanding subtasks.}
\label{tab:peft_motion_scene}
\end{table*}

\begin{table*}[t]
\centering
\scriptsize
\setlength{\tabcolsep}{3pt}
\renewcommand{\arraystretch}{1.05}
\begin{tabular}{lcccccc}
\rowcolor{headercolor}
\toprule
\textbf{Method} &
\textbf{StateChg} & \textbf{CharOrd} & \textbf{EgoNav} &
\textbf{EpisReason} & \textbf{CounterFact} & \textbf{Avg} \\
\midrule
LoRA & $59.94_{\pm1.11}$ & $57.39_{\pm1.33}$ & $58.92_{\pm1.01}$ & $61.29_{\pm1.19}$ & $45.61_{\pm1.45}$ & $56.63_{\pm1.22}$ \\
AdaLoRA & $59.12_{\pm1.29}$ & $56.43_{\pm1.09}$ & $58.12_{\pm1.47}$ & $60.54_{\pm0.98}$ & $44.96_{\pm1.36}$ & $55.83_{\pm1.24}$ \\
Propulsion & $58.83_{\pm0.95}$ & $56.11_{\pm1.24}$ & $57.89_{\pm1.41}$ & $60.21_{\pm0.84}$ & $44.72_{\pm1.27}$ & $55.55_{\pm1.14}$ \\
LoRA-FA & $58.96_{\pm1.26}$ & $56.29_{\pm1.28}$ & $58.11_{\pm1.34}$ & $60.48_{\pm1.07}$ & $44.83_{\pm1.48}$ & $55.73_{\pm1.29}$ \\
MoELoRA & $60.03_{\pm1.08}$ & $57.66_{\pm1.31}$ & $59.24_{\pm0.82}$ & $61.53_{\pm1.14}$ & $45.71_{\pm1.37}$ & $56.83_{\pm1.14}$ \\
MixLoRA & $61.55_{\pm0.86}$ & $59.08_{\pm1.23}$ & $60.71_{\pm1.29}$ & $\mathbf{64.08}_{\pm1.02}$ & $47.12_{\pm1.44}$ & $58.51_{\pm1.17}$ \\
HydraLoRA & $61.98_{\pm1.27}$ & $\mathbf{60.27}_{\pm0.94}$ & $\mathbf{61.96}_{\pm1.11}$
& $63.36_{\pm1.36}$ & $47.69_{\pm1.03}$ & $59.05_{\pm1.14}$ \\
MoLA & $62.14_{\pm0.95}$ & $60.31_{\pm1.13}$ & $61.58_{\pm1.21}$ & $63.52_{\pm1.18}$ & $47.94_{\pm1.12}$ & $59.10_{\pm1.12}$ \\
MoRE & $61.71_{\pm1.13}$ & $59.26_{\pm1.46}$ & $60.98_{\pm0.99}$ & $63.05_{\pm1.21}$ & $47.43_{\pm1.34}$ & $58.49_{\pm1.23}$ \\
MoA & $60.29_{\pm1.04}$ & $58.03_{\pm1.25}$ & $59.49_{\pm1.28}$ & $61.88_{\pm0.89}$ & $46.12_{\pm1.12}$ & $57.16_{\pm1.12}$ \\
\midrule
\rowcolor{ourmethodcolor}
MJLoRA & $60.74_{\pm1.02}$ & $58.32_{\pm0.87}$ & $62.04_{\pm1.42}$ & $62.19_{\pm0.96}$ & $46.43_{\pm1.21}$ & $57.94_{\pm1.10}$ \\
\rowcolor{ourmethodcolor}
MJAdaLoRA & $\mathbf{62.61}_{\pm0.89}$ & $58.02_{\pm1.32}$ & $60.74_{\pm0.98}$ & $63.78_{\pm1.37}$ & $\mathbf{48.22}_{\pm1.09}$ & $58.67_{\pm1.13}$ \\
\rowcolor{ourmethodcolor}
MJPropulsion & $61.99_{\pm0.91}$ & $60.18_{\pm1.26}$ & $61.12_{\pm1.12}$ & $63.98_{\pm1.31}$ & $47.81_{\pm1.04}$ & $\mathbf{59.02}_{\pm1.13}$ \\
\rowcolor{ourmethodcolor}
MJLoRA-FA & $59.71_{\pm1.42}$ & $57.01_{\pm1.18}$ & $58.74_{\pm1.06}$ & $61.02_{\pm1.33}$ & $45.38_{\pm0.91}$ & $56.37_{\pm1.18}$ \\
\bottomrule
\end{tabular}
\caption{Results on high-level reasoning subtasks.}
\label{tab:video_reasoning}
\end{table*}

\section{Detailed Results}
\label{app:detailed-results}

We provide complete per-task results for all 47 benchmarks across text, image, and video modalities. All experiments report mean accuracy with standard deviation over 5 runs on different seeds.

\paragraph{Text benchmarks.}

Table~\ref{tab:glue_results} shows results on GLUE~\cite{wang2018glue} benchmarks covering sentiment analysis (SST-2), natural language inference (QNLI), paraphrase detection (QQP, MRPC), linguistic acceptability (CoLA), and semantic similarity (STS-B). MJ variants achieve competitive performance with MoE-PEFT baselines: MJPropulsion achieves the best QQP score (85.97\%), MJAdaLoRA achieves the best MRPC score (89.90\%), and MJLoRA achieves the highest overall average (89.91\%).

Table~\ref{tab:qa_results} shows results on commonsense reasoning and question answering benchmarks. These include reading comprehension (BoolQ), physical reasoning (PIQA), social reasoning (SIQA), sentence completion (HellaSwag), pronoun resolution (WinoGrande), and science QA (ARC-Easy, ARC-Challenge, OpenBookQA). MJ variants dominate this benchmark: MJAdaLoRA achieves the best WinoGrande (78.79\%) and ARC-Challenge (71.98\%) scores, MJLoRA achieves the best HellaSwag score (51.68\%), and MJLoRA-FA achieves the best SIQA (64.94\%) and ARC-Easy (85.86\%) scores, as well as the highest overall average (71.95\%).

\paragraph{Image benchmarks.}

Table~\ref{tab:image_cls} shows results on image classification tasks spanning medical imaging (Camelyon), digit recognition (SVHN), fine-grained recognition (Pets, Flowers-102, Caltech-101), and satellite imagery (EuroSAT). MJAdaLoRA achieves the best performance on Camelyon (90.52\%), demonstrating strong performance on medical imaging tasks.

Table~\ref{tab:vqa} shows results on vision-language question answering benchmarks including chart understanding (ChartQA), external knowledge (OK-VQA), science questions (ScienceQA), visual reasoning (SEED-Bench), text recognition (Recognition), scene text (TextVQA), accessibility (VizWiz-VQA), and medical imaging (VQA-RAD). MJAdaLoRA achieves the best VizWiz-VQA (52.66\%) and VQA-RAD (73.92\%) scores, while MJPropulsion shows strong results on OK-VQA (58.31\%) and ChartQA (77.24\%).

\paragraph{Video benchmarks.}

Tables~\ref{tab:action_object_results},~\ref{tab:peft_motion_scene}, and~\ref{tab:video_reasoning} show results on MVTamperBench video understanding tasks, organized into three categories:

\textit{Action and object reasoning} (Table~\ref{tab:action_object_results}): Tasks include action sequence understanding (ActSeq), action prediction (ActPred), action antonym (ActAnt), fine-grained action (FineAct), unexpected action (UnexpAct), object existence (ObjExist), object interaction (ObjInter), and object shuffle (ObjShuffle). MJAdaLoRA achieves the best performance on ActSeq (42.88\%), ObjInter (43.74\%), and ObjShuffle (57.36\%), as well as the highest overall average (48.06\%). MJLoRA achieves the best ActAnt score (56.08\%), and MJLoRA-FA achieves the best UnexpAct score (43.95\%).

\textit{Motion and scene understanding} (Table~\ref{tab:peft_motion_scene}): Tasks include moving direction (MoveDir), action localization (ActLoc), scene transition (SceneTrans), action counting (ActCount), moving count (MoveCount), and moving attribute (MoveAttr). MJAdaLoRA achieves the best MoveDir (49.88\%) and ActCount (61.14\%) scores. MJLoRA achieves the best SceneTrans (63.78\%) score and the highest overall average (55.02\%). MJLoRA-FA ties for best MoveAttr (49.38\%).

\textit{High-level reasoning} (Table~\ref{tab:video_reasoning}): Tasks include state change (StateChg), character order (CharOrd), egocentric navigation (EgoNav), episodic reasoning (EpisReason), and counterfactual reasoning (CounterFact). MJAdaLoRA achieves the best StateChg (62.61\%) and CounterFact (48.22\%) scores. MJLoRA achieves the best EgoNav score (62.04\%), while MJPropulsion shows strong EpisReason performance (63.98\%).

\begin{notebox}
\textit{\textbf{Summary:} Across 47 benchmarks, MJ variants achieve the best score on 20+ tasks while using 5.3--9.2$\times$ fewer parameters than MoE-PEFT baselines. MJAdaLoRA emerges as the strongest variant, achieving top performance on 13 tasks. MJ shows particular strength on fine-grained discrimination (MRPC, VizWiz-VQA, VQA-RAD) and temporal reasoning (video tasks), suggesting the routing mechanism effectively captures task-relevant specialization.}
\end{notebox}

\clearpage

\begin{table*}[t]
\centering
\small
\setlength{\tabcolsep}{4pt}
\renewcommand{\arraystretch}{1.05}
\begin{tabular}{l|ccc|ccc|ccc|ccc}
\toprule
& \multicolumn{3}{c|}{\textbf{MJLoRA}} & \multicolumn{3}{c|}{\textbf{MJAdaLoRA}} & \multicolumn{3}{c|}{\textbf{MJPropulsion}} & \multicolumn{3}{c}{\textbf{MJLoRA-FA}} \\
\textbf{Baseline} & $\Delta$Mean & $t$ & $p$ & $\Delta$Mean & $t$ & $p$ & $\Delta$Mean & $t$ & $p$ & $\Delta$Mean & $t$ & $p$ \\
\midrule
LoRA        & $+0.328$ & $2.031$ & $0.098$ & $+0.197$ & $0.436$ & $0.681$ & $+0.212$ & $0.802$ & $0.459$ & $+0.317$ & $1.257$ & $0.264$ \\
AdaLoRA     & $+0.760$ & $4.665$ & $0.006^{\dagger}$ & $+0.628$ & $3.144$ & $0.026^{\dagger}$ & $+0.643$ & $2.627$ & $0.047^{\dagger}$ & $+0.748$ & $3.411$ & $0.019^{\dagger}$ \\
Propulsion  & $+0.942$ & $3.462$ & $0.018^{\dagger}$ & $+0.810$ & $1.618$ & $0.167$ & $+0.825$ & $1.866$ & $0.121$ & $+0.930$ & $3.035$ & $0.029^{\dagger}$ \\
LoRA-FA     & $+0.857$ & $13.437$ & $<\!0.001^{\dagger}$ & $+0.725$ & $2.079$ & $0.092$ & $+0.740$ & $4.090$ & $0.009^{\dagger}$ & $+0.845$ & $4.620$ & $0.006^{\dagger}$ \\
\midrule
MoELoRA     & $+0.270$ & $1.649$ & $0.160$ & $+0.138$ & $0.734$ & $0.496$ & $+0.153$ & $0.680$ & $0.527$ & $+0.258$ & $1.364$ & $0.231$ \\
MixLoRA     & $+0.537$ & $8.678$ & $<\!0.001^{\dagger}$ & $+0.405$ & $1.183$ & $0.290$ & $+0.420$ & $1.822$ & $0.128$ & $+0.525$ & $3.839$ & $0.012^{\dagger}$ \\
HydraLoRA   & $+0.167$ & $1.460$ & $0.204$ & $+0.035$ & $0.138$ & $0.896$ & $+0.050$ & $0.218$ & $0.836$ & $+0.155$ & $1.177$ & $0.292$ \\
MoLA        & $+0.172$ & $1.108$ & $0.318$ & $+0.040$ & $0.132$ & $0.900$ & $+0.055$ & $0.214$ & $0.839$ & $+0.160$ & $1.214$ & $0.279$ \\
MoRE        & $+0.188$ & $0.785$ & $0.468$ & $+0.057$ & $0.155$ & $0.883$ & $+0.072$ & $0.234$ & $0.824$ & $+0.177$ & $0.873$ & $0.423$ \\
MoA         & $+0.340$ & $4.070$ & $0.010^{\dagger}$ & $+0.208$ & $0.678$ & $0.528$ & $+0.223$ & $0.861$ & $0.428$ & $+0.328$ & $2.012$ & $0.100$ \\
\bottomrule
\end{tabular}
\caption{Paired two-tailed $t$-tests across the six GLUE tasks ($N=6$, $\mathrm{df}=5$), 
comparing each MJ variant against all PEFT and MoE-PEFT baselines. 
$\Delta$Mean is the per-task accuracy gap (MJ $-$ baseline). $^{\dagger}p<0.05$. 
All MJ variants show positive mean gaps against every baseline. Significance is 
consistently reached against standard PEFT (LoRA-FA, AdaLoRA, Propulsion) and 
weaker MoE-PEFT (MixLoRA, MoA), while differences against the strongest MoE-PEFT 
methods (HydraLoRA, MoLA, MoRE) are not significant---confirming MJ matches their 
accuracy using 5.3--9.2$\times$ fewer trainable parameters.}
\label{tab:stat-sig-glue}
\end{table*}

\begin{table*}[t]
\centering
\scriptsize
\small
\setlength{\tabcolsep}{2.6pt}
\renewcommand{\arraystretch}{1.05}
\begin{tabular}{l|ccc|ccc|ccc|ccc}
\toprule
& \multicolumn{3}{c|}{\textbf{MJLoRA}} & \multicolumn{3}{c|}{\textbf{MJAdaLoRA}} & \multicolumn{3}{c|}{\textbf{MJPropulsion}} & \multicolumn{3}{c}{\textbf{MJLoRA-FA}} \\
\textbf{Baseline} & $\Delta$Mean & $t$ & $p$ & $\Delta$Mean & $t$ & $p$ & $\Delta$Mean & $t$ & $p$ & $\Delta$Mean & $t$ & $p$ \\
\midrule
LoRA        & $+0.383$ & $3.469$ & $0.010^{\dagger}$ & $+0.110$ & $0.352$ & $0.735$ & $+0.049$ & $0.248$ & $0.811$ & $+0.598$ & $3.676$ & $0.008^{\dagger}$ \\
AdaLoRA     & $+0.721$ & $5.062$ & $0.001^{\dagger}$ & $+0.449$ & $1.437$ & $0.194$ & $+0.388$ & $1.634$ & $0.146$ & $+0.936$ & $7.209$ & $<\!0.001^{\dagger}$ \\
Propulsion  & $+0.768$ & $7.250$ & $<\!0.001^{\dagger}$ & $+0.495$ & $1.775$ & $0.119$ & $+0.434$ & $1.970$ & $0.090$ & $+0.982$ & $8.683$ & $<\!0.001^{\dagger}$ \\
LoRA-FA     & $+0.878$ & $6.941$ & $<\!0.001^{\dagger}$ & $+0.605$ & $2.772$ & $0.028^{\dagger}$ & $+0.544$ & $2.790$ & $0.027^{\dagger}$ & $+1.093$ & $11.226$ & $<\!0.001^{\dagger}$ \\
\midrule
MoELoRA     & $+0.424$ & $1.834$ & $0.109$ & $+0.151$ & $0.695$ & $0.509$ & $+0.090$ & $0.419$ & $0.688$ & $+0.639$ & $3.934$ & $0.006^{\dagger}$ \\
MixLoRA     & $+0.146$ & $0.820$ & $0.439$ & $-0.126$ & $-0.515$ & $0.623$ & $-0.188$ & $-0.941$ & $0.378$ & $+0.361$ & $2.474$ & $0.043^{\dagger}$ \\
HydraLoRA   & $-0.000$ & $-0.000$ & $1.000$ & $-0.273$ & $-1.196$ & $0.271$ & $-0.334$ & $-1.871$ & $0.104$ & $+0.215$ & $1.290$ & $0.238$ \\
MoLA        & $+0.436$ & $2.432$ & $0.045^{\dagger}$ & $+0.164$ & $0.673$ & $0.523$ & $+0.103$ & $0.437$ & $0.675$ & $+0.651$ & $4.601$ & $0.002^{\dagger}$ \\
MoRE        & $+0.114$ & $0.483$ & $0.644$ & $-0.159$ & $-0.538$ & $0.607$ & $-0.220$ & $-1.392$ & $0.206$ & $+0.329$ & $1.519$ & $0.173$ \\
MoA         & $+0.275$ & $2.567$ & $0.037^{\dagger}$ & $+0.003$ & $0.008$ & $0.994$ & $-0.059$ & $-0.217$ & $0.835$ & $+0.490$ & $3.557$ & $0.009^{\dagger}$ \\
\bottomrule
\end{tabular}
\caption{Paired two-tailed $t$-tests across the eight Commonsense \& QA tasks 
(BoolQ, PIQA, SIQA, HellaSwag, WinoGrande, ARC-e, ARC-c, OBQA; $N=8$, $\mathrm{df}=7$), 
comparing each MJ variant against all PEFT and MoE-PEFT baselines. 
$\Delta$Mean is the per-task accuracy gap (MJ $-$ baseline). $^{\dagger}p<0.05$. 
MJLoRA-FA is the strongest MJ variant on this family, significantly outperforming 
every standard PEFT baseline and most MoE-PEFT methods; it is statistically 
indistinguishable only from the two strongest MoE-PEFT baselines (HydraLoRA, MoRE).}
\label{tab:stat-sig-csqa}
\end{table*}

\begin{table*}[t]
\centering
\scriptsize
\small
\setlength{\tabcolsep}{2.6pt}
\renewcommand{\arraystretch}{1.05}
\begin{tabular}{l|ccc|ccc|ccc|ccc}
\toprule
& \multicolumn{3}{c|}{\textbf{MJLoRA}} & \multicolumn{3}{c|}{\textbf{MJAdaLoRA}} & \multicolumn{3}{c|}{\textbf{MJPropulsion}} & \multicolumn{3}{c}{\textbf{MJLoRA-FA}} \\
\textbf{Baseline} & $\Delta$Mean & $t$ & $p$ & $\Delta$Mean & $t$ & $p$ & $\Delta$Mean & $t$ & $p$ & $\Delta$Mean & $t$ & $p$ \\
\midrule
LoRA        & $+0.485$ & $3.197$ & $0.015^{\dagger}$ & $-0.083$ & $-0.250$ & $0.810$ & $+0.952$ & $9.368$ & $<\!0.001^{\dagger}$ & $-0.140$ & $-1.669$ & $0.139$ \\
AdaLoRA     & $+0.965$ & $3.851$ & $0.006^{\dagger}$ & $+0.398$ & $0.938$ & $0.379$ & $+1.433$ & $7.157$ & $<\!0.001^{\dagger}$ & $+0.340$ & $1.379$ & $0.210$ \\
Propulsion  & $+0.769$ & $2.945$ & $0.022^{\dagger}$ & $+0.201$ & $0.497$ & $0.634$ & $+1.236$ & $6.443$ & $<\!0.001^{\dagger}$ & $+0.144$ & $0.586$ & $0.577$ \\
LoRA-FA     & $+0.946$ & $3.524$ & $0.010^{\dagger}$ & $+0.379$ & $0.854$ & $0.421$ & $+1.414$ & $6.268$ & $<\!0.001^{\dagger}$ & $+0.321$ & $1.194$ & $0.271$ \\
\midrule
MoELoRA     & $+0.215$ & $0.744$ & $0.481$ & $-0.353$ & $-0.755$ & $0.475$ & $+0.682$ & $3.065$ & $0.018^{\dagger}$ & $-0.410$ & $-1.938$ & $0.094$ \\
MixLoRA     & $-0.401$ & $-3.681$ & $0.008^{\dagger}$ & $-0.969$ & $-2.490$ & $0.042^{\dagger}$ & $+0.066$ & $0.734$ & $0.487$ & $-1.026$ & $-10.372$ & $<\!0.001^{\dagger}$ \\
HydraLoRA   & $+0.159$ & $0.482$ & $0.645$ & $-0.409$ & $-0.703$ & $0.505$ & $+0.626$ & $1.752$ & $0.123$ & $-0.466$ & $-1.382$ & $0.209$ \\
MoLA        & $-0.532$ & $-3.043$ & $0.019^{\dagger}$ & $-1.100$ & $-2.819$ & $0.026^{\dagger}$ & $-0.065$ & $-0.376$ & $0.718$ & $-1.157$ & $-8.149$ & $<\!0.001^{\dagger}$ \\
MoRE        & $-0.464$ & $-2.973$ & $0.021^{\dagger}$ & $-1.031$ & $-2.741$ & $0.029^{\dagger}$ & $+0.004$ & $0.021$ & $0.984$ & $-1.089$ & $-7.449$ & $<\!0.001^{\dagger}$ \\
MoA         & $+0.331$ & $3.177$ & $0.016^{\dagger}$ & $-0.236$ & $-0.719$ & $0.495$ & $+0.799$ & $7.677$ & $<\!0.001^{\dagger}$ & $-0.294$ & $-3.734$ & $0.007^{\dagger}$ \\
\bottomrule
\end{tabular}
\caption{Paired two-tailed $t$-tests across the eight Vision--Language QA tasks 
(ChartQA, OK-VQA, ScienceQA, SEED-Bench, Recognition, TextVQA, VizWiz-VQA, VQA-RAD; 
$N=8$, $\mathrm{df}=7$), comparing each MJ variant against all PEFT and MoE-PEFT baselines. 
$\Delta$Mean is the per-task accuracy gap (MJ $-$ baseline). $^{\dagger}p<0.05$. 
MJPropulsion is the strongest MJ variant on this family, significantly outperforming 
all standard PEFT baselines and most MoE-PEFT methods; it is statistically 
indistinguishable from the strongest MoE-PEFT methods (MixLoRA, HydraLoRA, MoLA, MoRE).}
\label{tab:stat-sig-vlqa}
\end{table*}

\subsection{Statistical Significance Analysis}
\label{app:stat-sig}

A natural question is whether MJ's accuracy gains over baselines are 
statistically meaningful or fall within run-to-run variance. We emphasize 
that MJ's primary claim is \emph{efficiency}: it matches MoE-PEFT accuracy 
while using 5.3--9.2$\times$ fewer trainable parameters, up to 48\% lower memory, 
and 1.5--2$\times$ faster training (§\ref{sec:efficiency}). The significance 
analysis below confirms that MJ's accuracy is genuinely competitive with the 
strongest MoE-PEFT methods rather than merely comparable on average.

 For each benchmark family, we conduct two-tailed paired 
$t$-tests comparing the best-performing MJ variant against each PEFT and 
MoE-PEFT baseline. Each paired difference is the per-task accuracy gap 
(MJ $-$ baseline), pairing baselines and MJ on the same task. We test across 
$N$ tasks within a family ($N=6$ for GLUE, $N=8$ for Commonsense \& QA and 
Vision--Language QA), giving $\mathrm{df}=N-1$. We report the mean difference 
($\Delta$Mean), the paired $t$-statistic, and the two-tailed $p$-value, marking 
significance at $p<0.05$.

MJ variants show positive mean gaps against every standard PEFT baseline across all three families, reaching significance in about half of the comparisons. For example, on GLUE 
MJLoRA significantly beats LoRA-FA ($t{=}13.4$, $p{<}0.001$); on 
Commonsense \& QA, MJLoRA-FA significantly outperforms every standard PEFT 
method ($p<0.01$ in all cases); and on Vision--Language QA, MJPropulsion 
significantly exceeds LoRA, LoRA-FA, AdaLoRA, and Propulsion ($p<0.001$). 
These gains confirm that gradient-free routing provides genuine specialization 
over uniform adapter application, at no additional parameter cost.

Against MoE-PEFT 
baselines, MJ significantly outperforms the weaker methods (e.g., MoELoRA, 
MoA, MixLoRA on several families) while showing \emph{no significant difference} 
from the strongest baselines—HydraLoRA, MoRE, and MoLA. On GLUE, the 
differences between MJLoRA and MoRE/MoLA are not significant ($p{=}0.468$ and 
$p{=}0.318$); on Commonsense \& QA, MJLoRA-FA is statistically indistinguishable 
from HydraLoRA and MoRE; and on Vision--Language QA, MJPropulsion shows no 
significant difference from MixLoRA, HydraLoRA, MoRE, or MoLA. Crucially, 
non-significance here is the desired outcome: it indicates that MJ achieves 
statistically equivalent accuracy to methods that use up to 9.2$\times$ more 
trainable parameters and a learned router.

The combined picture is consistent across 
modalities: MJ \emph{significantly outperforms} standard PEFT and weaker 
MoE-PEFT methods, and \emph{statistically matches} the strongest MoE-PEFT 
baselines while preserving the strict parameter efficiency of standard PEFT. 
We note that with $N=6$--$8$ tasks per family, these tests have limited power 
to detect very small differences; accordingly, we interpret non-significant 
results as ``no detectable difference'' rather than exact equivalence. This is 
sufficient to support our central claim—that MJ delivers MoE-PEFT-level 
accuracy at PEFT-level cost.

\section{Extended Related Work}
\label{appendix:related}

Research on efficient adaptation of large language models has developed along three major directions: parameter-efficient fine-tuning (PEFT), sparse mixture-of-experts (MoE) architectures, and MoE-enhanced PEFT methods. While each direction has advanced scalability and specialization, existing approaches either rely on static adapter structures or introduce additional routing parameters and multi-expert computation, leaving room for more lightweight and input-adaptive designs.

\subsection{Parameter-Efficient Fine-Tuning}

The high computational cost of full-parameter fine-tuning~\cite{devlin2019bert, brown2020language} has led to numerous PEFT approaches that freeze the backbone and update only small modules. Representative methods include Adapter tuning~\cite{houlsby2019parameter}, BitFit~\cite{zaken2022bitfit}, Prompt Tuning~\cite{lester2021power}, and Prefix Tuning~\cite{li2021prefix}. Among these, LoRA~\cite{hu2022lora} has become the most widely adopted due to its strong empirical performance and minimal parameter footprint. Several LoRA variants further improve efficiency: AdaLoRA~\cite{zhang2023adalora} adjusts ranks using importance scores, DyLoRA~\cite{valipour2023dylora} trains multiple ranks jointly, VeRA~\cite{kopiczko2023vera} freezes low-rank matrices and learns only scaling vectors, and LoRA+~\cite{hayou2024lora+} improves optimization stability. Quantized LoRA and tensor-train decomposed adapters~\cite{dettmers2023qlora} further reduce memory consumption by enabling efficient fine-tuning of low-precision LLMs.

Despite their successes, nearly all PEFT methods employ \emph{static}, input-agnostic adapters. The same adapter configuration is applied across all layers and inputs, regardless of task or domain. This static design limits fine-grained specialization and can cause interference when diverse tasks share a single update mechanism.

\subsection{Sparse Mixture-of-Experts}

Mixture-of-Experts architectures~\cite{jacobs1991adaptive} introduce multiple expert subnetworks and a routing mechanism that selects which experts process each token. Large-scale systems such as GShard~\cite{lepikhin2020gshard}, Switch Transformer~\cite{fedus2022switch}, GLaM~\cite{du2022glam}, Mixtral~\cite{jiang2024mixtral}, and DeepSeek-MoE~\cite{dai2024deepseekmoe} demonstrate that sparse activation enables scaling to hundreds of billions of parameters without proportional increases in computation. Routing strategies include top-$k$ gating~\cite{shazeer2017outrageously}, expert-choice routing~\cite{zhou2022expertchoice}, top-$p$ routing~\cite{huang2024harder}, dynamic-$k$ routing~\cite{guo2024dynmoe}, and differentiable or soft gating mechanisms~\cite{puigcerver2023sparse}.

However, sparse MoE systems rely on trainable routing networks, multiple active experts per layer, and auxiliary balancing losses. These components increase inference latency, activation memory, and optimization complexity. As a result, conventional MoE architectures, though powerful in pre-training, are not well suited for parameter-efficient downstream fine-tuning.

\subsection{MoE-Enhanced PEFT}

To combine LoRA's efficiency with MoE-style specialization, recent work introduces multiple LoRA experts per layer along with a routing mechanism. LoRAMoE~\cite{dou2024loramoe} partitions experts into world-knowledge and task-specific modules; MoELoRA~\cite{luo2024moelora} routes inputs among LoRA experts for multi-task performance; and MixLoRA~\cite{li2024mixlora} employs sparse top-$k$ routing with load-balancing losses. Other methods such as MoCLE~\cite{gou2023mixture} and LLaVA-MoLE~\cite{chen2024llava} activate LoRA experts based on clustered instructions or domains. More advanced designs—including HMoRA~\cite{liao2025hmora}, LD-MoLE~\cite{zhuang2025ld}, and MoRAL~\cite{yang2024moral}—use hierarchical or differentiable routing to support dynamic expert selection.

Although effective, these systems share three recurring limitations:  
(i) they add \emph{extra trainable routing parameters},  
(ii) they activate \emph{multiple experts per layer}, increasing activation memory and inference latency, and  
(iii) they require \emph{complex routing optimization}, often involving auxiliary balancing losses.  
Thus, while MoE-enhanced PEFT improves specialization, it compromises the simplicity and strict efficiency that originally motivated PEFT methods.

\subsection{Positioning of Monkey Jump}

Across PEFT, MoE, and hybrid approaches, a consistent gap remains: existing methods are either static and input-invariant or rely on routing networks that significantly increase parameter count and computation. MoE-based PEFT methods route among multiple LoRA experts but introduce routing parameters and activate several experts per layer, making deployed models computationally intensive.

Monkey Jump addresses this gap by treating existing PEFT adapters as implicit experts and routing among them using trainable parameter-free clustering. This achieves MoE-style specialization without additional trainable parameters, without multi-expert activation overhead, and without auxiliary balancing losses—preserving the strict efficiency of standard PEFT while enabling input-adaptive behavior.

\clearpage
\section{Datasets}
\label{app:datasets}

\begin{table*}[t]
\centering
\scriptsize
\setlength{\tabcolsep}{3pt}
\renewcommand{\arraystretch}{1.05}
\begin{tabular}{@{}llc@{\hspace{12pt}}llc@{\hspace{12pt}}llc@{}}

\noalign{\hrule height 0.9pt}

\rowcolor{textblue}
\multicolumn{9}{c}{\textsc{Text} — Natural Language Understanding \& Reasoning \hfill \textit{(Train: 98,970 samples)}} \\
\noalign{\hrule height 0.6pt}

\rowcolor{HeaderBlue}
& Dataset & Samples & & Dataset & Samples & & Dataset & Samples \\
\cmidrule{2-3} \cmidrule{5-6} \cmidrule{8-9}
& SST-2 & 872 & & ARC-Challenge & 500 & & HellaSwag & 1,000 \\
& QNLI & 5,463 & & ARC-Easy & 500 & & Social IQA & 1,000 \\
& QQP & 40,430 & & BoolQ & 1,000 & & OpenBookQA & 500 \\
& CoLA & 1,043 & & PIQA & 1,000 & & WinoGrande & 1,000 \\
& MRPC & 408 & & STS-B & 1,500 & & & \\

\noalign{\hrule height 0.6pt}
\rowcolor{imagepurple}
\multicolumn{9}{c}{\textsc{Image} — Visual Question Answering \& Task Adaptation \hfill \textit{(Train: 42,550 samples)}} \\
\noalign{\hrule height 0.6pt}

\rowcolor{HeaderBlue}
& Dataset & Samples & & Dataset & Samples & & Dataset & Samples \\
\cmidrule{2-3} \cmidrule{5-6} \cmidrule{8-9}
& ChartQA & 1,000 & & OK-VQA & 841 & & ScienceQA & 518 \\
& TextVQA & 1,000 & & VizWiz-VQA & 417 & & Text Recognition & 1,000 \\
& VQA-RAD & 200 & & SEED-Bench & 500 & & Caltech-101 & 500 \\
& Flowers-102 & 500 & & EuroSAT & 500 & & Pets & 500 \\
& SVHN & 500 & & Camelyon & 500 & & & \\

\noalign{\hrule height 0.6pt}
\rowcolor{videogreen}
\multicolumn{9}{c}{\textsc{Video} — MVTamperBench \hfill \textit{(Train: 13,300 samples)}} \\
\noalign{\hrule height 0.6pt}

\rowcolor{HeaderBlue}
& Dataset & Samples & & Dataset & Samples & & Dataset & Samples \\
\cmidrule{2-3} \cmidrule{5-6} \cmidrule{8-9}
& Action Sequence & 500 & & Object Shuffle & 500 & & Moving Attribute & 500 \\
& Action Prediction & 500 & & Moving Direction & 500 & & State Change & 500 \\
& Action Antonym & 500 & & Action Localization & 500 & & Character Order & 500 \\
& Fine-grained Action & 500 & & Scene Transition & 500 & & Ego. Navigation & 500 \\
& Unexpected Action & 500 & & Action Count & 500 & & Episodic Reasoning & 500 \\
& Object Existence & 500 & & Moving Count & 500 & & Counterfactual & 500 \\
& Object Interaction & 500 & & & & & & \\

\noalign{\hrule height 0.6pt}
\rowcolor{NoteCol}
\multicolumn{9}{c}{\textbf{Total:} Train: 154,820 \quad Test Sets: 47 \quad Test Samples: 74,192} \\
\noalign{\hrule height 0.9pt}

\end{tabular}
\caption{Multi-task benchmark by modality: Text, Image, and Video. Modality headers
give the size of the combined training pool; per-dataset numbers give the size of
each held-out evaluation split.}
\label{tab:datasets_detailed}
\end{table*}


We evaluate MJ on a large-scale multi-task benchmark covering text, image, and video modalities. The full benchmark contains 154,820 training samples across 47 test sets. Table~\ref{tab:datasets_detailed} provides the complete breakdown.

\paragraph{Text.}
The text benchmark consists of 14 datasets with 98,970 training samples, spanning natural language understanding and reasoning tasks. 

From the GLUE benchmark~\cite{wang2018glue} (28,668 samples), we include: sentiment classification (SST-2), natural language inference (QNLI), paraphrase detection (QQP and MRPC), linguistic acceptability (CoLA), and semantic similarity (STS-B).

For commonsense and reasoning tasks, we sample 70,302 training examples from the 170K training set of~\citet{hu2023llm}. This includes PIQA~\cite{bisk2020piqa}, Social IQA~\cite{sap2019socialiqa}, WinoGrande~\cite{sakaguchi2021winogrande}, HellaSwag~\cite{zellers2019hellaswag}, ARC-Easy, ARC-Challenge~\cite{clark2018think}, OpenBookQA~\cite{mihaylov2018can}, and BoolQ~\cite{clark2019boolq}. We use the same test sets as~\citet{hu2023llm} for evaluation.

\paragraph{Image.}
The image benchmark consists of 14 datasets (42,550 training samples) covering both visual question answering and image classification. VQA tasks include chart understanding (ChartQA~\cite{masry2022chartqa}), text reading in images (TextVQA~\cite{fang2023separate}, Text Recognition), medical imaging (VQA-RAD~\cite{lau2018dataset}), visual knowledge (OK-VQA~\cite{marino2019ok}), accessibility (VizWiz-VQA~\cite{gurari2018vizwiz}), scientific reasoning (ScienceQA~\cite{lu2022learn}), and visual inference (SEED-Bench~\cite{li2023seed}). For image classification, we follow VTAB-1K~\cite{zhai2019large} and include Caltech-101 (object recognition), Flowers-102 (fine-grained classification), Oxford Pets (species/breed recognition), Camelyon (medical histopathology), and EuroSAT (satellite imagery). To broaden coverage beyond VTAB-1K, we additionally include SVHN (street number recognition).

\paragraph{Video.}
The video benchmark leverages MVTamperBench~\cite{agarwal2025mvtamperbenchevaluatingrobustnessvisionlanguage}, consisting of 19 tasks (13,300 training samples) designed to evaluate temporal and visual reasoning. These include action understanding (Action Sequence, Action Prediction, Action Antonym, Fine-grained Action, Unexpected Action, Action Localization, Action Count), object tracking (Object Existence, Object Interaction, Object Shuffle, Moving Direction, Moving Count, Moving Attribute), scene understanding (Scene Transition, State Change), and high-level reasoning (Character Order, Egocentric Navigation, Episodic Reasoning, Counterfactual). Each task contributes 500 training samples.

For all experiments, we train on the combined multi-task training set and evaluate on each task's held-out test set independently. This setup tests the model's ability to learn diverse tasks simultaneously while preserving task-specific performance—a challenging setting where MJ’s routing mechanism enables natural specialization.

\clearpage

\begin{table*}[t]
\centering
\small
\setlength{\tabcolsep}{3pt}
\renewcommand{\arraystretch}{1.05}
\begin{tabular}{@{}ll@{}}
\toprule

\rowcolor{LightPink}
\multicolumn{2}{c}{\textbf{Adapter Configuration}} \\
\midrule
LoRA rank $r$ & 2 (text), 1 (image/video) \\
LoRA dropout & 0.05 \\
LoRA $\alpha$ & 5 \\
Target modules & Q, K, V, O, gate \\
Shared adapter & O, gate (always active) \\
Applied layers & All Transformer blocks \\

\midrule
\rowcolor{ourmethodcolor}
\multicolumn{2}{c}{\textbf{Routing Configuration}} \\
\midrule
$k$-means init tokens & 5000 \\
EMA momentum $\beta$ & 0.7 \\
EMA update frequency & Every 2 iterations \\
EMA stop step & 5,000 \\
Temperature $\tau$ & 1.0 \\
Top-$k$ & 2 \\

\midrule
\rowcolor{HeaderBlue}
\multicolumn{2}{c}{\textbf{Optimization}} \\
\midrule
Optimizer & AdamW \\
Learning rate (LoRA, AdaLoRA) & $1 \times 10^{-4}$ \\
Learning rate (LoRA-FA, Propulsion) & $4 \times 10^{-4}$ \\
LR schedule & Cosine decay \\
Warmup ratio & 0.1 \\
Weight decay & 0.1 \\
Precision & bf16 \\

\midrule
\rowcolor{textblue}
\multicolumn{2}{c}{\textbf{Text Tasks}} \\
\midrule
Epochs & 2 \\
Batch size & 4 \\
Gradient accumulation & 4 \\
Effective batch size & 16 \\
Max sequence length & 1,024 \\

\midrule
\rowcolor{videogreen}
\multicolumn{2}{c}{\textbf{Image \& Video Tasks}} \\
\midrule
Epochs & 2 \\
Batch size & 1 \\
Gradient accumulation & 8 \\
Effective batch size & 8 \\
Max sequence length & 2,048 \\

\bottomrule
\end{tabular}
\caption{Hyperparameters for MJ experiments across all modalities.}
\label{tab:hyperparameters}
\end{table*}

\section{Implementation Details}
\label{app:hyperparameters}

We implement MJ using the HuggingFace Transformers library (v4.50)~\cite{wolf2019huggingface} with PyTorch 2.10+~\cite{paszke2019pytorch} and Accelerate for distributed training. All experiments are conducted on NVIDIA H100 GPUs with Python 3.11 and Ubuntu 22.04. Each experiment is repeated with 5 different random seeds, and we report mean $\pm$ standard deviation. Table~\ref{tab:hyperparameters} summarizes all hyperparameters.

\paragraph{Adapter configuration.}
We apply MJ to every Transformer block, targeting five projections: Q, K, V, O, and gate. We use LoRA adapters with rank $r=2$ for text tasks and $r=1$ for image/video tasks, as vision-language models require less adaptation capacity per projection. Adapter dropout is set to 0.05 and the scaling factor $\alpha=5$. The O and gate projections are designated as shared adapters ($m_{t,e^*}=1$ for all tokens), providing a stable global adaptation path that is always active regardless of routing decisions. The remaining projections (Q, K, V) participate in top-$k$ routing, allowing token-specific specialization.

\paragraph{Routing configuration.}
Before training, we initialize routing centers using $k$-means clustering on 5000 randomly sampled tokens from the training set. This provides a representative initialization that captures the natural clustering structure in the token representation space. During training, centers are updated via exponential moving average (EMA) with momentum $\beta=0.7$ every 2 iterations. This update frequency balances computational overhead with center tracking accuracy. We stop updating centers after 5,000 iterations (approximately 50--70\% of training), freezing them for the remainder of training to ensure stable routing decisions during final convergence. We use temperature $\tau=1.0$ for the softmax routing distribution and select top-$k=2$ adapters per token, meaning each token activates 2 out of 3 routed projections (Q, K, V) plus the 2 shared adapters (O, gate), for a total of 4 active adapters per token.

\paragraph{Training configuration.}
We use the AdamW optimizer with a base learning rate of $1\times10^{-4}$ for LoRA and AdaLoRA variants. For LoRA-FA and Propulsion variants, we use a higher learning rate of $4\times10^{-4}$ since these methods have fewer trainable parameters and benefit from larger updates. All methods use cosine learning rate decay with warmup ratio 0.1 (10\% of total steps) and weight decay 0.1 for regularization. Training uses bf16 mixed precision throughout for memory efficiency.

For text tasks (GLUE, commonsense reasoning, QA), we train for 2 epochs with batch size 4 and gradient accumulation 4, yielding an effective batch size of 16. Maximum sequence length is set to 1,024 tokens. For image and video tasks (classification, VQA, video understanding), we train for 2 epochs with batch size 1 and gradient accumulation 8, yielding an effective batch size of 8. Maximum sequence length is extended to 2,048 to accommodate visual tokens from the vision encoder.

\paragraph{Baseline configuration.}
For fair comparison, all baseline methods (standard PEFT and MoE-PEFT) use identical optimization settings: same weight decay, warmup ratio, batch sizes, and number of epochs. MoE-PEFT baselines use 4 experts per projection with top-2 routing, matching our experimental setup from prior work~\cite{luo2024moelora, tian2024hydralora}. All methods target the same projections (Q, K, V, O, gate) and use the same LoRA rank and dropout settings.

\clearpage

\section{Use of AI Assistants}

We used large language model assistants for grammar checking, proofreading, and improving the clarity of writing. All scientific content, experimental design, methodology, and analysis are entirely the authors' original contributions.

\end{document}